\documentclass[numbib]{imaiai2}
\usepackage[usenames, dvipsnames]{xcolor}
\usepackage[T1]{fontenc}    \usepackage[utf8]{inputenc} \usepackage{hyperref}       \usepackage{url}            \usepackage{booktabs}       \usepackage{amsfonts}       \usepackage{nicefrac}       \usepackage{microtype}      
\usepackage{array, makecell}
\usepackage{amsthm}
\usepackage{amsmath}
\usepackage{amssymb}
\usepackage{microtype}
\usepackage{graphicx}
\usepackage{subfigure}
\usepackage{booktabs} \usepackage{hyperref}
\usepackage{breakcites}
\usepackage{wrapfig}
\usepackage{multirow}
\usepackage[mathscr]{euscript}
\usepackage{algorithm,algpseudocode}
\usepackage{mathtools}
\usepackage{tikz}
\usepackage{nicefrac}

\newcommand{\losswass}{H^{\Omega}_{p}}
\newcommand{\lossgromov}{J_{p}}
\newcommand{\lossfgw}{E^{\Omega}_{p,q,\alpha}}
\newcommand{\couplingset}{\Pi}

\usepackage{etoolbox}
\makeatletter
\patchcmd\@combinedblfloats{\box\@outputbox}{\unvbox\@outputbox}{}{   \errmessage{\noexpand\@combinedblfloats could not be patched}} \makeatother

\usepackage{array}
\newcolumntype{P}[1]{>{\centering\arraybackslash}p{#1}}
\newcolumntype{M}[1]{>{\centering\arraybackslash}m{#1}}

\AtBeginDocument{
  \label{CorrectFirstPageLabel}
  
}

\ifx\undefined\BySame
\newcommand{\BySame}{\leavevmode\rule[.5ex]{3em}{.5pt}\ }
\fi
\ifx\undefined\textsc
\newcommand{\textsc}[1]{{\sc #1}}
\newcommand{\emph}[1]{{\em #1\/}}
\let\tmpsmall\small
\renewcommand{\small}{\tmpsmall\sc}
\fi

\begin{document}

\title{Fused Gromov-Wasserstein distance for structured objects:\\ theoretical foundations and mathematical properties}

\shorttitle{Fused Gromov-Wasserstein distance for structured objects} \shortauthorlist{Titouan Vayer, Laetita Chapel, Rémi Flamary, Romain Tavenard, Nicolas Courty} 
\author{{\sc Titouan Vayer}$^*$,\\[2pt]
Univ. Bretagne-Sud, CNRS, IRISA, F-56000 Vannes\\
$^*${\email{titouan.vayer@irisa.fr}}\\[2pt]
{\sc Laetita Chapel}\\[2pt]
Univ. Bretagne-Sud, CNRS, IRISA, F-56000 Vannes\\
{laetitia.chapel@irisa.fr}\\[6pt]
{\sc Rémi Flamary} \\[2pt]
Univ. C\^ote d'Azur, CNRS, OCA Lagrange, F-06000 Nice\\
{remi.flamary@unice.fr}\\[6pt]
{\sc Romain Tavenard} \\[2pt]
Univ. Rennes, CNRS, LETG,  F-35000 Rennes\\
{ romain.tavenard@univ-rennes2.fr} \\
{\sc and} \\[6pt]
{\sc Nicolas Courty} \\[2pt]
Univ. Bretagne-Sud, CNRS, IRISA, F-56000 Vannes\\
{nicolas.courty@irisa.fr}}

\maketitle

\begin{abstract}
{Optimal transport theory has recently found many applications in machine learning thanks to its capacity for comparing various machine learning objects considered as distributions. The Kantorovitch formulation, leading to the Wasserstein distance, focuses on the features of the elements of the objects but treat them independently, whereas the Gromov-Wasserstein distance focuses only on the relations between the elements, depicting the structure of the object, yet discarding its features. 
 In this paper we propose to extend these distances in order to encode simultaneously both the feature and structure informations, resulting in the Fused Gromov-Wasserstein distance. We develop the mathematical framework for this novel distance, prove its metric and interpolation properties and provide a concentration result for the convergence of finite samples. We also illustrate and interpret its use in various contexts where structured objects are involved. }
{Optimal transport. Structured objects. Wasserstein and Gromov-Wasserstein distances.}
\\
2000 Math Subject Classification: 34K30, 35K57, 35Q80,  92D25
\end{abstract}

\tableofcontents

\section{Introduction}

In this paper we focus on the comparison of structured objects, \textit{i.e} objects defined by both a \emph{feature} and a \emph{structure} information. Abstractly, the feature information covers all the attributes of an object. For example it can model the value of a signal when objects are time series, or the node labels in a graph context. In shape analysis, the spatial positions of the nodes can be regarded as features, or, when objects are images, local color histograms can describe the image's feature information.
As for the structure information, it encodes the specific relationships that exist among the components of the object. In a graph context, nodes and edges are representative of this notion so that each label of the graph may be linked to some others through the edges between the nodes. In a time series context, the values of the signal are related to each other through a temporal structure. This representation is clearly related with the concept of \emph{relational reasoning} (see \cite{relationnalreasoning}) where some \emph{entities} (or elements with attributes such as an intensity of a signal) coexist with some \emph{relations} or properties between them (or some structure as described above).

Including structural knowledge about objects in a machine learning context has often been valuable in order to build more generalizable models. As shown in many contexts such as graphical models \cite{Pearl:1986:FPS:9075.9076,Pearl:2009:CMR:1642718}, relational reinforcement learning \cite{Dzeroski2001} or bayesian non parametrics \cite{hjort10}, considering machine learning objects as a complex composition of entities together with their interactions is crucial in order to learn from small amounts of data. For a review of relational reasoning and its consequences, see \cite{relationnalreasoning}.

Unlike recent deep learning end-to-end approaches \cite{DBLP:journals/nature/LeCunBH15,Goodfellow-et-al-2016} that attempt to avoid integration of prior knowledge or assumptions about the structure wherever possible, \textit{ad hoc} methods, depending on the kind of structured objects involved, aim to build meaningful tools that include structure information in the machine learning process. In graph classification the structure can be taken into account through dedicated graph kernels, in which the structure drives the combination of the feature information~\cite{Shervashidze:2011:WGK:1953048.2078187,pmlr-v48-niepert16,Vishwanathan:2010:GK:1756006.1859891}. In a time series context, Dynamic Time Warping and related approaches are based on the similarity between the features while allowing limited temporal distortion in the time instants that are matched \cite{pmlr-v70-cuturi17a,sakoe1978dynamic}. Closely related, an entire field has focused on predicting the structure as an output and has been deployed on tasks such as segmenting an image into meaningful components or predicting a natural language sentence \cite{Bakir:2007:PSD:1296180,Nowozin:2014:ASP:2627999,pmlr-v80-niculae18a}.

All these approaches rely on meaningful representations of the structured objects that are involved. In this context, an interesting description of machine learning objects can be done through distributions or probability measures. This allows to compare them within the Optimal Transport (OT) framework which provides an elegant way of comparing distributions by capturing the underlying geometric properties of the space through a cost function. When distributions dwell in a common metric space $(\Omega,d)$, the Wasserstein distance defines a metric between these distributions \cite{Villani}. In contrast, the Gromov-Wasserstein distance \cite{Sturm2006, Memoli:2004:CPC:1057432.1057436} aims at comparing distributions that live in different metric spaces through the intrinsic pair-to-pair distances in each space. Unifying both distances, the Fused Gromov-Wasserstein distance was proposed in \cite{2018arXiv180509114V} and used in the discrete setting to encode, in a single OT formulation, both feature and structure information of structured objects. This approach considers structured objects as joint distributions over a common feature space associated with a structure space specific to each object. An OT formulation is derived by considering a tradeoff between the feature and the structure costs, respectively defined with respect to the Wasserstein and the Gromov-Wasserstein standpoints.

This paper presents the theoretical foundations of this distance and states the mathematical properties of the $FGW$ metric in the general setting. We first introduce a representation of structured objects using distributions. We show that classical Wasserstein and Gromov-Wasserstein distance can be used in order to compare either the feature information or the structure information of the structured object but that both fail at comparing the entire object. We then present the Fused Gromov-Wasserstein distance in its general formulation and we derive some of its mathematical properties. Particularly, we show that it is a metric in a given case, we give a concentration result, and we study its interpolation properties and its geodesic properties. We conclude by illustrating and interpreting the distance in several applicative contexts.

\paragraph{Notations}

Let $P(\Omega)$ be the set of all probability measures on a space $\Omega$ and $\mathcal{B}(A)$ the set of all Borel sets of a $\sigma$-algebra A. We note $\#$ the push forward-operator such that for $B \in \mathcal{B}(A)$, $T\#\mu(B)=\mu(T^{-1}(B))$.

A measure $\mu$ on a set $\Omega$ is said to be fully supported if $\text{supp}[\mu]=\Omega$, where $\text{supp}[\mu]$ is the minimal closed subset $A\subset \Omega$ such that $\mu(\Omega \textbackslash A)=0$. Informally, this is the set where the measure ``lives''. We note $P_{i}\#\mu$ the projection on the $i$-th marginal of $\mu$.

For two probability measures $\mu \in P(A)$ and $\nu \in P(B)$ we note $\Pi(\mu,\nu)$ the set of all couplings or matching measures of $\mu$ and $\nu$, \emph{i.e.} the set $\{ \pi \in P(A\times B) \, | \, \forall (A_{0},B_{0}) \in B(A) \times B(B), \pi(A_{0}\times B)=\mu(A_{0}),  \pi(A\times B_{0})=\nu(B_{0})\}$.

For two metric spaces $(X,d_{X})$ and $(Y,d_{Y})$ we define the distance $d_{X} \oplus d_{Y}$ on $X\times Y$ such that, for $(x,y),(x',y') \in X \times Y, \ d_{X} \oplus d_{Y}((x,y),(x',y'))=d_{X}(x,x')+d_{Y}(y,y')$.

We note the simplex of $N$ bins as $\Sigma_{N}=\{a \in (\mathbb{R}^{*}_{+})^{N},\sum_{i} a_{i}=1\}$. For two histograms $a \in \Sigma_{n}$ and $b \in \Sigma_{m}$  we note with some abuses $\couplingset(a,b)$ the set of all couplings of $a$ and $b$, \emph{i.e.} the set $\couplingset(a,b)=\{ \pi \in \mathbb{R}_{+}^{n \times m} | \sum_{i}\pi_{i,j}=b_j ; \sum_{j}\pi_{i,j}=a_i \}$. We also note $\otimes$ the tensor product, \emph{i.e.} for a tensor $L=(L_{i,j,k,l})$, $L\otimes B$ is the matrix $\left(\sum_{k,l} L_{i,j,k,l}B_{k,l}\right)_{i,j}$. Finally, for $x \in \Omega$, $\delta_{x}$ denotes the dirac in $x$.

\paragraph{Assumption} In the paper we suppose that all metric spaces are Polish, non trivial and all measures are Borel.

\section{Structured objects as distributions and Fused Gromov Wasserstein distance}

We can represent structured objects in the discrete case by a labelled graph $\mathcal{G}$ described by $(\{x_{i},a_{i}\})_{i \in [1,..n]}$ where $A=(a_{1},...,a_{n}) \in \Omega^{n} $ is the set of labels (also called features) and $X=(x_{1},...,x_{n})$  a representation of the graph vertices. To this extent, features $a_{i}$ are structured by the intrinsic relation between the vertices $x_{i}$ (see Fig. \ref{graphex}).  Note that in this model, we suppose that we can encode the relation between the vertices in the ambiant space of the vertices $x_i$ through the distance $d_X$ in this space.
\begin{figure}
\centering
\includegraphics[width=0.8\textwidth]{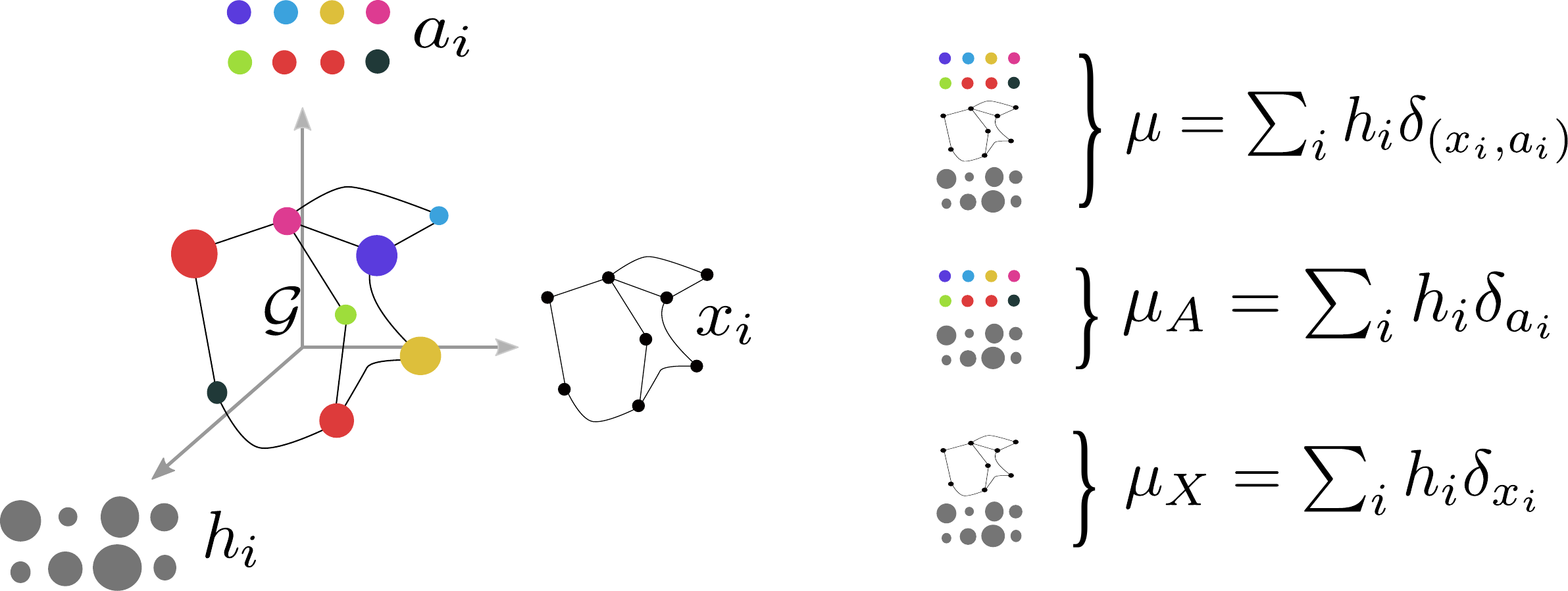}
\caption{Structured object (left) can be described by a labelled graph with $(a_{i})_{i}$ the feature information of the object and $(x_{i})_{i}$ the structure information. If we enrich this object with a histrogram $(h_{i})_{i}$ aiming at measuring the relative importance of the nodes between them we can represent the structured object as a fully supported probability measure $\mu$ over the couple space of feature and structure with marginals $\mu_{X}$ and $\mu_{A}$ on the structure and the features respectively (right) \label{graphex}}
\end{figure}
In many applications in machine learning, objects are readily endowed with a notion of distance between their points, hence defining metric spaces. As such, structured objects can be viewed as couples of $X \times A$ where $A$ is a subset of some metric space $(\Omega,d)$ representative of the feature space and $(X,d_{X})$ is a metric space representative of the structure, with $d_{X}$ the distance between elements of the space, modeling the intrinsic relationships between points of the structured object. 

In this paper, we propose to enrich the previous definition of a structured object with a (fully supported) probability measure which serves the purpose of signaling the relative importance of the object's elements. For example, we can add weights $(h_{i})_{i} \in \Sigma_{n}$ to each node in the graph defined previously. This way, we have created a fully supported probability measure $\mu= \sum_{i} h_{i} \delta_{(x_{i},a_{i})}$ which includes all the structured object information (see Fig. \ref{graphex}).

These considerations lead to the following and formal definition of a \emph{structured object} and the space it belongs to:

\begin{definition}{Structured objects.}

\noindent 
A structured object over a metric space $(\Omega,d)$ is the triplet $(X \times A, d_{X}, \mu)$, where $(X,d_{X})$ is a compact metric space, $A$ is a compact of $\Omega$ and $\mu$ is a fully supported probability measure over $X \times A$. $(\Omega,d)$ is denoted as the \emph{feature space}, $A$ is the \emph{feature information} of the structured object and $(X,d_{X})$ its \emph{structure information}.
\end{definition}

\begin{definition}{Space of structured objects.}

\noindent A structured object over $(\Omega,d)$ (simply denoted as $\Omega$) is an element of the following space:
$$H(\Omega)= \{(X \times A, d_{X}, \mu) | (X,d_{X}) \in \mathbb{X},A \in C(\Omega),\mu \in \overline{P}(X \times A) \}$$
where $C(\Omega)$ is the set of all compact subsets of $\Omega$, $\mathbb{X}$ the set of all compact metric spaces and $\overline{P}(X \times A)$ the set of fully supported probability measures on $X \times A$.
\end{definition}
We will note $\mu_{X}$ and $\mu_{A}$ the structure and feature (fully-supported) marginals of $\mu$. Those marginals encode a very partial information since they focus only on independent feature distributions or only on the structure. An example of $\mu$, $\mu_{X}$ and $\mu_{A}$ is provided for a labeled graph in Fig. \ref{graphex}. With this definition, the features of all structured objects are directly comparable since they live in the same ambient space $(\Omega,d)$. For the sake of simplicity, and when it is clear from the context, we will denote only  $\mu$ the whole structured object.

With this definition, we only consider the fully-supported case. Although mathematical results can be expanded to non fully supported measures, it leads to discussions about the support of the measures and for the sake of clarity, we omit here this extension. In the following paragraphs, $(X\times A, d_{X}, \mu)$ and $(Y \times B, d_{Y},\nu)$ are structured objects.

\subsection{Comparing structured objects}

We now aim to define a notion of equivalence between two structured objects. 
Intuitively, two structured objects are the same if they share the same feature information, if their structure information are lookalike and if the probability measures are corresponding in some sense. In this section, we present mathematical tools for comparing individually the elements of structured objects.

First, our formalism implies comparing metric spaces, which can be done \textit{via} the notion of \emph{isometry}.

\begin{definition}{Isometry \label{isometrydef}}

\noindent Let $(X,d_{X})$ and $(Y,d_{Y})$ be two metric spaces. An isometry is a sujective map $f : X \rightarrow Y$ that preserves the distances:
\begin{equation}
\label{isometryproperty}
\forall x,x' \in X, d_{Y}(f(x),f(x'))=d_{X}(x,x')
\end{equation}

\end{definition}

An isometry is bijective, since for $f(x)=f(x')$ we have $d_{Y}(f(x),f(x'))=0=d_{X}(x,x')$ and hence $x=x'$ (in the same way $f^{-1}$ is also a isometry).  When it exists, $X$ and $Y$ have the same size and any ``metric statement'' in the first space is ``transported'' to the second space by the isometry $f$.

\begin{example}
Let us consider the two following graphs whose discrete metric spaces are obtained as shortest path between the vertices (see corresponding graphs in Figure \ref{isometric_spaces}) $$\small\begin{pmatrix}x_{1}\\x_{2}\\x_{3}\\x_{4}\end{pmatrix},{\underbrace{\begin{pmatrix}0&1&1&1 \\
    1&0&1&2\\
    1&1&0&2\\
    1&2&2&0
    \end{pmatrix}}_{d_X(x_i,x_j)}} \quad\text{ and }\quad \begin{pmatrix}y_{1}\\y_{2}\\y_{3}\\y_{4}\end{pmatrix},{\small\underbrace{\begin{pmatrix}0&1&1&1 \\
        1&0&2&2\\
        1&2&0&1\\
        1&2&1&0
    \end{pmatrix}}_{d_Y(y_i,y_j)}}.$$  These spaces are isometric since $f(x_{1})=y_{1}$, $f(x_{2})=y_{3}$, $f(x_{3})=y_{4}$, $f(x_{4})=y_{2}$ verifies \eqref{isometryproperty}.

    \tikzstyle{vertex1}=[circle,fill=black,minimum size=6pt,inner sep=0pt]
\tikzstyle{vertex2}=[circle,fill=black,minimum size=6pt,inner sep=0pt]
\tikzstyle{vertex3}=[circle,fill=black,minimum size=6pt,inner sep=0pt]
\tikzstyle{vertex4}=[circle,fill=black,minimum size=6pt,inner sep=0pt]

    \begin{figure}[t]
    \centering
\tikzstyle{edge} = [draw]
\begin{tikzpicture}[scale=1, auto,swap]
    \foreach \pos/\name in {{(0,0)/a}}
        		\node[vertex1] (\name) at \pos {};
    \foreach \pos/\name in {{(2,0)/b}}
        		\node[vertex2] (\name) at \pos {};
    \foreach \pos/\name in {{(1,1)/c}}
        		\node[vertex3] (\name) at \pos {};
    \foreach \pos/\name in {{(2.5,1.5)/d}}
        		\node[vertex4] (\name) at \pos {};
\def\x{5}

    \foreach \pos/\name in {{(0+\x,0)/e}}
        		\node[vertex1] (\name) at \pos {};
    \foreach \pos/\name in {{(2+\x,0)/f}}
        		\node[vertex2] (\name) at \pos {};
    \foreach \pos/\name in {{(1+\x,1)/g}}
        		\node[vertex3] (\name) at \pos {};
    \foreach \pos/\name in {{(2.5+\x,1.5)/h}}
        		\node[vertex4] (\name) at \pos {};

    \foreach \source/ \dest in {b/a, c/a, c/b, c/d}
        \path[edge] (\source) -- (\dest);

    \foreach \source/ \dest in {e/g, g/f, g/h, h/f}
        \path[edge] (\source) -- (\dest);

\def\y{0.1}
\foreach \pos/ \name in {{(0,00-\y)/x_2}, {(2,00-\y)/x_3}, {(2.5,1.50-\y)/x_4}}
  \draw \pos node[below, scale = 1]{$\name$};
\draw (1,1) node[left, scale = 1]{$x_1$};

\foreach \pos/ \name in {{(0+\x,0-\y)/y_2}, {(2+\x,00-\y)/y_3}}
  \draw \pos node[below, scale = 1]{$\name$};
\draw (1+\x,1) node[left, scale = 1]{$y_1$};
\draw (2.5+\x,1.5) node[right, scale = 1]{$y_4$};

 \end{tikzpicture}
  \caption{Two isometric metric spaces. Distances between the nodes are given by the shortest path, and the weight of each edge is equal to 1.\label{isometric_spaces}}
     \end{figure}
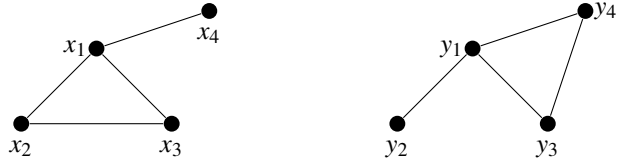
\end{example}

The previous definition can be used in order to compare the structure information of two structured objects.  Regarding the feature information, since they all lie in the same ambient space $\Omega$, a natural way for comparing them is by the standard equality $A=B$. Finally, in order to compare measures on different spaces, the notion of \emph{preserving application} can be used.

\begin{definition}{Preserving application\label{preservingdef}
}

\noindent Let $\Omega_{1},\mu_1 \in P(\Omega_{1})$ and $\Omega_{2},\mu_2 \in P(\Omega_{2})$ be two measurable spaces. An application $f: \Omega_{1} \rightarrow \Omega_{2}$ is said to be \emph{measure preserving} if it transports the measure $\mu_{1}$ on $\mu_{2}$ such that $$f\#\mu_{1}=\mu_{2}.$$
If there exists such a measure preserving map, the properties about measures of $\Omega_{1}$ are transported via $f$ to $\Omega_{2}$ .
\end{definition}

Let us now consider a measurable metric space (denoted mm-space), \textit{i.e.} a metric space $(X,d_{X})$ enriched with a probability measure and described by a triplet $(X,d_{X},\mu_{X} \in \overline{P}(X))$. An interesting notion for comparing mm-spaces is the notion of \textit{isomophism}.

 \begin{definition}{Isomorphism.}

\noindent Two mm-spaces $(X,d_{X},\mu_{X}),(Y,d_{Y},\mu_{Y})$ are isomorphic if there exists a measure preserving isometry $f: X \rightarrow Y$ between them.
\end{definition}

\begin{example}
    Let us consider two mm-spaces $(X=\{x_1,x_2\},d_{X}=\{1\},\mu_{X}=\{\frac{1}{2},\frac{1}{2}\})$ and $(Y=\{y_1,y_2\},d_{Y}=\{1\},\mu_{Y}=\{\frac{1}{4},\frac{3}{4}\})$ as depicted in Figure \ref{isometric_not_isomoprhic}. These spaces are isometric but not isomorphic as there exists no measure preserving application between them.

    \begin{figure}[t]
    \centering
\tikzstyle{vertex1}=[circle,fill=black,minimum size=6pt,inner sep=0pt]
\tikzstyle{vertex2}=[circle,fill=black,minimum size=3pt,inner sep=0pt]
\tikzstyle{vertex3}=[circle,fill=black,minimum size=9pt,inner sep=0pt]

\tikzstyle{edge} = [draw]
\begin{tikzpicture}[scale=1, auto,swap]
    \foreach \pos/\name in {{(0,0)/a}}
        		\node[vertex2] (\name) at \pos {};
    \foreach \pos/\name in {{(2,0)/b}}
        		\node[vertex3] (\name) at \pos {};
\def\x{5}

    \foreach \pos/\name in {{(0+\x,0)/e}}
        		\node[vertex1] (\name) at \pos {};
    \foreach \pos/\name in {{(2+\x,0)/f}}
        		\node[vertex1] (\name) at \pos {};

    \foreach \source/ \dest in {b/a, e/f}
        \path[edge] (\source) -- (\dest);

\def\y{0.1}
\foreach \pos/ \name in {{(0,00-\y)/x_1}, {(2,00-\y)/x_2}, {(0+\x,0-\y)/y_1}, {(2+\x,00-\y)/y_2}}
  \draw \pos node[below, scale = 1]{$\name$};
\def\y{-0.1}
\foreach \pos/ \name in {{(0,00-\y)/\frac{1}{4}}, {(2,00-\y)/\frac{3}{4}}, {(0+\x,0-\y)/\frac{1}{2}}, {(2+\x,00-\y)/\frac{1}{2}}}
  \draw \pos node[above, scale = 1]{$\name$};

 \end{tikzpicture}
    \caption{Two isometric but not isomorphic spaces. \label{isometric_not_isomoprhic}}
    \end{figure}
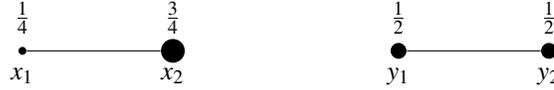
\end{example}

All this considered, we can now define a notion of equivalence between structured objects.

\begin{definition}{Equivalence of structured objects.\label{eqstructobjects}
}

\noindent Two structured objects are said to be \emph{equivalent} (and the equivalence relation is denoted $\sim$ ) if there exists an application $f : X \times A \rightarrow Y \times B$ such that $f$ is measure-preserving ($f\#\mu=\nu$) and if one defines $f_{1}: X \rightarrow Y$ and $f_{2}: A \rightarrow B$ such that $f(x,a)=(f_{1}(x),f_{2}(a))$  then  $f_{1}$ is an isometry and $f_{2}=I_{d}$ ( \textit{i.e} $A=B$).

\end{definition}

It is clear that it defines an equivalence relation over $H(\Omega)$.

\begin{example}
    To illustrate this definition, we consider a simple example of two structured objects:

$$\underbrace{\begin{pmatrix}(x_{1},a_{1})\\(x_{2},a_{2})\\(x_{3},a_{3})\\(x_{4},a_{4})\end{pmatrix}}_{x_i,a_i},{\underbrace{\begin{pmatrix}0&1&1&1 \\
    1&0&1&2\\
    1&1&0&2\\
    1&2&2&0
    \end{pmatrix}}_{d_X(x_i,x_j)}} ,\underbrace{\begin{pmatrix}\nicefrac{1}{4}\\\nicefrac{1}{4}\\\nicefrac{1}{4}\\\nicefrac{1}{4}\end{pmatrix}}_{h_i}\quad \text{and}\quad
    \underbrace{\begin{pmatrix}(y_{1},b_{1})\\(y_{2},b_{2})\\(y_{3},b_{3})\\(y_{4},b_{4})\end{pmatrix}}_{y_{i},b_i},{\small\underbrace{\begin{pmatrix}0&1&1&1 \\
        1&0&2&2\\
        1&2&0&1\\
        1&2&1&0
    \end{pmatrix}}_{d_Y(y_i,y_j)}},\underbrace{\begin{pmatrix}\nicefrac{1}{4}\\\nicefrac{1}{4}\\\nicefrac{1}{4}\\\nicefrac{1}{4}\end{pmatrix}}_{h'_i}
        $$

with $\forall i, a_{i}=b_{i}$ and $\forall i \neq j, a_{i} \neq a_{j}$ (see Figure \ref{equivalent_objects}). The two structured objects have isometric structures and same features individually but they are not equivalent. One possible application $f=(f_{1},f_{2}) : X \times A \rightarrow Y \times B$ such that $f_{1}$ is an isometry is $f(x_{1},a_{1})=(y_{1},b_{1})$, $f(x_{2},a_{2})=(y_{3},b_{3})$, $f(x_{3},a_{3})=(y_{4},b_{4})$, $f(x_{4},a_{4})=(y_{2},b_{2})$. Yet this application does not verifies $f_{2}=I_{d}$ since $f_{2}(a_{2})=b_{3}$ and $a_{2} \neq b_{3}$. The other possible applications such that $f_{1}$ is an isometry are simple permutation of this example, yet it is easy to check that none of them verifies $f_{2}=I_{d}$ (for example with $f(x_{2},a_{2})=(y_{4},b_{4})$).

 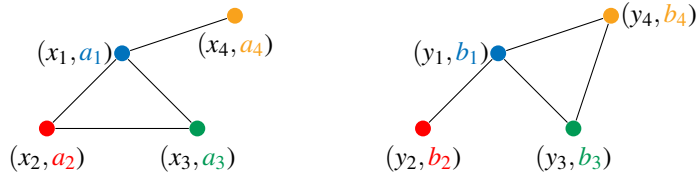
\begin{figure}[t]
    \centering
\tikzstyle{vertex1}=[circle,fill=red,minimum size=6pt,inner sep=0pt]
\tikzstyle{vertex2}=[circle,fill=ForestGreen,minimum size=6pt,inner sep=0pt]
\tikzstyle{vertex3}=[circle,fill=RoyalBlue,minimum size=6pt,inner sep=0pt]
\tikzstyle{vertex4}=[circle,fill=YellowOrange,minimum size=6pt,inner sep=0pt]

\tikzstyle{edge} = [draw]
\begin{tikzpicture}[scale=1, auto,swap]
    \foreach \pos/\name in {{(0,0)/a}}
        		\node[vertex1] (\name) at \pos {};
    \foreach \pos/\name in {{(2,0)/b}}
        		\node[vertex2] (\name) at \pos {};
    \foreach \pos/\name in {{(1,1)/c}}
        		\node[vertex3] (\name) at \pos {};
    \foreach \pos/\name in {{(2.5,1.5)/d}}
        		\node[vertex4] (\name) at \pos {};
\def\x{5}

    \foreach \pos/\name in {{(0+\x,0)/e}}
        		\node[vertex1] (\name) at \pos {};
    \foreach \pos/\name in {{(2+\x,0)/f}}
        		\node[vertex2] (\name) at \pos {};
    \foreach \pos/\name in {{(1+\x,1)/g}}
        		\node[vertex3] (\name) at \pos {};
    \foreach \pos/\name in {{(2.5+\x,1.5)/h}}
        		\node[vertex4] (\name) at \pos {};

    \foreach \source/ \dest in {b/a, c/a, c/b, c/d}
        \path[edge] (\source) -- (\dest);

    \foreach \source/ \dest in {e/g, g/f, g/h, h/f}
        \path[edge] (\source) -- (\dest);

\def\y{0.1}
\foreach \pos/ \name  in {{(0,00-\y)/(x_2, \color{red}{a_2}}, {(2,00-\y)/(x_3, \color{ForestGreen}{a_3}}, {(2.5,1.50-\y)/(x_4, \color{YellowOrange}{a_4}}}
  \draw \pos node[below, scale = 1]{$\name$)};
\draw (1,1) node[left, scale = 1]{$(x_1, \color{RoyalBlue}{a_1}$)};

\foreach \pos/ \name in {{(0+\x,0-\y)/(y_2, \color{red}{b_2}}, {(2+\x,00-\y)/(y_3, \color{ForestGreen}{b_3}}}
  \draw \pos node[below, scale = 1]{$\name$)};
\draw (1+\x,1) node[left, scale = 1]{$(y_1, \color{RoyalBlue}{b_1}$)};
\draw (2.5+\x,1.5) node[right, scale = 1]{$(y_4, \color{YellowOrange}{b_4}$)};

 \end{tikzpicture}
    \caption{Two structured objects with isometric structures and identical features that are not equivalent. The color of the nodes represent the node feature and each edge represents a distance of 1 between the connected nodes. \label{equivalent_objects}}
    \end{figure}
\end{example}

\subsection{Background on OT distances}

The Optimal Transport (OT) framework defines useful distances between probability measures that describe either the feature or the structure information of structured objects.

\paragraph{Wasserstein distance}

When the probability measures live in the same metric space $(\Omega, d)$, the quantity:
\begin{equation}
d^{\Omega}_{W,p}(\mu_{A},\nu_{B})=\bigg(\underset{\pi \in \Pi(\mu_{A},\nu_{B})}{\text{inf}} \losswass(\pi)\bigg)^{\frac{1}{p}}
\label{otw}
\end{equation}

where $$\losswass(\pi)=\underset{A \times B}{\int} d(a,b)^{p} d\pi(a,b)$$
is usually called the $p$-Wasserstein distance (also known  with $p=1$  as Earth Mover's distance \cite{Rubner2000} in the computer vision comunity) between distributions $\mu_{A}$ and $\nu_{B}$.

\begin{figure}[t]
\label{wassex}
\centering
\includegraphics[width=0.60\textwidth]{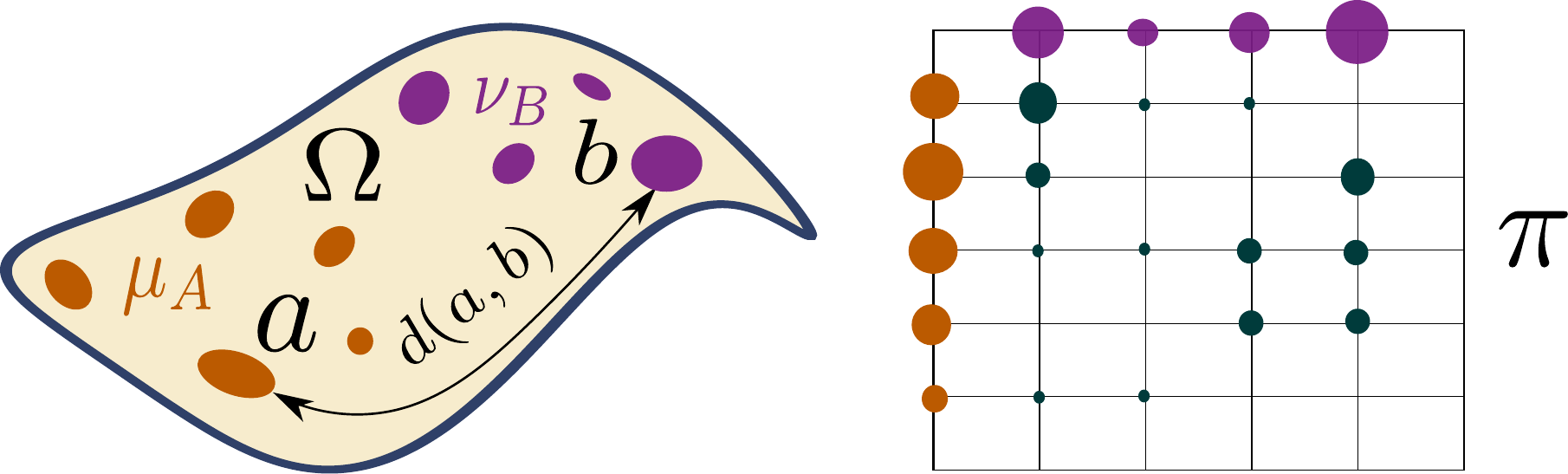}
\caption{Example of coupling between two discrete measures on the same ground space equipped with a distance $d$ that will define the Wasserstein distance. Left: the discrete measures on $\Omega$. Right: One possible coupling between these measures which respects the mass conservation. Image inspired from \cite[Fig 2.6]{peyre2017computational} \label{wassex}}
\end{figure}

Optimal transport theory defines a distance on probability measures such that $d^{\Omega}_{W,p}(\mu_{A},\nu_{B})=0$ \textit{iff} $\mu_{A}=\nu_{B}$. This distance also has a nice geometrical interpretation as it represents the optimal cost \emph{w.r.t.} $d$  to move the measure $\mu_{A}$ onto $\nu_{B}$ with $\pi(a,b)$ the amount of probability mass shifted from $a$ to $b$ (see figure \ref{wassex})
To this extent, the Wasserstein distance quantifies how ``far'' $\mu_{A}$ is from $\nu_{B}$ by measuring how ``difficult'' it is to move all the mass from $\mu_{A}$ onto $\nu_{B}$. Optimal transport can deal with both smooth and discrete measures and has proved very useful for comparing distributions in a shared space but with different (and even non-overlapping) supports.

\paragraph{Gromov Wasserstein distance}

In order to compare measures that are not necessarily in the same ambient space, \cite{Sturm2006,Memoli:2004:CPC:1057432.1057436} define an OT-like distance. By relaxing the classical Hausdorff distance~\cite{Memoli:2004:CPC:1057432.1057436,Villani} that is untractable in practice, authors build a distance over the space of all metric spaces. For two compact mm-spaces $\mathcal{X}=(X,d_{X},\mu_{X} \in \overline{P}(X))$ and $\mathcal{Y}=(Y,d_{Y},\nu_{Y} \in \overline{P}(Y))$, the Gromov-Wasserstein distance is defined as:
\begin{equation}
d_{GW,p}(\mu_{X},\nu_{Y})=\big( \ \underset{\pi \in \Pi(\mu_X, \nu_Y)}{\text{inf}} \lossgromov(\pi)\ \big)^{\frac{1}{p}}
\label{otgw}
\end{equation}
where  $$\lossgromov(\pi)= \underset{X\times Y \times X\times Y}{\int} L(x,y,x',y')^{p} d\pi(x,y)d\pi(x',y')$$ with $$L(x,y,x',y')=|d_{X}(x,x')-d_{Y}(y,y')|$$

Note that, with some abuse of notation, we denote the entire mm-space by its probability measure and that the Gromov-Wasserstein distance depends on the choice of the metrics $d_{X}$ and $d_{Y}$. When it is not clear from the context we will denote by $\lossgromov(d_{X},d_{Y},\pi)$ the Gromov-Wasserstein loss. The resulting coupling tends to associate pairs of points with similar distances within each pair (see figure~\ref{gromex}). The Gromov-Wasserstein distance defines a metric over the space of all metric spaces quotiented by measure-preserving isometries (see def~\ref{isometrydef} and~\ref{preservingdef}), thus allowing the comparison of measures over different ground spaces.   This distance has been used for shape comparison in~\cite{springerlink:10.1007/s10208-011-9093-5} and is invariant to rotations and translations in either spaces.

\begin{figure}[t]
\label{gwex}
\centering
\includegraphics[width=0.6\textwidth]{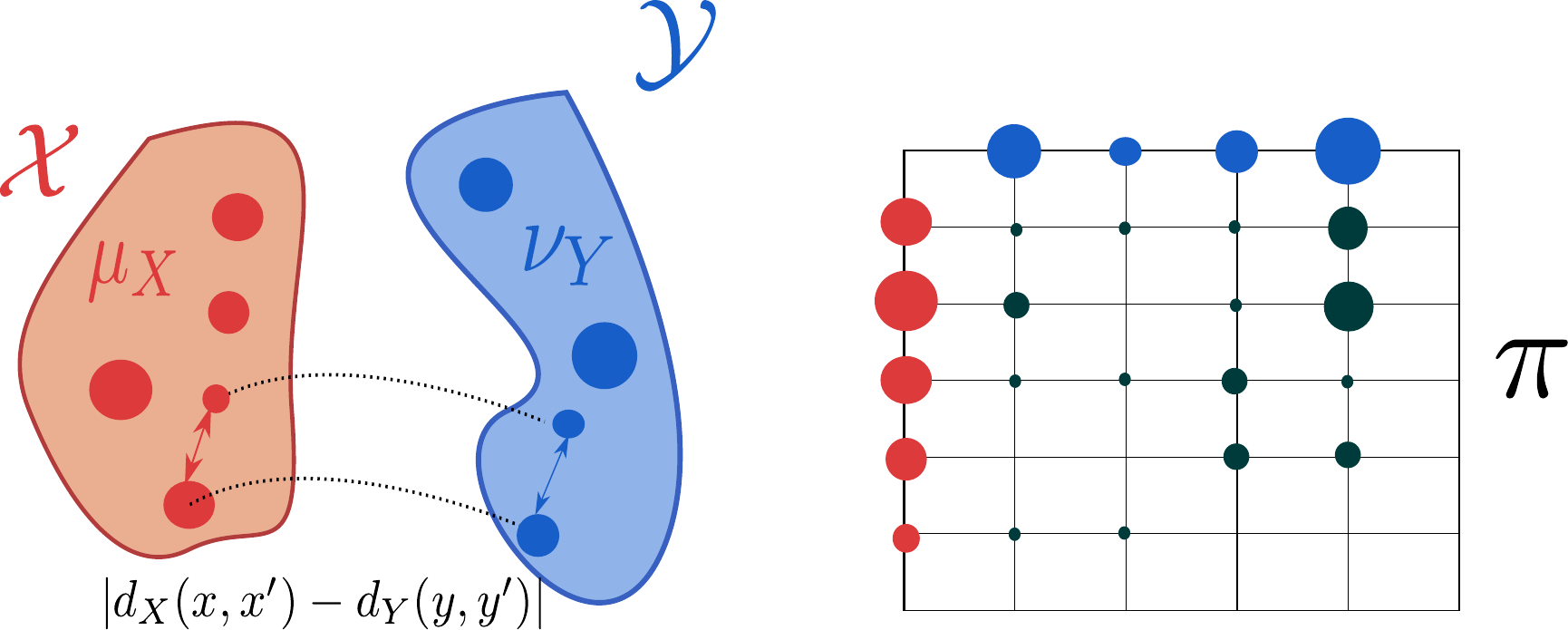}
\caption{Gromov-Wasserstein coupling of two mm-spaces $\mathcal{X}=(X,d_{X},\mu_{X})$ and $\mathcal{Y}=(Y,d_{Y},\nu_{Y})$. Left: the mm-spaces share nothing in common. Similarity between pairwise distances is measured by $|d_{X}(x,x')-d_{Y}(y,y')|$. Right: an admissible coupling of $\mu_{X}$ and $\mu_{Y}$. Image inspired from~\cite[Fig 10.8]{peyre2017computational} \label{gromex}}
\end{figure}

\paragraph{Some adaptations of $W$ and $GW$}

Despite the appealing properties of both Wasserstein and Gromov-Wasserstein distances, they fail at comparing structured objects as originally defined by focusing only on the feature and structure marginals respectively. However, with some hypotheses, one could adapt these distances for structured objects.

If the structure spaces are part of a same ground space $(Z,d_{Z})$, one can build a distance $\hat{d}$ between couples $(x,a)$ and $(y,b)$ and apply the Wasserstein distance so as to compare the two structured objects. In this case, when the Wasserstein distance vanishes it implies that the structured objects are equal in the sense of equality between the structures and the features respectively ($X=Y$ and $A=B$). This approach is very related with the one discussed in \cite{Thorpe2017} where authors define the Transportation $L^{p}$ distance for signal analysis purposes. Their approach can be viewed as a transport between two joint measures $\mu(X\times A)=\lambda(\{x \ \text{s.t} \ x \in X \subset Z=\mathbb{R}^{d}; \ f(x) \in A \subset \mathbb{R}^{m}\})$, $\nu(Y\times B)=\lambda(\{y \ \text{s.t} \ y \in Y \subset Z=\mathbb{R}^{d}; \ g(y) \in B \subset \mathbb{R}^{m}\})$ for function $f,g : Z \rightarrow \mathbb{R}^{m}$ representative of the signal values and $\lambda$ the Lebesgue measure. The distance for the transport is defined as $\hat{d}((x,f(x)),(y,g(y)))= \frac{1}{\alpha} \|x-y\|_{p}^{p} + \|f(x)-g(y)\|_{p}^{p}$ for $\alpha > 0$ and $\|\cdot\|_{p}$ the $l_{p}$ norm. In this case $f(x)$ and $g(y)$ can be interpreted as encoding the feature information of the signal while $x,y$ encode its structure information. In contrast to the $FGW$ approach, invariants of this approach are the feature \emph{and} structure preserving applications from $X\times A$ to $Y \times B$ whereas the invariants for the $FGW$ distance are the feature preserving applications that are isometries in the structure space (as seen further in theorem \ref{Topology}).

The Gromov-Wasserstein distance can also be adapted to structured objects by considering for example the distances $d_{X} \oplus d$ and $d_{Y} \oplus d$ within each space $X \times A$ and $Y \times B$ respectively. When the resulting distance vanishes, structured objects are isomorphic with respect to $d_{X} \oplus d$ and $d_{Y} \oplus d$, yet resulting on a weaker result than when $FGW$ vanishes. Indeed, as seen in theorem \ref{Topology} $FGW$ is null \textit{iff} the structured objects are equivalent and in such case $(X \times A,d_{X} \oplus d,\mu)$ and $(Y \times B,d_{Y} \oplus d,\nu)$ are \textit{de facto} isomorphic. However the converse is not necessarily true. For example in Fig. \ref{equivalent_objects} the structures are isometric and the distances between the features within each space are the same between each structured objects so $(X \times A,d_{X} \oplus d,\mu)$ and $(Y \times B,d_{Y} \oplus d,\nu)$ are isomorphic, yet not equivalent as shown in the example.

\subsection{Fused Gromov-Wasserstein distance}

Building on both Gromov-Wasserstein and Wasserstein distances we define the Fused Gromov-Wasserstein ($FGW$) distance on $H(\Omega)$:

\begin{definition}{Fused Gromov-Wasserstein distance.\label{fgwdef}}

\noindent The Fused-Gromov-Wasserstein distance is defined for $\alpha \in [0,1]$ and $p,q \geq 1$ as:
\begin{equation}
d^{\Omega}_{FGW,\alpha,p,q}(\mu,\nu)=\big( \ \underset{\pi \in \Pi(\mu,\nu)}{\text{inf}}  \lossfgw(\pi)\ \big)^{\frac{1}{p}}
\label{otfgx}
\end{equation}
where
$$\lossfgw(\pi){=}\int\limits_{(X\times A \times Y \times B)^{2}} \big((1-\alpha) d(a,b)^{q} +\alpha L(x,y,x',y')^{q} \big)^{p} \,d\pi((x,a),(y,b))\,d\pi((x',a'),(y',b'))$$

\end{definition}

\begin{figure}[t]
\label{space}
\centering
\includegraphics[width=0.4\textwidth]{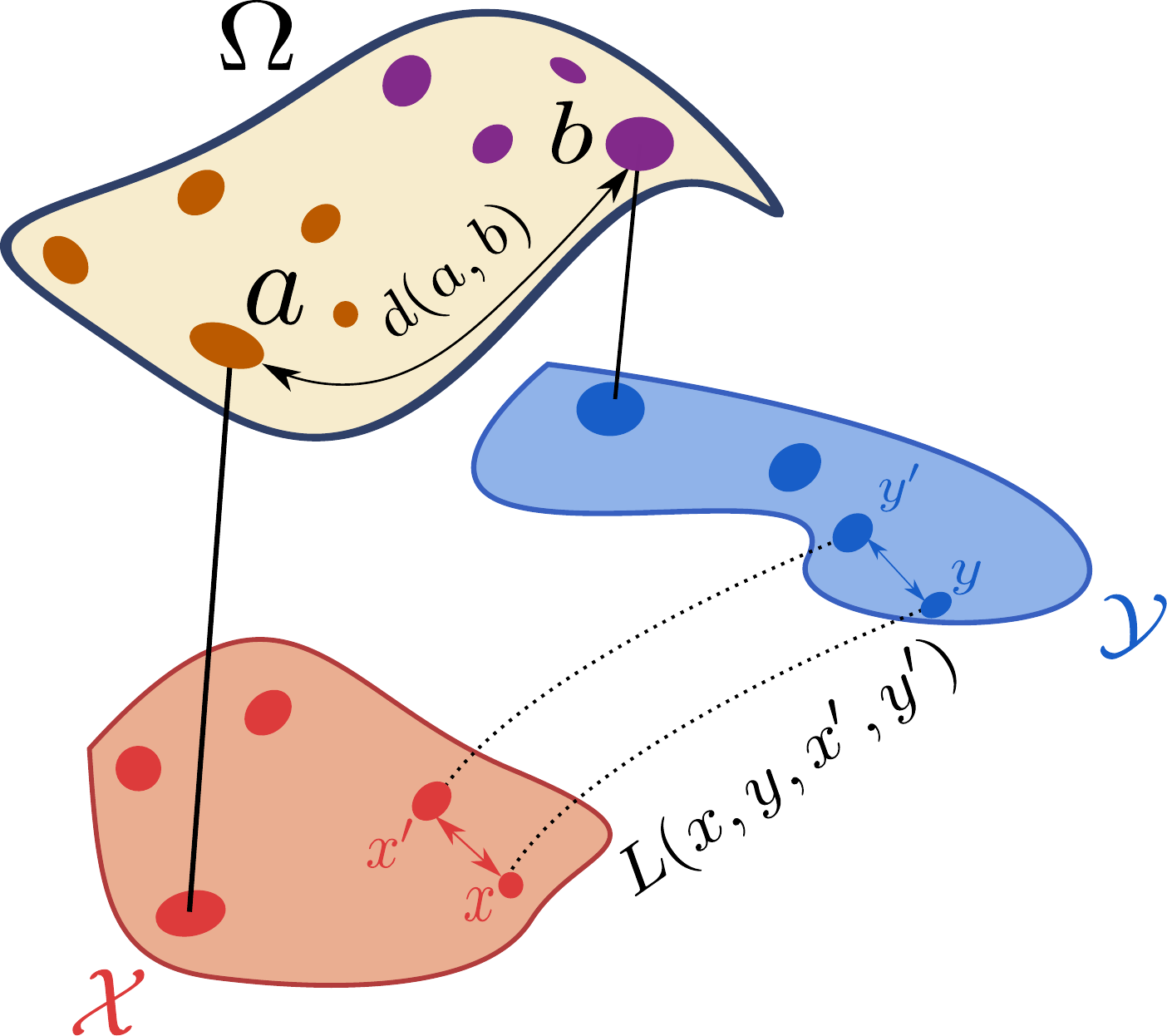}
\caption{Illustration of the definition \ref{fgwdef}. The figure shows two structured objects $(X\times A,d_{X},\mu)$ and $(Y \times B,d_{Y},\mu)$. The feature space $\Omega$ is the common space for all features. The two metric spaces $(X,d_{X})$ and $(Y,d_{Y})$ represent the structures of our two structured objects, the similarity between all pair to pair distances of the structure points is measured by $L(x,y,x',y')$. $\mu$ and $\nu$ are the joint measures on the structure space and the feature space. \label{fgwex}}
\end{figure}

This definition is illustrated in Figure \ref{fgwex}. $\alpha$ acts as a trade-off parameter between the structure term represented by $L(x,y,x',y')$ and the feature term $d(a,b)$. In this way, the convex combination of both terms leads to the use of both information in a single formalism resulting on a single map $\pi$ that behaves as the optimal map with respect to the structure and the feature costs in order to ``move'' the mass from the one joint probability measure to the other.

In contrast to the work presented in \cite{2018arXiv180509114V} where the trade-off is defined \textit{via} $d(a,b)^{q}+ \alpha L(x,y,x',y')^{q}$  for $\alpha \in [0,\infty[$ we rather consider a convex combination  $(1-\alpha)d(a,b)^{q}+ \alpha L(x,y,x',y')^{q}$ for $\alpha \in [0,1]$. Both formulations are strictly equivalent since any optimal plan \textit{w.r.t} the cost with the convex combination leads to an optimal plan \textit{w.r.t} for the other cost and conversely (see \ref{deffgweq} for more details). However the definition with the convex combination of a feature and structure cost carries out more theoretical properties such as the interpolation (theorem \ref{interpolationtheorem}).

Many desirable properties arise from this definition. Among them, one can define a topology over the space of structured objects using the $FGW$ distance to compare structured objects, in the same philosophy as for Wasserstein and Gromov-Wasserstein distances. The definition also implies that $FGW$ acts as a generalization of both Wasserstein and Gromov-Wasserstein distances, with $FGW$ achieving an interpolation of these two distances. More remarkably, $FGW$ distance also realizes geodesic properties over the space of structured objects, allowing the definition of gradient flows. All these properties are detailed in the next section, and before reviewing them, we first compare $FGW$ with $GW$ and $W$ distances and state the following proposition (by assuming for now that $FGW$ exists, which will be shown later).

\begin{proposition}{Comparaison between $FGW$, $GW$ and $W$.}
\label{fgwcomparetogwandw}
\begin{itemize}
\item The following inequalities hold:
\begin{equation}
\label{wassinequality}
d^{\Omega}_{FGW,\alpha,p,q}(\mu,\nu) \geq (1-\alpha)d^{\Omega}_{W,pq}(\mu_{A},\nu_{B})^{q}
\end{equation}
\begin{equation}
\label{gromovinequality}
d^{\Omega}_{FGW,\alpha,p,q}(\mu,\nu) \geq \alpha d_{GW,pq}(\mu_{X},\nu_{Y})^{q}
\end{equation}
\item Let us suppose that the structure spaces $(X,d_{X})$,$(Y,d_{Y})$ are part of a single ground space $(Z,d_{Z})$ (\textit{i.e.} $d_{X}=d_{Y}=d_{Z}$). We consider the Wasserstein distance between $\mu$ and $\nu$ (well defined in this case) for the distance on $Z \times \Omega$ : $\tilde{d}((x,a),(y,b))=(1-\alpha)d(a,b)+\alpha d_{Z}(x,y)$. Then:
\begin{equation}
\label{samespace}
d^{\Omega}_{FGW,\alpha,p,1}(\mu,\nu)^{p} \leq 2d^{Z \times \Omega}_{W,p}(\mu,\nu)^{p}
\end{equation}

\end{itemize}

\end{proposition}

In particular, following this proposition, when the $FGW$ distance is null then both $GW$ and $W$ distances vanish so that the structure and the feature of the structure object are individually ``the same'' (with respect to their corresponding equivalence relation). However the converse is not necessarily true as shown in the following example.

\paragraph{Toy trees \label{toytreesex}} We construct two trees as illustrated in Figure \ref{mapstoy} where the 1D node features are shown with colors. The shortest path between the nodes is used to capture the structures of the two structured objects and the euclidean distance is used for the features. Figure \ref{mapstoy} illustrates the differences between $FGW$, $GW$ and $W$ distances. The left part is Wasserstein distance between the features: red nodes are transported on red ones and the blue nodes on the blue ones but tree structures are completely discarded. In this case, the Wasserstein distance vanishes. In the right part, we compute the Gromov-Wasserstein distance between structures: all couples of points are transported to another couple of points, which enforces the matching of tree structures without taking into account the features. Since structures are isometric, the Gromov-Wasserstein distance is null. Finally, we compute the $FGW$ using intermediate $\alpha$ (center), the bottom and first level structure is preserved as well as the feature matching (red on red and blue on blue) and $FGW$ discriminates the two structured objects.

\tikzstyle{vertex}=[circle,fill=black,minimum size=4pt,inner sep=0pt]
\tikzstyle{edge} = [draw]
\tikzstyle{transp} = [draw, dashed]
\tikzstyle{feuille1}=[rectangle,draw,fill=blue,text=blue,minimum size=4pt,inner sep=0pt]
\tikzstyle{feuille2}=[circle,draw,fill=red,text=blue,minimum size=4pt,inner sep=0pt]
\tikzstyle{entour}=[ellipse,draw,text=blue]

\begin{figure}[t]
\centering
\begin{tikzpicture}[scale=0.3, auto,swap]
    \foreach \pos/\name in {{(0,4)/a}, {(1,2)/c}, {(1,6)/b},
                             {(2,7)/d}, {(2,5)/e}, {(2,3)/f}, {(2,1)/g}}
        		\node[vertex] (\name) at \pos {};
	    \foreach \pos/\name in {
                             {(3,7.5)/h},  {(3,6.5)/i},  {(3,3.5)/l}, {(3,2.5)/m}}
        		\node[feuille1] (\name) at \pos {};
      \foreach \pos/\name in {
                             {(3,5.5)/j},  {(3,4.5)/k},  {(3,0.5)/n}, {(3,1.5)/o}}
       		\node[feuille2] (\name) at \pos {};
                     \foreach \source/ \dest in {b/a, c/a, d/b,e/b, f/c, g/c, d/h, d/i, e/j, e/k, f/l, f/m, g/n, g/o}
        \path[edge] (\source) -- (\dest);

   \draw[black] (4,0) grid[step=2](12,8);
  \draw[black, thick] (4,0) grid[step=4](12,8);

\draw[fill=black!40] (8,7) rectangle(9,8);
\draw[fill=black!40] (4,6) rectangle(5,7);
\draw[fill=black!40] (11,5) rectangle(12,6);
\draw[fill=black!40] (7,4) rectangle(8,5);
\draw[fill=black!40] (10,3) rectangle(11,4);
\draw[fill=black!40] (6,2) rectangle(7,3);
\draw[fill=black!40] (9,1) rectangle(10,2);
\draw[fill=black!40] (5,0) rectangle(6,1);

    \foreach \pos/\name in {{(8,12)/a}, {(6,11)/b}, {(10,11)/c},
                             {(5,10)/d}, {(7,10)/e}, {(9,10)/f}, {(11,10)/g}}
        		\node[vertex] (\name) at \pos {};
	    \foreach \pos/\name in {
                             {(4.5,9)/h},  {(6.5,9)/j},  {(8.5,9)/l}, {(10.5,9)/n}}
        		\node[feuille1] (\name) at \pos {};
      \foreach \pos/\name in {
                             {(5.5,9)/i},  {(7.5,9)/k},  {(9.5,9)/m}, {(11.5,9)/o}}
       		\node[feuille2] (\name) at \pos {};
    \foreach \source/ \dest in {b/a, c/a, d/b,e/b, f/c, g/c, d/h, d/i, e/j, e/k, f/l, f/m, g/n, g/o}
        \path[edge] (\source) -- (\dest);

\draw  (8,-1) node {$d_{W}=0$};

   \draw[black] (12,0) grid[step=2, ,xshift=1cm](20,8);
  \draw[step=4,black, thick, xshift=1cm] (12,0) grid (20,8);
\draw  (17,-1) node {$d_{FGW}>0$};

    \foreach \pos/\name in {{(8+9,12)/a}, {(6+9,11)/b}, {(10+9,11)/c},
                             {(5+9,10)/d}, {(7+9,10)/e}, {(9+9,10)/f}, {(11+9,10)/g}}
        		\node[vertex] (\name) at \pos {};
	    \foreach \pos/\name in {
                             {(4.5+9,9)/h},  {(6.5+9,9)/j},  {(8.5+9,9)/l}, {(10.5+9,9)/n}}
        		\node[feuille1] (\name) at \pos {};
      \foreach \pos/\name in {
                             {(5.5+9,9)/i},  {(7.5+9,9)/k},  {(9.5+9,9)/m}, {(11.5+9,9)/o}}
       		\node[feuille2] (\name) at \pos {};
    \foreach \source/ \dest in {b/a, c/a, d/b,e/b, f/c, g/c, d/h, d/i, e/j, e/k, f/l, f/m, g/n, g/o}
        \path[edge] (\source) -- (\dest);

   \draw[black] (24,0) grid[step=2, ,xshift=-2cm](32,8);
  \draw[step=4,black, thick, xshift=-2cm] (24,0) grid (32,8);

\draw[fill=black!40] (17,7) rectangle(18,8);
\draw[fill=black!40] (19,6) rectangle(20,7);
\draw[fill=black!40] (18,5) rectangle(19,6);
\draw[fill=black!40] (20,4) rectangle(21,5);
\draw[fill=black!40] (15,3) rectangle(16,4);
\draw[fill=black!40] (13,2) rectangle(14,3);
\draw[fill=black!40] (14,1) rectangle(15,2);
\draw[fill=black!40] (16,0) rectangle(17,1);

    \foreach \pos/\name in {{(8+18,12)/a}, {(6+18,11)/b}, {(10+18,11)/c},
                             {(5+18,10)/d}, {(7+18,10)/e}, {(9+18,10)/f}, {(11+18,10)/g}}
        		\node[vertex] (\name) at \pos {};
	    \foreach \pos/\name in {
                             {(4.5+18,9)/h},  {(6.5+18,9)/j},  {(8.5+18,9)/l}, {(10.5+18,9)/n}}
        		\node[feuille1] (\name) at \pos {};
      \foreach \pos/\name in {
                             {(5.5+18,9)/i},  {(7.5+18,9)/k},  {(9.5+18,9)/m}, {(11.5+18,9)/o}}
       		\node[feuille2] (\name) at \pos {};
    \foreach \source/ \dest in {b/a, c/a, d/b,e/b, f/c, g/c, d/h, d/i, e/j, e/k, f/l, f/m, g/n, g/o}
        \path[edge] (\source) -- (\dest);
\draw  (26,-1) node {$d_{GW}=0$};

\draw[fill=black!40] (24,7) rectangle(25,8);
\draw[fill=black!40] (25,6) rectangle(26,7);
\draw[fill=black!40] (22,5) rectangle(23,6);
\draw[fill=black!40] (23,4) rectangle(24,5);
\draw[fill=black!40] (27,3) rectangle(28,4);
\draw[fill=black!40] (26,2) rectangle(27,3);
\draw[fill=black!40] (28,1) rectangle(29,2);
\draw[fill=black!40] (29,0) rectangle(30,1);

\end{tikzpicture}

\caption{Difference on transportation maps between $FGW$, $GW$ and $W$ distances on synthetic trees. On the left the $W$ distance between the features is nul since feature information are the same, on the middle the $FGW$ is different from zero and discriminate the two structured objects and on the right the $GW$ between the two isometric structures is nul. \label{mapstoy}}
\end{figure}
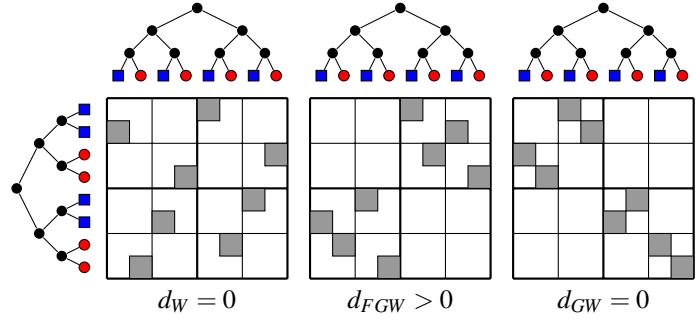

\section{Mathematical properties of $FGW$}

In this section, we establish some mathematical properties of the $FGW$ distance. The first result relates to the existence of the $FGW$ distance and to the topology of the space of structured objects. We then prove that the $FGW$ distance is indeed a distance regarding the equivalence relation between structured objects as defined in Defintion \ref{eqstructobjects}, allowing us to derive a topology on $H(\Omega)$.

\subsection{Topology of $(H(\Omega),d^{\Omega}_{FGW,\alpha,p,q})$}
\label{Topology}

The $FGW$ distance has the following properties:

\begin{theorem}{Metric properties}
\label{metrictheo}

\noindent Let $p,q\geq 1$ and $\alpha \in ]0,1[$, $\pi \rightarrow \lossfgw(\pi)$ always achieves a infimum $\pi^{*}$ in $\Pi(\mu,\nu)$ \textit{s.t} $d^{\Omega}_{FGW,\alpha,p,q}(\mu,\nu)=\lossfgw(\pi^{*})$ and:
\begin{itemize}
\item[$\bullet$] for all $q \geq 1$, $d^{\Omega}_{FGW,\alpha,p,q}(\mu,\nu)=0$ \emph{iff} there exists $ f=(f_{1},f_{2}): X\times A \rightarrow Y \times B$ with $f(x,a)=(f_{1}(x),f_{2}(a))$ such that:
\begin{equation}
\label{metrictheo:conservation}
f\#\mu=\nu
\end{equation}
\begin{equation}
\label{metrictheo:isometry}
f_{1}:X \rightarrow Y \text{ is an isometry}
\end{equation}
\begin{equation}
\label{metrictheo:conservationfeature}
f_{2}: A \rightarrow B \text{ is an identity map } (A=B).
\end{equation}

\item[$\bullet$] $d^{\Omega}_{FGW,\alpha,p,q}$ is symmetric and, for $q=1$, satisfies the triangle inequality. For $q\geq 2$, the triangular inequality is relaxed by a factor $2^{q-1}$ .
\end{itemize}
Proof of this theorem can be found in Section~\ref{allproofs}.
\end{theorem}

The previous theorem states that $FGW$ is a distance over $H(\Omega)$ quotiented by the measure preserving maps that are feature and structure preserving (through an isometry). Informally, invariants of the $FGW$ are objects that have both the same structure and the same features "in the same place". In other words, the $FGW$ distance vanishes \textit{iff} the structured objects are equivalent with respect to the equivalence relation $\sim$ defined in Definition \ref{eqstructobjects}.

Theorem~\ref{metrictheo} allows a wide set of applications for $FGW$ such as $k$-nearest-neighbors, distance-substitution kernels, pseudo-Euclidean embeddings, or representative-set methods \cite{distsubstitutionnkernel,BorgGroenen2005,2017arXiv170306476B}. Arguably, such a distance allows for a better interpretation than to end-to-end learning machines such as neural networks because the $\pi$ matrix exhibits the relationships between the elements of the objects.

The metric property naturally endows the structured object space with a notion of convergence as described in the next definition:

\begin{definition}{Convergence of structured objects.}

\noindent Let $(\mathbb{S}_{n})_{n \in \mathbb{N}}=\big((X_{n} \times A_{n},d_{X_{n}}, \mu_{n})\big)_{n \in \mathbb{N}}$ be a sequence of structured objects. $(\mathbb{S}_{n})_{n \in \mathbb{N}}$ is said to converge to $\mathbb{S}=(X \times A, d_{X},\mu)$ in the Fused Gromov-Wasserstein sense (denoted as $\mathbb{S}_{n} \underset{n \to +\infty}{\overset{FGW}{\longrightarrow}} \mathbb{S}$) if

$$\lim\limits_{n \rightarrow \infty}d^{\Omega}_{FGW,\alpha,p,1}(\mu_{n},\mu)=0$$

\end{definition}

Using Prop. \ref{fgwcomparetogwandw}, it is straightforward to see that, if $(\mathbb{S}_{n})_{n \in \mathbb{N}}$ converges to $\mathbb{S}$ in $FGW$ sense,  both the features and the structure of $(\mathbb{S}_{n})_{n \in \mathbb{N}}$ converge respectively in Wasserstein and Gromov-Wasserstein sense (see \cite{Memoli:2004:CPC:1057432.1057436} for the definition of convergence in the Gromov-Wasserstein sense).

An interesting question arises from this definition. Let us consider a structured object $\mathbb{S}=(X \times A,d_{X}, \mu )$ and let us sample the joint distribution so as to consider $(\mathbb{S}_{n})_{n \in \mathbb{N}}=\{(x_{i},a_{i})\}_{i \in \{1,..,n\} }, d_{X}, \mu_{n} )_{n \in \mathbb{N}}$ with $\mu_{n}=\frac{1}{n}\sum\limits_{i=1}^{n} \delta_{x_{i},a_{i}}$ where $(x_{i},a_{i}) \in X_{n}\times A_{n}$ are sampled from $\mu$. Does $(\mathbb{S}_{n})_{n \in \mathbb{N}}$ converge to $\mathbb{S}$ in the $FGW$ sense and how fast is the convergence?

To answer this question, we will use the theory developped in \cite{weedbach2017}. We recall the following definitions:

\begin{definition}{Upper Wasserstein dimension.}

\noindent Let $S$ be a subset of some polish metric space $\Omega$. The $\epsilon$-covering of $S$, denoted $N_{\epsilon}(S)$, is the minimum integer $m$ such that there exists closed balls $B_{1},...,B_{m}$ of diameter $\epsilon$ which cover $S$. More precisely, the balls verify $ S \subset \cup_{i=1}^{m} B_{i} $. The $\epsilon$-dimension of $S$ is defined by:
\begin{equation}
\label{epsdim}
\text{dim}_{\epsilon}(S)=\frac{\log(N_{\epsilon}(S))}{-\log(\epsilon)}.
\end{equation}
Given a measure $\mu$ on $\Omega$, we consider its ($\epsilon$-$\tau$) covering as the number $$N_{\epsilon,\tau}(\mu)=\text{inf}\{N_{\epsilon}(S) \ \text{\textit{s.t}} \ \mu(S) \geq 1- \tau \}$$ which represents the smallest $\epsilon$-covering of sufficiently ``large'' subsets (with respect to $\mu$).
The ($\epsilon$-$\tau$) dimension of $\mu$ is then defined as:
\begin{equation}
\label{epstaudim}
\text{dim}_{\epsilon}(\mu,\tau)=\frac{\log(N_{\epsilon,\tau}(\mu))}{-\log(\epsilon)}.
\end{equation}
The \emph{upper Wasserstein dimension} is defined by:
\begin{equation}
\label{upperwassdim}
d_{p}^{*}(\mu)=\text{inf}\{s \in [2p,\infty[ \ \text{\textit{s.t.}} \ \underset{\epsilon \rightarrow \infty}{\text{limsup}} \ \text{dim}_{\epsilon}(\mu,\epsilon^{\frac{sp}{s-2p}}) \leq s \}
\end{equation}
\end{definition}

This notion of dimension exists due to the monotonicity of $\text{dim}_{\epsilon}(\mu,\tau)$ and coincides with the intuitive notion of ``dimension''  when the measure is sufficiently well behaved. For example, for any absolutely continuous measure $\mu$ with respect to the Lebesgue measure on $[0,1]^{d}$, we have $d_{p}^{*}(\mu)=d$ for any $p \in [1,\frac{d}{2}]$. For more general cases see Prop.7 in \cite{weedbach2017}.

Using these definitions and the results in \cite{weedbach2017},  we answer the question of convergence of finite sample in the following proposition (proof can be found in Section~\ref{allproofs}) :

\begin{proposition}{Convergence of finite samples and a concentration inequality}
\label{concentration}

\noindent Let $p \geq 1$. We have:
$$\lim\limits_{n \rightarrow \infty}d^{\Omega}_{FGW,\alpha,p,1}(\mu_{n},\mu)=0.$$
Moreover, suppose that $s > d_{p}^{*}(\mu)$. Then:
\begin{equation}
\label{epstaudim}
\mathbb{E}[d^{\Omega}_{FGW,\alpha,p,1}(\mu_{n},\mu)] \underset{\sim}{<} n^{\frac{-1}{s}}.
\end{equation}
\end{proposition}

A particular case of this inequality is when $\alpha=1$ so that we can use the result above to derive a concentration result for the Gromov-Wassersten distance. More precisely, if $\nu_{n}=\frac{1}{n} \sum_{i} \delta_{x_{i}}$ denotes the empirical measure of $\nu \in \overline{P}(X)$ and if $s' > d_{p}^{*}(\nu)$ we have:
\begin{equation}
\label{concentrationgromov}
\mathbb{E}[d^{\Omega}_{GW,p}(\nu_{n},\nu)] \underset{\sim}{<} n^{\frac{-1}{s'}}.
\end{equation}
To the best of our knowledge, this is the first result about concentration for the Gromov-Wasserstein distance. In contrast to the Wasserstein distance case, it is not necessary sharp but it proves that considering the $GW$ and $FGW$ distances by sampling a continuous distribution makes sense as the finite samples concentrate around the expectation.

\subsection{Interpolation properties between Wasserstein and Gromov-Wasserstein distances}

In this section, we prove that the $FGW$ distance is a generalization of both Wasserstein and Gromov-Wasserstein distances in the sense that it achieves an interpolation between them. More precisely, we have the following theorem:

\begin{theorem}{Interpolation properties.}
\label{interpolationtheorem}

\noindent As $\alpha$ tends to zero, one recovers the Wasserstein distance between the feature information and as $\alpha$ goes to one, one recovers the Gromov-Wasserstein distance between the structure information :

Let  $\big((X,d_{X}), A,\mu \in \overline{P}(X \times A) \big),\big((Y,d_{Y}), B,\nu \in \overline{P}(Y \times B) \big) \in H(\Omega)^{2}$
\begin{equation}
\lim\limits_{\alpha \rightarrow 0}d^{\Omega}_{FGW,\alpha,p,q}(d_{X},d_{Y},\mu,\nu)=(d^{\Omega}_{W,qp}(\mu_{A},\nu_{B}))^{q}
\end{equation}

\begin{equation}
\lim\limits_{\alpha \rightarrow 1}d^{\Omega}_{FGW,\alpha,p,q}(d_{X},d_{Y},\mu,\nu)=(d^{\Omega}_{GW,qp}(\mu_{X},\nu_{Y}))^{q}
\end{equation}
Proof of this theorem can be found in Section~\ref{allproofs}.
\end{theorem}

This result shows that  $FGW$ can revert to one of the other distances. In machine learning, this allows for a validation of the  $\alpha$ parameter to better fit the data properties (\textit{i.e.} by tuning the relative importance of the feature \textit{vs} structure information). One can also see the choice of $\alpha$ as a representation learning problem and its value can be found by optimizing a given criterion.

\subsection{Geodesic properties}

In this section we present some geodesic properties about the $FGW$ distance. These properties are useful in order to define dynamic formulation of OT problems. This dynamic point of view is inspired by fluid dynamics and found its origin in the Wasserstein context with \cite{Benamou2000}. Various applications in machine learning can be derived from this formulation: interpolation along geodesic paths was used in computer graphics for color
or illumination interpolations \cite{Bonneel:2011:DIU:2024156.2024192}; more recently, \cite{2018arXiv180509545C} used Wasserstein gradient flows in an optimization context, deriving global minima results for non-convex particle gradient descent paving the way for new methods for training neural networks; \cite{pmlr-v80-zhang18a} used Wasserstein gradient flows in the context of reinforcement learning for policy optimization.

The main idea of this dynamic formulation is to describe the optimal transport problem between two measures as a curve in the space of measures minimizing its total length. We first describe some generality about geodesic spaces and recall classical results for dynamic formulation in both Wasserstein and Gromov-Wasserstein contexts. In a second part, we derive new geodesic properties in the $FGW$ context.

\paragraph{Generality about geodesic spaces}

Let $(X,d)$ be a metric space and $x,y$ two points in $X$. We say that a curve $w:[0,1] \rightarrow X$ joining the \textit{endpoints} $x$ and $y$ (\textit{i.e.} with $w(0)=x$ and $w(1)=y$) is a \textit{constant speed geodesic} if it satisfies $d(w(t),w(s)) \leq |t-s| d(w(0),w(1))=|t-s| d(x,y)$ for $t,s \in [0,1]$. Moreover, if $(X,d)$ is a length space (\textit{i.e.} if the distance between two points of $X$ is equal to the infimum of the lengths of the curves connecting these two points) then the converse is also true and a constant speed geodesic satisfies  $d(w(t),w(s)) =|t-s| d(x,y)$. It is easy to compute distances along such curve as they are directly embedded into $\mathbb{R}$.

In the Wasserstein context, if the ground space is a complete separable, locally compact length space and if endpoints of the geodesic are given, then there exists a geodesic curve. Moreover, if the optimal transport between the endpoints is unique then there is a unique displacement interpolation between the endpoints (see Corollary 7.22 and 7.23 in \cite{Villani}). For example, if the ground space is $\mathbb{R}^{d}$ and the distance between the points is measured via the $\ell_{2}$ norm, then geodesics exist and are uniquely determined (this can be generalized to strictly convex cost).

In the Gromov-Wasserstein context, there always exists constant speed geodesics as long as the endpoints are given and these geodesics are unique modulo the isomorphism equivalence relation (see \cite{Sturm2006}).

\paragraph{The $FGW$ case}

In this paragraph we suppose that $\Omega=\mathbb{R}^{d}$.

We are interested in finding a geodesic curve in the space $\left(H(\mathbb{R}^{d}),d^{\mathbb{R}^{d}}_{FGW,\alpha,p,q}\right)$, \textit{i.e.} a constant speed curve of structured objects joining two structured objects. As for Wasserstein and Gromov-Wasserstein, the structured object space endowed with the Fused Gromov-Wasserstein distance maintains some geodesic properties. The following result proves the existence of such a geodesic and characterizes it:

\begin{theorem}{Constant speed geodesic.}
\label{cstespeedtheo}

\noindent Let $p,q \geq 1$ and $( X_{0} \times A_{0} ,d_{X_{0}},\mu_{0})$ and $(X_{1} \times A_{1} ,d_{X_{1}},\mu_{1}) \in H(\mathbb{R}^{d})$. Let $\pi^{*}$ be the optimal coupling for the Fused Gromov-Wasserstein distance between those two sets and $t \in [0,1]$. We equip $\mathbb{R}^{d}$ with any $\ell_{m}$ norm for all $m \geq 1$.

We define $\eta_{t} : X_{0} \times A_{0} \times X_{1} \times A_{1} \rightarrow X_{0} \times X_{1} \times \hat{A_{t}}$ with $\hat{A_{t}}\stackrel{def}{=}\{(1-t)a_{0}+ta_{1}|\forall a_0\in A_0,\forall a_1\in A_1\}$ such that
$$ \eta_{t}(x_{0},a_{0},x_{1},a_{1})=(x_{0},x_{1},(1-t)a_{0}+ta_{1}),\quad \forall x_{0},a_{0},x_{1},a_{1} \in X_{0} \times A_{0} \times X_{1} \times A_{1}$$
Then
\begin{equation}
\label{geodesic}
\left((X_{0}\times X_{1},(1-t)d_{X_{0}} \oplus t d_{X_{1}}), \hat{A_{t}},\mu_{t}=\eta_{t} \# \pi^{*} \in \overline{P}(X_{0}\times X_{1} \times \hat{A_{t}})\right)_{t \in [0,1]}
\end{equation}
is a constant speed geodesic in $\left(H(\mathbb{R}^{d}), d^{\mathbb{R}^{d}}_{FGW,\alpha,p,q}\right)$.

\end{theorem}

From the existence of a geodesic in the structured object space, one can wonder if this geodesic is unique so as to define properly the velocity field associated to the geodesic curve. Informally, if one tries to define the speed of a particle passing a point $p$ (here a structured object) at a time $t$ then the uniqueness of this particle passing through $p$ at $t$ seems mandatory. The following result proves that it is indeed the case modulo the equivalence relation of structured objects $\sim$ in the case where $\Omega=\mathbb{R}^{d}$

\begin{theorem}{Unicity of geodesic in $\left(H(\mathbb{R}^{d}),d^{\mathbb{R}^{d}}_{FGW,\alpha,1,q}\right)$.}
\label{unicitytheorem}

\noindent Let $p=1$ and $q\geq 2$. We equip $\mathbb{R}^{d}$ with the $\ell_q$ norm. Then each geodesic $ (\mathbb{S}_{t})_{t\in [0,1]}=\left((X_{t} \times A_{t},d_{X_{t}}, \mu_{t})\right)_{t \in [0,1]}$ in $H(\mathbb{R}^{d})$ is of the same form as stated in Eq. (\ref{geodesic}).

More precisely, for each geodesic $(\mathbb{S}_{t})_{t\in [0,1]} \in H(\mathbb{R}^{d})$ there exists an optimal coupling $\pi^{*} \in \overline{P}(X_{0} \times A_{0} \times X_{1}\times A_{1})$ of measures $\mu_{0}$ and $\mu_{1}$, representative as the endpoints, for the $d^{\mathbb{R}^{d}}_{FGW,\alpha,1,q}$ distance,  such that for each $t\in [0,1]$ a representative of the equivalence class $\sim$ of $\mathbb{S}_{t}$ is given by: $$\left(X_{0}\times X_{1} \times \hat{A_{t}},(1-t)d_{X_{0}} \oplus t d_{X_{1}},\eta_{t} \# \pi^{*} \right)$$ with $\eta_{t}$ and $\hat{A_{t}}$ defined in theorem \ref{cstespeedtheo}..

\end{theorem}

Proofs of the previous theorems can be found in Section~\ref{allproofs}. In a sense this result combines the geodesics in the Wasserstein space and in the space of all metric spaces since it suffices to interpolate the distances in the structure space and the features to construct the geodesic. The main interest is that it defines the minimum path between two structured objects. For example, considering two discrete structured objects represented by the measures $\mu_{0}=\sum_{i=1}^{n} h_{i}  \delta_{(x_{i},a_{i})}$ and $\mu_{1}=\sum_{j=1}^{m} g_{j}  \delta_{(y_{j},b_{j})}$, the interpolation path is given for $t\in [0,1]$ by the measure $\mu_{t}=\sum_{i=1}^{n}\sum_{j=1}^{m} \pi^{*}(i,j) \delta_{(x_{i},y_{j},(1-t)a_{i} +t b_{j})}$ where $\pi^{*}$ is the optimal coupling for the $FGW$ distance. However this geodesic is difficult to handle in practice since it requires the computation of the cartesian product $X_{0}\times X_{1}$. To overcome this obstacle, an extension using Fréchet mean is defined in section \ref{fgwbarysection}. The proper definition and properties of velocity fields associated to this geodesic is postponed to further works.

\section{Examples and applications for the discrete case}

In this section, we illustrate the behavior of $FGW$ on simple cases where structured objects are involved.

\subsection{FGW in the discrete case}

In the following section, $(\Omega,d)$ is the feature space and $\mathbb{X}_{n}$ is the set of all discrete metric spaces of size $n \geq 1$. Picking a structured object of size $n$ in the discrete case is choosing a metric space $(X,C)$, a set of $n$ elements $(x_{i},a_{i})$ where $(a_{i})_{i}$ denotes the feature information and $(x_{i})_{i}$ denotes the structure information. The matrix $C(i,j)$ aims at comparing the structure points $x_{i}$ and $x_{j}$. From this set we derive a fully supported probability measure $\mu$ by choosing a histogram $h \in \Sigma_{n}$ and $\mu=\sum_{i=1}^{n} h_{i}  \delta_{(x_{i},a_{i})}$.

More precisely, $H(\Omega)=\bigcup\limits_{n\in \mathbb{N}} H_{n}(\Omega)$ where $$H_{n}(\Omega) \stackrel{def}{=} \{(x_{i},a_{i})_{i \in \{1,..,n\}} \in (X \times \Omega)^{n}, \ C \in \mathbb{R}^{n \times n},  \ \mu=\sum_{i=1}^{n} h_{i}  \delta_{(x_{i},a_{i})} | (X,C) \in \mathbb{X}_{n}, h \in \Sigma_{n}  \}$$

This set includes all graphs with any number of vertices (each from a given metric space), where each vertex $x_{i}$ is associated to a feature $a_{i}$ in $\Omega$ and a weight $h_{i}$ on the simplex.

In the next paragraphs, $\mu \in H_{n}(\Omega)$ and $\nu \in H_{m}(\Omega)$ are structured data as described in the previous part. We suppose that $C_{1}$ and $C_{2}$ are the distance matrices inherent to each structure information of $\mu$ and $\nu$ respectively, and $a_{i}$, $b_{j}$ are the features. Let $p,q\geq 1$.

Using previous notations, the Fused Gromov-Wasserstein distance is defined as:
\begin{equation}
\label{discretefgw}
d^{\Omega}_{FGW,\alpha,p,q}(\mu,\nu)= \bigg(\underset{\pi \in \couplingset(h,g)}{\min}E_{p,q}(M_{AB},C_{1},C_{2},\pi) \bigg)^{\frac{1}{p}} \\
\end{equation}
where:
\begin{equation*}
E_{p,q}(M_{AB},C_{1},C_{2},\pi)=\sum\limits_{i,j,k,l} \big((1-\alpha) d(a_{i},b_{j})^{q}+\alpha |C_{1}(i,k)-C_{2}(j,l)|^{q}\big)^{p} \pi_{i,j} \, \pi_{k,l}
\end{equation*}

Algorithms for solving the numerical optimization above are given in \cite{2018arXiv180509114V}. They rely on Conditional Gradient but converge only to a local minimum due to the non-convexity of the optimization problem. We used these algorithms for all the applications below.

\subsection{Illustrations of FGW}

In this section, we present several applications of $FGW$ as a distance betweeen structured objects and provide interpretation of the OT matrix.

\paragraph{Example with 1D features and structure spaces}

Figure \ref{fig:illus_emp} illustrates the differences between Wasserstein, Gromov-Wasserstein and Fused Gromov-Wasserstein couplings $\pi^{*}$. In this example both the feature and structure space are 1-dimensional (Figure \ref{fig:illus_emp} left). The feature space denotes two clusters among the elements of both objects illustrated in the OT matrix $M_{AB}$ , the structure space denotes a noisy temporal sequence along the indexes liustrated in the matrices $C_1$ and $C_2$ (Figure \ref{fig:illus_emp} center). Wasserstein respects the clustering but forgets the temporal structure, Gromov-Wasserstein respects the structure but do not take the clustering into account. Only FGW retrieves a transport matrix respecting both feature and structure.

\begin{figure*}[t]
    \centering
        \includegraphics[height=4.1cm]{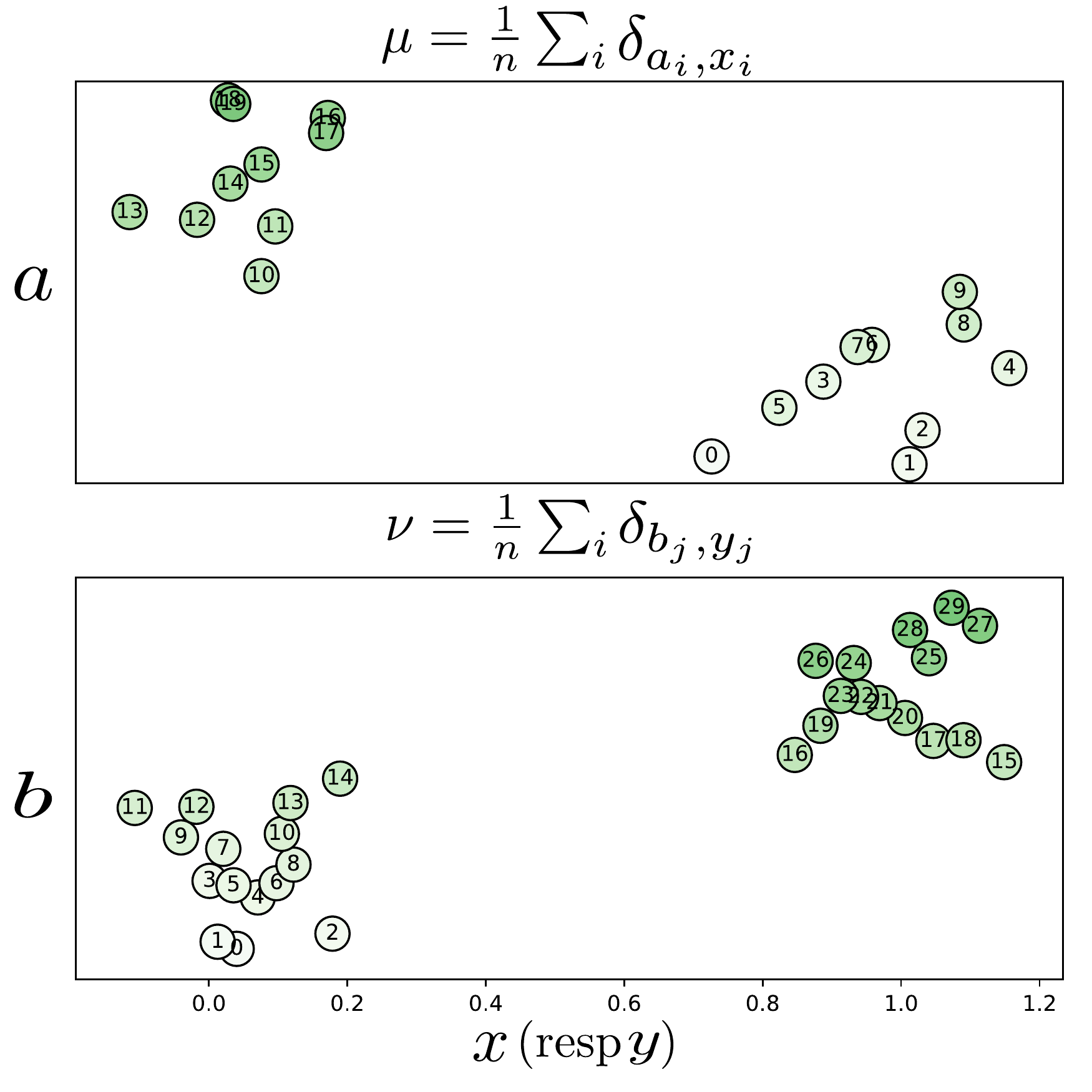}
  \hspace{-2mm}
  \includegraphics[height=4.1cm]{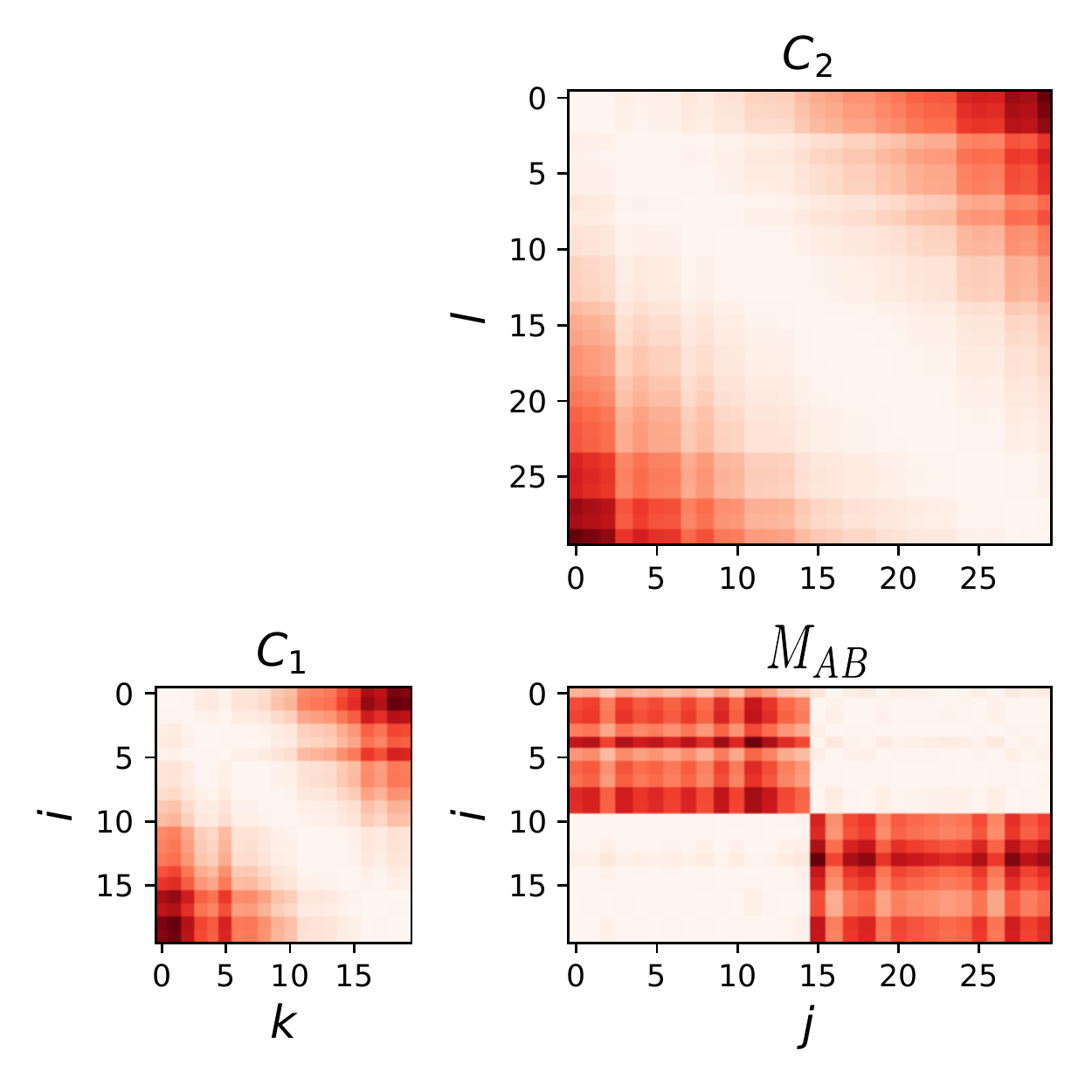}
  \hspace{-2mm}
      \includegraphics[height=4.1cm]{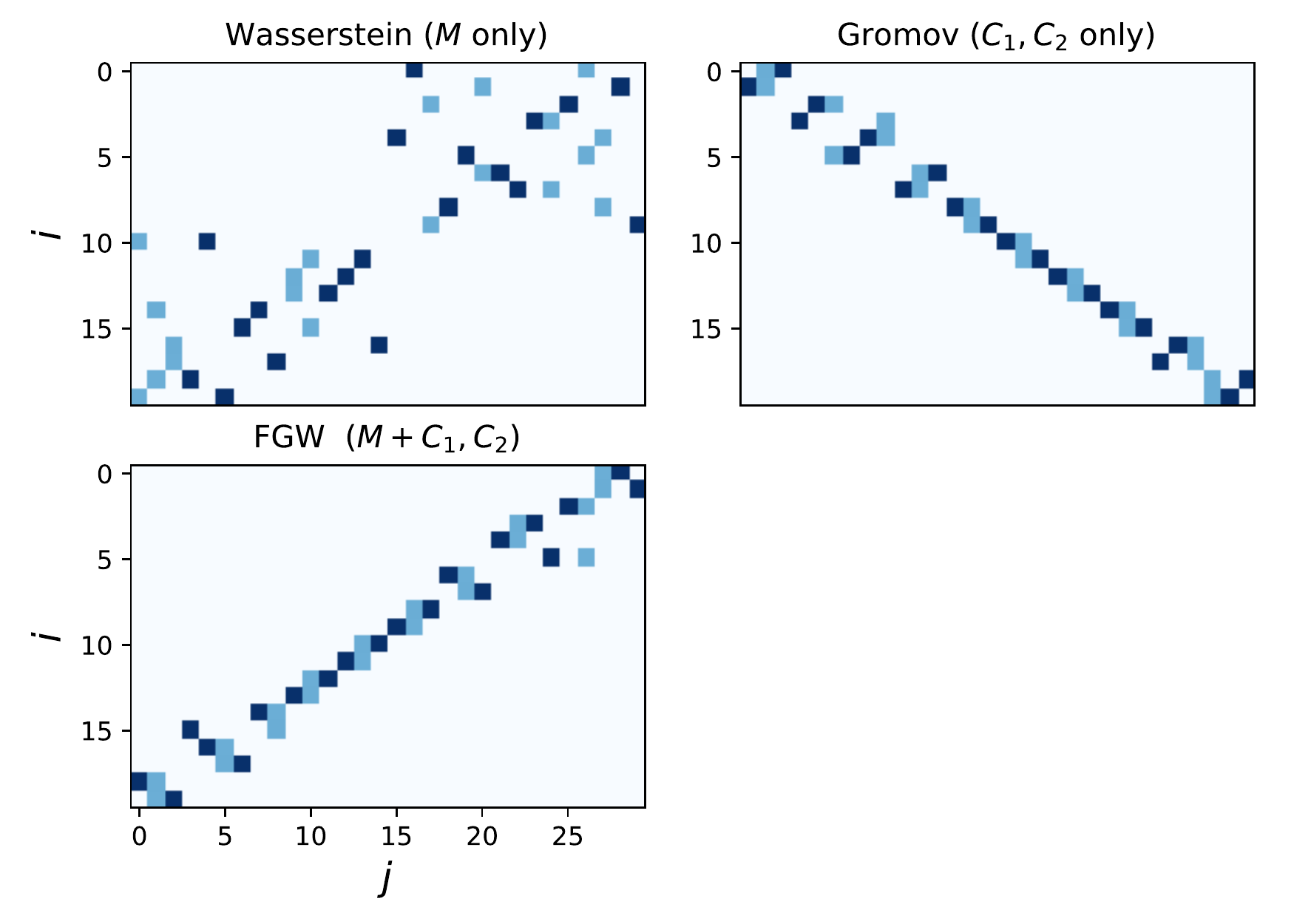}
    \caption{Illustration of the difference between $W$, $GW$ and $FGW$ couplings. (left) empirical distributions $\mu$ with 20 samples and $\nu$ with 30 samples which color is proportional to their index. (middle) Cost matrices in the feature ($M_{{AB}}$) and structure domains ($C_1,C_2$) with similar samples in white.   (right) Solution for all methods. Dark blue indicates a non zero coefficient of the transportation map between $i$ and $j$. Feature distances are large between points laying on the diagonal of $M_{AB}$ such that Wasserstein maps is anti-diagonal but unstructured. Fused Gromov-Wasserstein incorporates both feature and structure maps in a single transport map.}
    \label{fig:illus_emp}
\end{figure*}

\paragraph{Example on two simple images}
We extract one $28\times28$ image from the MNIST dataset and generate a second one by simply re-centering the digit on the frame. Features represent the gray level of each pixel, the structure is defined as the city-block distance on the pixel coordinate grid and we use equal weights for all the pixels in the image. Figure \ref{fig:illus_digit} shows the different couplings obtained when considering either the features only, the structure only or both information. $FGW$ aligns the pixels of the digits, recovering the correct order of the pixels, while both Wassertein and Gromov-Wasserstein distances fail at providing a meaningful transportation map. Note that in the Wasserstein and Gromov-Wasserstein case, the distances are equal to 0, whereas $FGW$ manages to spot that the two images are different. Also note that, in the $FGW$ sense, the original digit and its mirrored version are also equivalent as there exists an isometry between their structure spaces, making $FGW$ invariant to rotations or flips in the structure space in this case.
\begin{figure*}[t]
    \centering
        \includegraphics[height=3cm]{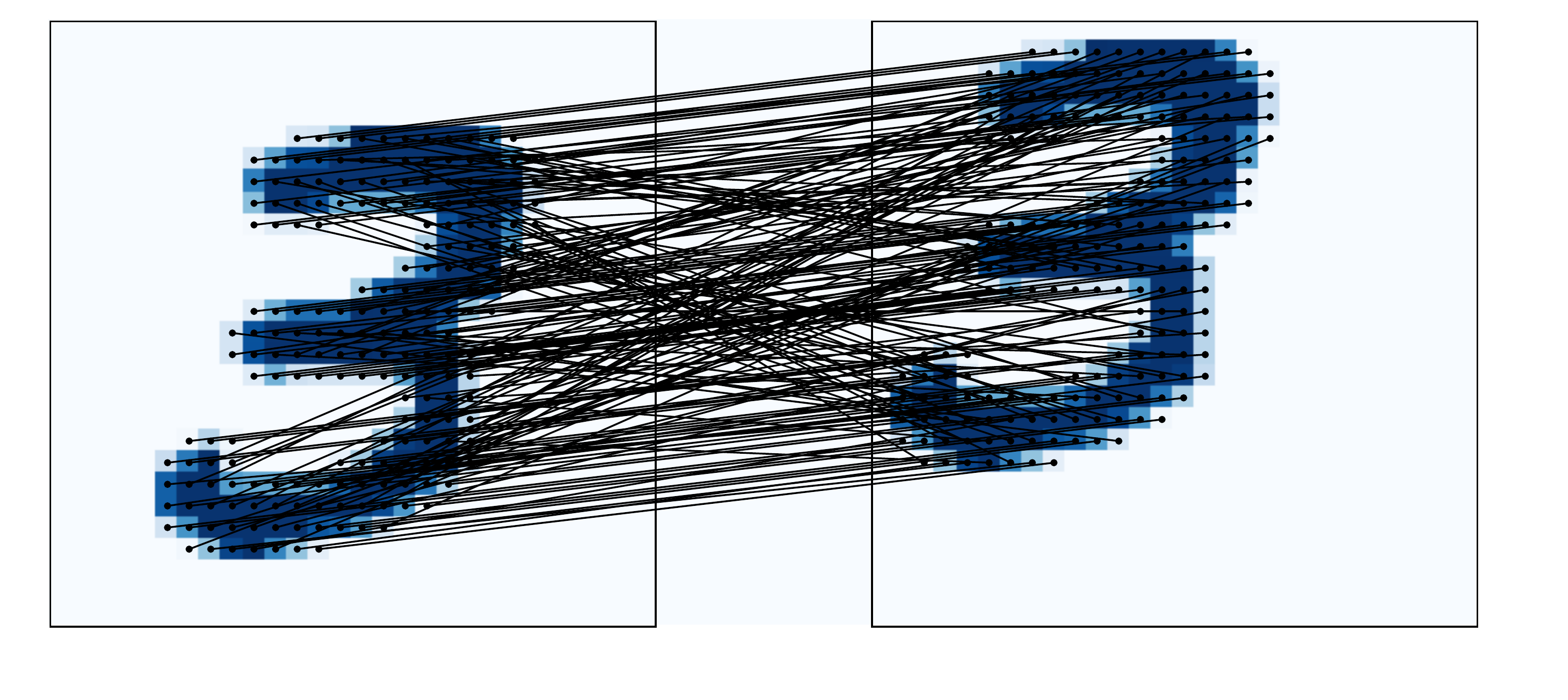}
  \hspace{-2mm}
  \includegraphics[height=3cm]{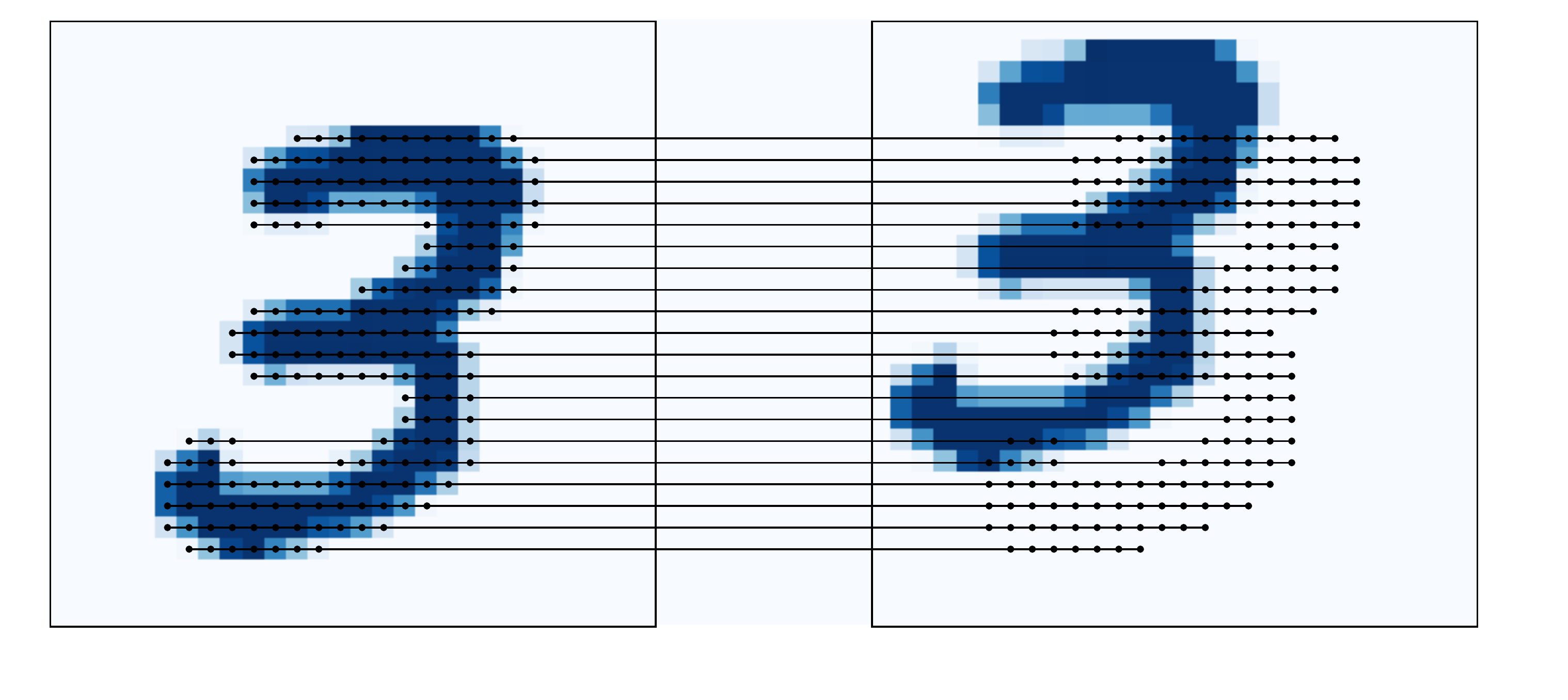}
  \hspace{-2mm}
      \includegraphics[height=3cm]{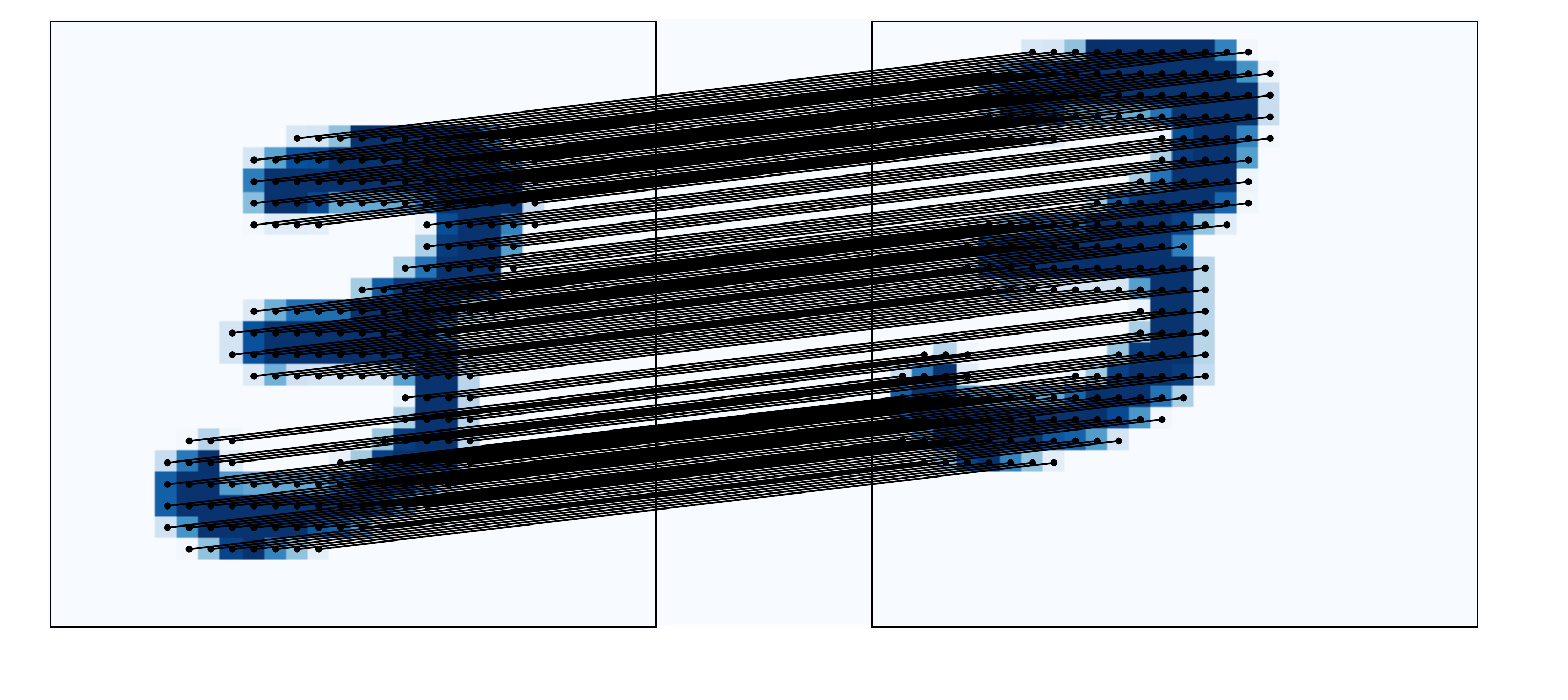}
     \hspace{-2mm}
   \includegraphics[height=3cm]{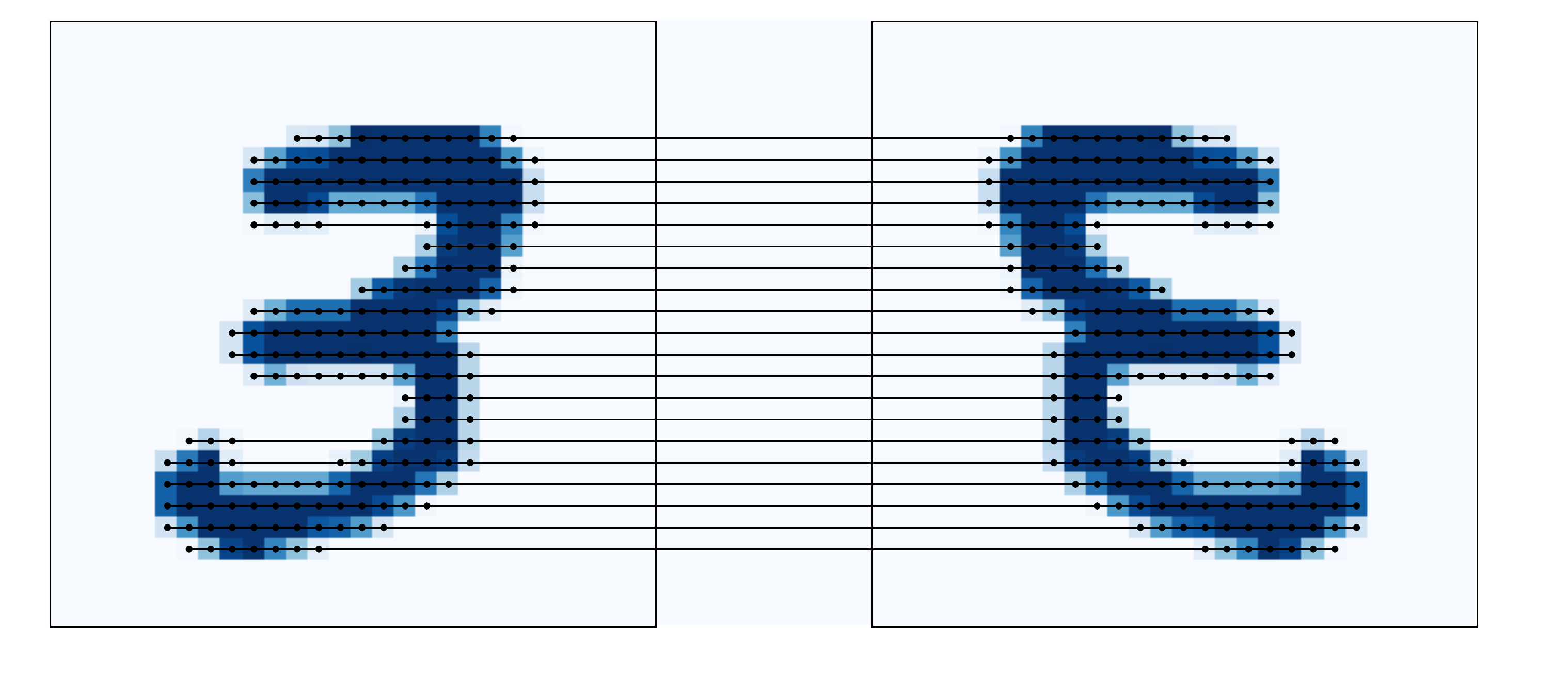}
    \caption{Couplings obtained when considering (Top left) the features only, where we have $d^{\Omega}_{W,1}=0$ (Top right) the structure only, with  $d_{GW,1}=0$ (Bottom left and right) both the features and the structure, with $d^{\Omega}_{FGW,0.1,1,2}$. For readibility issues, only the couplings starting from non white pixels on the left picture are depicted. }
    \label{fig:illus_digit}
\end{figure*}

\paragraph{Time series example}

One of the main assets of $FGW$ is that it can be used on a wide class of objects and time series are one more example of this. We consider here 25 monodimensional time series composed of two humps in $[0,1]$ with random uniform height between 0 and 1.  Signals are distributed according to two classes translated from each other with a fixed gap. The $FGW$ distance is computed by considering $d$ as the euclidean distance between the features of the signals (here the value of the signal in each point) and $d_X$ and $d_Y$ as the euclidean distance between timestamps.

 A 2D embedding is computed from a $FGW$ distance matrix between a number of examples in this dataset with multidimensional scaling (MDS) in Figure \ref{mds} (top). One can clearly see that the representation with a reasonable $\alpha$ value in the center is the most discriminant one. This can be better understood by looking as the OT matrices between the classes.
Figure \ref{mds} (bottom) illustrates the behavior of $FGW$ on one pair of examples when going from  Wasserstein to Gromov-Wasserstein. The black line depict the affectation provided by the transport matrix and one can clearly see that while Wasserstein on the left assigns samples completely independently to their temporal position, the Gromov-Wasserstein on the right tends to align perfectly the samples (note that it could have reversed exactly the alignment with the same loss) but discards the values in the signal. Only the true $FGW$ in the center finds a transport matrix that both respects the time sequences and aligns similar values in the signals.

\begin{figure*}[t]
\centering
\includegraphics[width=1\linewidth]{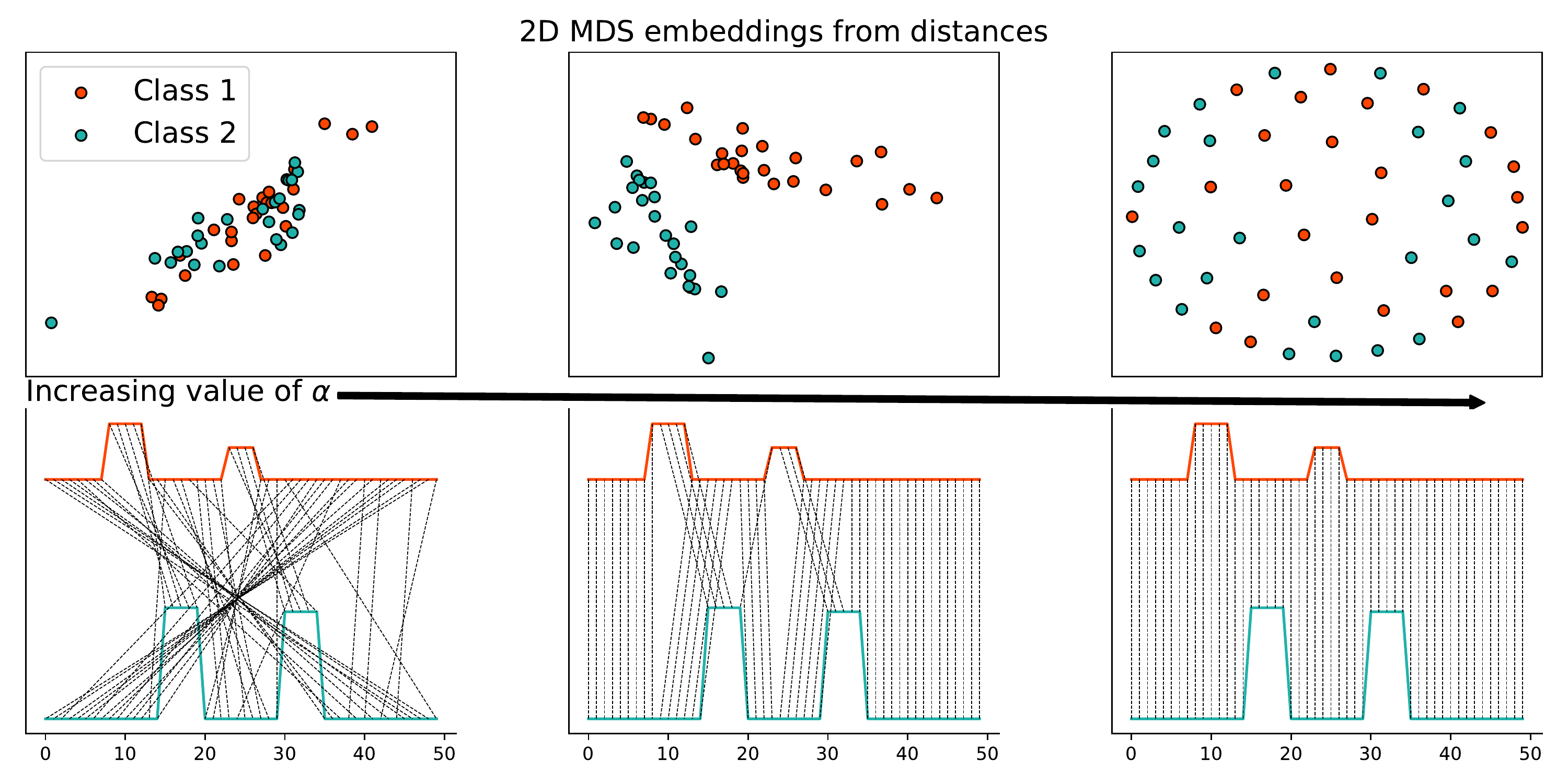}
\caption{Behavior of trade-off parameter $\alpha$ on a toy time series classification problem. $\alpha$ is increasing from left ($\alpha=0$ : Wasserstein distance) to right ($\alpha=1$ : Gromov-Wasserstein distance). (top row) 2D-embedding is computed from the set of pairwise distances between samples with MDS (bottom row) illustration of couplings between two sample time series from opposite classes.}
\label{mds}
\end{figure*}

\subsection{Structured Optimal Transport Barycenter \label{fgwbarysection}}

An interesting use of the $FGW$ distance is to define a barycenter of a set of structured data as a Fr\'echet mean. In that context, one seeks the structured object that minimizes the sum of the (weighted) $FGW$ distances with a given set of objects. OT barycenters have many desirable properties and applications \cite{agueh2011barycenters,peyre2016gromov}, yet no formulation can leverage both structural and feature information in the barycenter computation. Here we propose to use the $FGW$ distance to compute the barycenter of structured objects $(\mu_{k})_{k} \in H(\Omega)^{k}$ associated with structures $(C_{k})_{k}$, features $(B_{k})_{k}$ and base histograms $(h_{k})_{k}$.

We suppose that the feature space is $\Omega=(\mathbb{R}^{d},\ell_{2})$ and $p=1$. For simplicity, we assume that the base histograms and the histogram $h$ associated to the barycenter are known and fixed.

In this context, for a fixed $N \in \mathbb{N}$ and $(\lambda_{k})_{k}$ such that $\sum_{k} \lambda_{k}=1$ , we aim to find:
\begin{equation}
\underset{\mu}{\text{min}} \sum_{k} \lambda_{k} d^{\mathbb{R}^{d}}_{FGW,\alpha,1,q}(\mu,\mu_{k})=\underset{C,\ A \in \mathbb{R}^{N \times d},(\pi_{k})_{k}}{\text{min}} \sum_{k} \lambda_{k} E_{1,q}(M_{AB_{k}},C,C_{k},\pi_{k})
\label{fgwbarycenter}
\end{equation}
Note that this problem is convex \textit{w.r.t} $C$ and $A$ but not \textit{w.r.t} $\pi_{k}$. An algorithm to solve this problem is presented in~\cite{2018arXiv180509114V}. Intuitively, looking for a barycenter means finding feature values supported on a fixed size support, and the structure that relates them. Interestingly enough, there are several variants of this problem, where  features or structure can be fixed for the barycenter. Solving the related simpler optimization problems extend straightforwardly.

\paragraph{Graph barycenter and compression}

In this experiment, we use $FGW$ to compute barycenters and approximations of toy graphs.

\begin{figure}[t]
    \centerline{\includegraphics[width=1.3\textwidth]{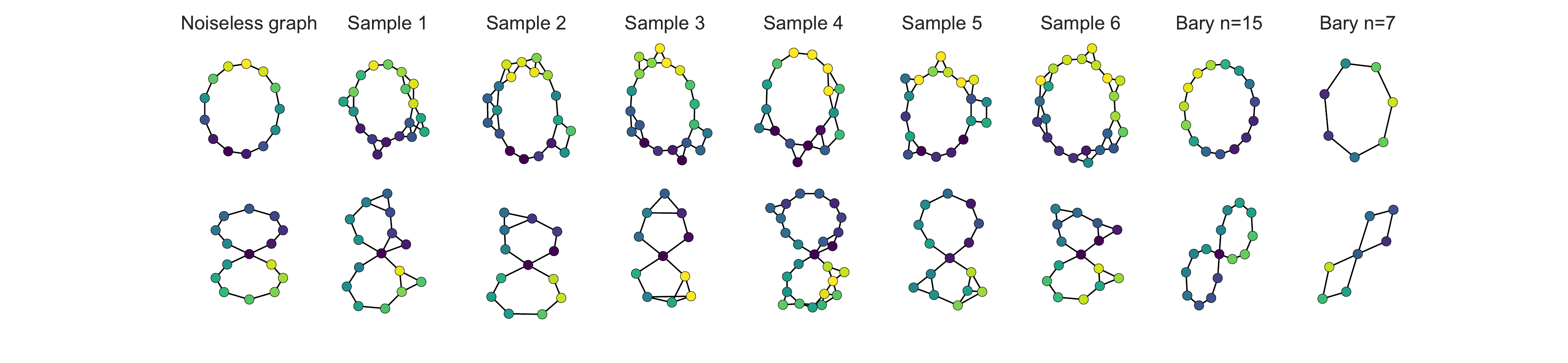}}
    \caption{\label{barygraph}Illustration of $FGW$ graph barycenter. The first column illustrates the original settings with the denoised graphs, and columns 2 to 7 are noisy samples that constitute the datasets. Columns 8 and 9 show the barycenters for each setting, with different number of nodes. Blue nodes indicates a feature value close to $-1$, yellow nodes close to $1$.  }
    \end{figure}

In the first example, we generate graphs following either a circle or $8$ symbol with 1D features following a sine and linear variation respectively.
For each example, the number of nodes is drawn randomly between 10 and 25, Gaussian noise is added to the features and a small noise is applied to the structure (some connections with the third neighbors are randomly added). An example graph  with no noise is provided for each class in the first column of Figure \ref{barygraph}. One can see from there that the circle class has a feature varying smoothly (sine) along the graph but the $8$ has a sharp feature change at its center (so that low pass filtering would loose some information). Some examples of the generated graphs are provided in the 2nd-to-7th columns of Figure \ref{barygraph}.
We compute the $FGW$ barycenter containing 10 samples using the shortest path distance between the nodes as the structural information and the $\ell_{2}$ distance for the features.
We recover an adjency matrix by thresholding the similarity matrix $C$ given by the barycenter. The threshold is tuned so as to minimize the Frobenius norm between the original $C$ matrix and the shortest path matrix constructed after thresholding $C$. Resulting barycenters are showed in Figure \ref{barygraph} for $n=15$ and $n=7$ nodes. First, one can see that the barycenters are denoised both in the feature space and the structure space. Also note that the sharp change at the center of the $8$ class is conserved in the barycenters which is a nice result compared to other divergences that tend to smooth-out their barycenters ($\ell_2$ for instance). Finally, note that by selecting the number of nodes in the barycenter one can compress the graph or estimate a "high resolution'' representation from all the samples.  To the best of our knowledge, no other method can compute such graph barycenters. Finally, note that $FGW$ is interpretable because the resulting OT matrix provides correspondence between the nodes from the samples and those from the barycenter.

\begin{figure}[t]
    \centering
    \includegraphics[width=.9\linewidth]{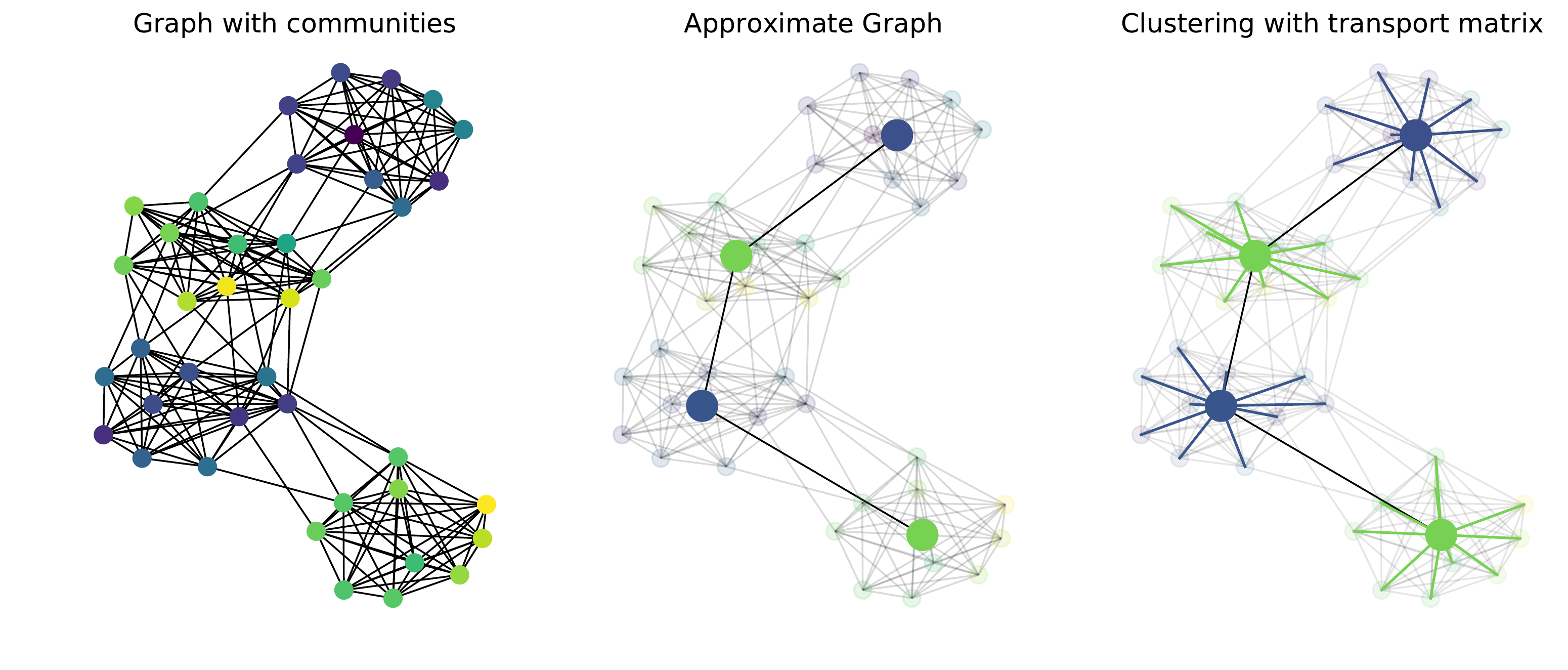}
    \includegraphics[width=.9\linewidth]{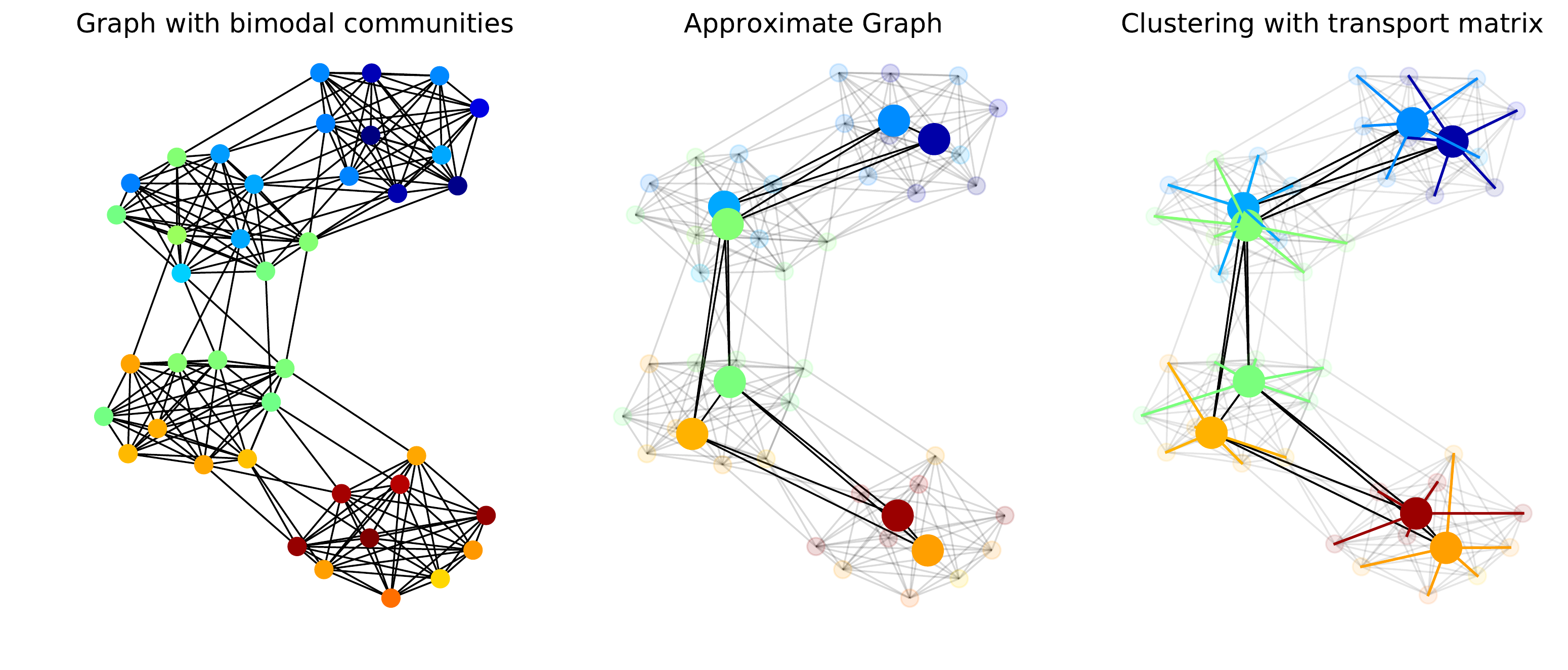}
    \caption{Example of community clustering on graphs using $FGW$. (up) Community clustering with 4 communities and uniform features per cluster. (down)  Community clustering with 4 communities and bimodal features per cluster (and two nodes per cluster in the approximate graph).}
    \label{fig:graphclustr}
\end{figure}

In the second experiment, we evaluate the ability of FGW to perform graph approximation and compression on a simple Stochastic Block Model graph \cite{wang1987stochastic,nowicki2001estimation}. The question is to see if estimating an approximated graph can recover the relation between the blocks and perform simultaneously a community clustering on the original graph (using the OT matrix). We generate two community graphs illustrated in the left column of Figure \ref{fig:graphclustr}. 
We can see that the relation between the blocks is sparse and has a 'linear' structure, the example in the first line has features that follow the blocks (noisy but similar in each block) whereas the example in the second line has two modes per blocks. The first graph approximation (top line) is done with 4 nodes and we can recover both the blocks in the graph and the average feature on each blocks (colors on the nodes). The second problem is more complex due to the two modes per  blocks but we can see that when approximating the graph with 8 nodes we recover both the structure between the blocks but also the sub-clusters in each block which illustrate the strength of FGW: encoding both features and structures.

\paragraph{Mesh barycenter}

We show in this section another example of barycenter. We aim at interpolating between unregistered 3D meshes.

Here, we consider the problem of interpolating between $k=2$  meshes in $3D$ that share a common topology but not the same number of vertices. Such an interpolation is realized by setting  $\lambda_{1}= \lambda$ and $\lambda_{2}=1-\lambda$ and varying $\lambda$ between 0 and 1. We interpolate between two quadrupeds: a deer and a cat, that are triangular meshes with respectively 460 and 989 vertices. This is a particularly difficult problem, since there is no prior matching between meshes available. It has long been considered in the computational geometry and vision communities (e.g. ~\cite{alexa2000rigid,sumner2004deformation}), and generally requires user interventions. In our setting, the structure of the barycenter is set to be the one of the cat: the barycenter should have the same topological structure. Our method then only solves for the vertex positions
$X \in \mathbb{R}^{989 \times 3}$. The topological structures $C_{1}$ and $C_{2}$ are set to be the shortest path along the mesh between two vertices, which is a good approximation of the geodesic distance on the manifold.

\begin{figure}[t]

\centerline{\includegraphics[width=1\textwidth]{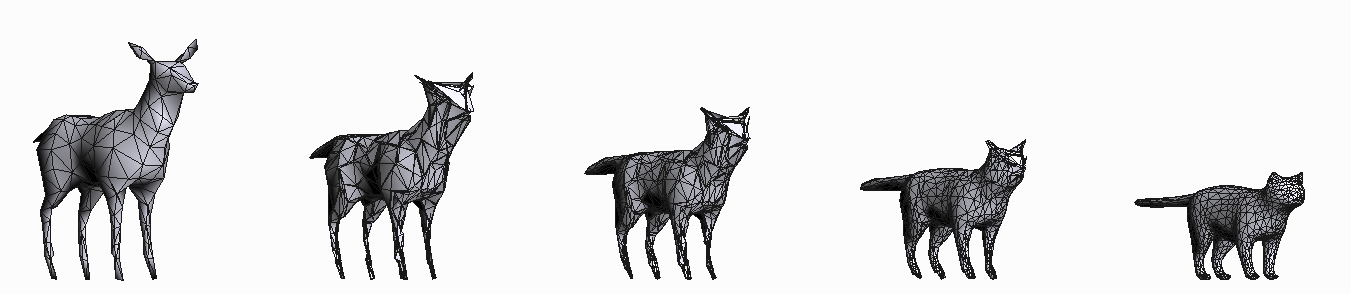}}\centerline{\includegraphics[width=1\textwidth]{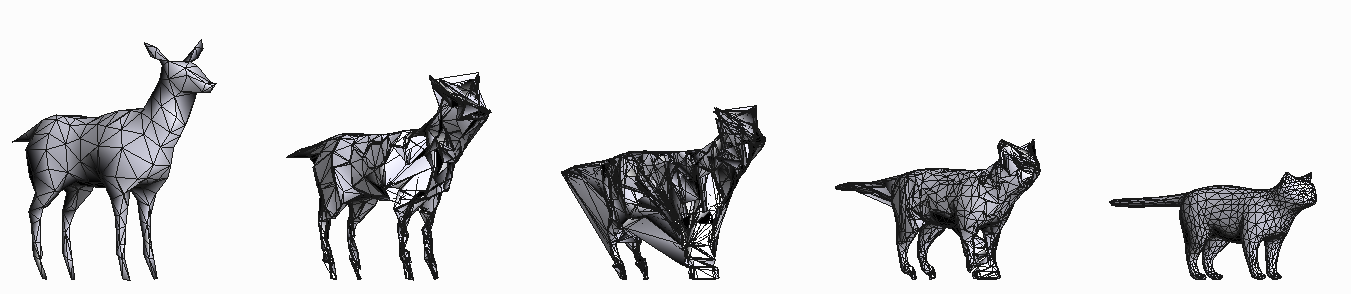}}\caption{Interpolation of a cat and a deer mesh using  $FGW$. (first line) Interpolation using the $FGW$ distance with a high $\alpha$ value (bottomline) same with a very low $\alpha$ value, {\em i.e.} the mesh structure is almost not taken into account.\label{mesh}}
\end{figure}

Results are presented in Figure~\ref{mesh} for $\lambda \in [0.75, 0.5 , 0.25]$. A good way of assessing the quality of the results is to visually check that the consistency of the manifold mesh is preserved throughout the interpolation. The first line shows the resulting interpolation when the weight on the structure is set to a high value. When only 3D distances are used to match the shapes (bottom line), one can see that points belonging to different parts of the meshes are matched, because of the different densities of points in the two meshes. This results in highly unrealistic mesh.

\section{Proofs of the mathematical properties \label{allproofs}}

This section presents all the proofs of previous theorem and results. We will frequently use the following lemma : 

\begin{lemma}
Let $q \geq 1$. We claim :

\begin{equation}
\label{holder}
\forall x,y \in \mathbb{R}_{+}, \ (x+y)^{q} \leq 2^{q-1}(x^{q} + y^{q})
\end{equation}
\end{lemma}
\begin{proof}
Indeed, if $q>1$

$(x+y)^{q} = \big((\frac{1}{2^{q-1}})^{\frac{1}{q}} \frac{x}{(\frac{1}{2^{q-1}})^{\frac{1}{q}}} + (\frac{1}{2^{q-1}})^{\frac{1}{q}} \frac{y}{(\frac{1}{2^{q-1}})^{\frac{1}{q}}}\big)^{q} \leq \big[ (\frac{1}{2^{q-1}})^{\frac{1}{q-1}} +(\frac{1}{2^{q-1}})^{\frac{1}{q-1}}\big]^{q-1} \big(\frac{x^{q}}{\frac{1}{2^{q-1}}} + \frac{y^{q}}{\frac{1}{2^{q-1}}}\big) \\
= \frac{x^{q}}{\frac{1}{2^{q-1}}} + \frac{y^{q}}{\frac{1}{2^{q-1}}}$

Last inequality is a consequence of Hölder inequality. The result remains valid for $q=1$.
\end{proof}

\subsection{Proof of Prop. \ref{fgwcomparetogwandw} Comparison between FGW, GW and W}

\begin{proof}{of the Proposition}

For the two inequalities \eqref{wassinequality} and \eqref{gromovinequality} let $\pi$ be the optimal coupling for the Fused Gromov-Wasserstein distance between $\mu$ and $\nu$ (assuming its existence for now). Clearly : \\ \\
$d^{\Omega}_{FGW,\alpha,p,q}(\mu,\nu) {=}\big(\int\limits_{(X\times A \times Y \times B)^{2}}\!\!\!\!\!\!\!\!\!\!\!\!\!\! ((1-\alpha) d(a,b)^{q} +\alpha L((x,y,x',y')^{q})^{p} \,d\pi((x,a),(y,b))\,d\pi((x',a'),(y',b'))\big)^{\frac{1}{p}} \\ \\
{\geq} \big(\int\limits_{X\times A \times Y \times B} (1-\alpha)^{p} d(a,b)^{pq} \,d\pi((x,a),(y,b))\big)^{\frac{1}{p}}= (1-\alpha) \big( \int\limits_{A \times B}d(a,b)^{pq} \,dP_{2,4}\#\pi(a,b)\big)^{\frac{1}{p}}
$
\\ \\
Since $\pi \in \couplingset(\mu,\nu)$ the coupling $P_{2,4}\#\pi$ is in $\couplingset(\mu_{A},\nu_{B})$. So by suboptimality :
$$d^{\Omega}_{FGW,\alpha,p,q}(\mu,\nu){\geq} (1-\alpha) (d^{\Omega}_{W,pq}(\mu_{A},\nu_{B}))^{q}$$
which proves (\ref{wassinequality}). Same reasoning is used for (\ref{gromovinequality}).

For the last inequality \eqref{samespace} let $\pi \in \couplingset(\mu,\nu)$ be any admissible coupling. By suboptimality :
$d^{\Omega}_{FGW,\alpha,p,1}(\mu,\nu) \\ \\
{\leq}\big(\!\!\!\!\!\!\!\!\!\!\!\!\!\! {\small \int\limits_{(X\times A \times Y \times B)^{2}}\!\!\!\!\!\!\!\!\!\!\!\!\!\!  \big((1{-}\alpha) d(a,b) +\alpha |d_{Z}(x,x')-d_{Z}(y,y')|\big)^{p}d\pi((x,a),(y,b))d\pi((x',a'),(y',b'))}\big)^{\frac{1}{p}} \\ \\
\stackrel{(*)}{\leq}\big(\!\!\!\!\!\!\!\!\!\!\!\!\!\! {\small \int\limits_{(X\times A \times Y \times B)^{2}}\!\!\!\!\!\!\!\!\!\!\!\!\!\!  \big((1{-}\alpha) d(a,b) +\alpha d_{Z}(x,y)+\alpha d_{Z}(x',y')\big)^{p}d\pi((x,a),(y,b))d\pi((x',a'),(y',b'))}\big)^{\frac{1}{p}} \\ \\
{\leq}\big(\!\!\!\!\!\!\!\!\!\!\!\!\!\! {\small \int\limits_{(X\times A \times Y \times B)^{2}}\!\!\!\!\!\!\!\!\!\!\!\!\!\!  \big((1{-}\alpha) d(a,b) +\alpha d_{Z}(x,y)+(1{-}\alpha) d(a',b')+\alpha d_{Z}(x',y')\big)^{p}d\pi((x,a),(y,b))d\pi((x',a'),(y',b'))}\big)^{\frac{1}{p}} \\ \\
\stackrel{(**)}{\leq}2\big(\!\!\!\!\!\!\!\!\!\!\!\!\!\! {\small \int\limits_{X\times A \times Y \times B}\!\!\!\!\!\!\!\!\!\!\!\!\!\!  \big((1{-}\alpha) d(a,b) +\alpha d_{Z}(x,y)\big)^{p}d\pi((x,a),(y,b))}\big)^{\frac{1}{p}} 
$

(*) is the triangle inequality of $d_{Z}$ and (**) Minkowski inequality. Since this inequality is true for any admissible coupling $\pi$ we can apply it with the optimal coupling for the Wasserstein distance defined in the proposition and the claim follows.

\end{proof}

\subsection{Proof of the theorem \ref{metrictheo} Metric properties of FGW}

We propose to prove the theorem point by point : first the existence, then the equality relation and finally the triangle inequality statement. We first recall the following lemma (lemma 10.3 in \cite{springerlink:10.1007/s10208-011-9093-5}):

\begin{lemma}{}
\label{minimizer}
Let $(W,d_{W})$ be a compact metric space and $\mathcal{M}$ be a subset of $P(W)$ which is sequentially compact for the weak convergence.

If we find $\phi: W \times W \rightarrow \mathbb{R}$  Lipschitz for following the $L^{1}$ metric on $W \times W$ :

$$\hat{d}((w_{1},w_{2}),(w_{1}',w_{2}'))=d_{W}(w_{1},w_{1}')+d_{W}(w_{2},w_{2}')$$
with $w_{1},w_{1}',w_{2},w_{2}' \in W^{4}$.

Then the application $\mu \rightarrow I(\mu)=\underset{W\times W}{\int \int} \phi(w,w')d\mu(w)d\mu(w')$ admits a minimizer in $\mathcal{M}$.

Moreover if $(\mu_{n})_{n \in \mathbb{N}}$ converges weakly to $\mu$ then $I(\mu_{n}) \rightarrow I(\mu)$ as $n \rightarrow \infty$

\end{lemma}

\begin{proposition}{Existence of an optimal coupling for the $FGW$ distance}

For $p,q \geq 1$, $\pi \rightarrow \lossfgw(\pi)$ always always achieves a infimum $\pi^{*}$ in $\Pi(\mu,\nu)$ such that $d^{\Omega}_{FGW,\alpha,p,q}(\mu,\nu)=\lossfgw(\pi^{*})$

\end{proposition}

\begin{proof}

The sequential compactness of $\couplingset(\mu,\nu)$ is classic results (see lemma 4.4 in \cite{Villani}). To prove the existence of $FGW$ distance we use previous lemma \ref{minimizer} which states the existence of a minimizer for the integral of $\phi$. The main idea is to rewrite the definition of $FGW$ in the form of the lemma.

We first consider the case $p=1$ and use previous lemma \ref{minimizer} with $W=X\times A \times Y \times B$ and for $w,w' \in W \times W$ :
$$\phi(w=(x,a,y,b),w'=(x',a',y',b'))=(1-\alpha)d(a,b)^{q}+\alpha L(x,y,x',y')^{q}$$ such that 

$$d^{\Omega}_{FGW,\alpha,1,q}(\mu,\nu) =\underset{\pi \in \mathcal{M}(\mu,\nu)}{\text{inf}}\underset{W\times W}{\int \int} \phi(w,w')d\pi(w)d\pi(w')$$
We equip W with the metric: 
$$d_{W}(w=(x,a,y,b),w'=(x',a',y',b'))=d_{X}(x,x')+d_{Y}(y,y')+d(a,a')+d(b,b').$$
So by lemma \ref{minimizer} it suffices to show that $\phi$ is Lipschitz on $W \times W$ with respect to  :
$$\hat{d}(z=(w_{1},w_{2}),z'=(w_{1}',w_{2}'))=d_{W}(w_{1},w_{1}')+d_{W}(w_{2},w_{2}')$$ 
with $w_{1},w_{1}',w_{2},w_{2}' \in W^{4}$

We also consider $g(t)=t^{q}$ and, $$\phi_{1}(w=(x,a,y,b),w'=(x',a',y',b'))=L(x,y,x',y')^{q}$$ and $$\phi_{2}(w=(x,a,y,b),w'=(x',a',y',b'))=d(a,b)^{q}$$ such that 
$$\phi=(1-\alpha) g\circ\ \phi_{1}+\alpha g\circ\ \phi_{2}.$$

This notations will be useful to prove the case $q>1$ from the case $q=1$. Indeed, we will show that $\phi_{1}$ and $\phi_{2}$ are 1-Lipschitz wrt $\hat{d}$, this will prove the result for $q=1$. Using the boundedness of $g$,$\phi_{1}$ and $\phi_{2}$ over compacts we will conclude for $q > 1$. 

To prove that $\phi_{1}$ and $\phi_{2}$ are 1-Lipschitz wrt $\hat{d}$ we have to show that for $i=1,2$ and $(w_{i}=(x_{i},a_{i},y_{i},b_{i}),w_{i}'=(x_{i}',a_{i}',y_{i}',b_{i}')) \in W\times W$ :
$$|\phi_{i}(z=(w_{1},w_{2}))-\phi_{i}(z'=(w_{1}',w_{2}'))| \leq \hat{d}(z=(w_{1},w_{2}),z'=(w_{1}',w_{2}'))$$ with by definition : \\ \\
$\hat{d}(z=(w_{1},w_{2}),z'=(w_{1}',w_{2}'))= d_{X}(x_{1},x_{1}')+d_{Y}(y_{1},y_{1}')+d(a_{1},a_{1}')+d(b_{1},b_{1}')+d_{X}(x_{2},x_{2}')+d_{Y}(y_{2},y_{2}')\\ \\+d(a_{2},a_{2}')+d(b_{2},b_{2}') $

\paragraph{}
We first consider $\phi_{1}$ :

$|\phi_{1}(z=(w_{1},w_{2}))-\phi_{1}(z'=(w_{1}',w_{2}'))|
= \\ \\|\phi_{1}(w_{1}{=}(x_{1},a_{1},y_{1},b_{1}),w_{2}{=}(x_{2},a_{2},y_{2},b_{2}))-\phi_{1}(w_{1}'{=}(x_{1}',a_{1}',y_{1}',b_{1}'),w_{2}'{=}(x_{2}',a_{2}',y_{2}',b_{2}'))| \\ \\
=|\ |d_{X}(x_{1},x_{2})-d_{Y}(y_{1},y_{2})|-|d_{X}(x_{1}',x_{2}')-d_{Y}(y_{1}',y_{2}')| \ | \\ \\
$
Yet :

$|\ |d_{X}(x_{1},x_{2})-d_{Y}(y_{1},y_{2})|-|d_{X}(x_{1}',x_{2}')-d_{Y}(y_{1}',y_{2}')| \ | 
\leq |d_{X}(x_{1},x_{2})-d_{Y}(y_{1},y_{2})+d_{X}(x_{1}',x_{2}')-d_{Y}(y_{1}',y_{2}')| \\ \\
\leq |d_{X}(x_{1},x_{2})-d_{X}(x_{1}',x_{2}')|+|d_{Y}(y_{1},y_{2})-d_{Y}(y_{1}',y_{2}')|
\leq d_{X}(x_{1},x_{1}')+d_{Y}(y_{1},y_{1}')+d_{X}(x_{2},x_{2}')+d_{Y}(y_{2},y_{2}') \\ \\
$
Last inequality is consequence of triangle inequalities of $d_{X}$ and $d_{Y}$.

Consequently $|\ |d_{X}(x_{1},x_{2})-d_{Y}(y_{1},y_{2})|-|d_{X}(x_{1}',x_{2}')-d_{Y}(y_{1}',y_{2}')| \ | \leq d_{X}(x_{1},x_{1}')+d_{Y}(y_{1},y_{1}')+d(a_{1},a_{1}')+d(b_{1},b_{1}')+d_{X}(x_{2},x_{2}')+d_{Y}(y_{2},y_{2}')+d(a_{2},a_{2}')+d(b_{2},b_{2}')$ and so $\phi_{1}$ is 1-Lipschitz \textit{w.r.t} $\hat{d}$.
\\ \\
Regarding $\phi_{2}$ :

$|\phi_{1}(z=(w_{1},w_{2}))-\phi_{1}(z'=(w_{1}',w_{2}'))| \\ \\ 
=|\phi_{2}(w_{1}{=}(x_{1},a_{1},y_{1},b_{1}),w_{2}{=}(x_{2},a_{2},y_{2},b_{2}))-\phi_{2}(w_{1}'{=}(x_{1}',a_{1}',y_{1}',b_{1}'),w_{2}'{=}(x_{2}',a_{2}',y_{2}',b_{2}'))| \\ \\
=|d(a_{1},b_{1})-d(a_{1}',b_{1}')  | \leq d(a_{1},a_{1}')+d(b_{1},b_{1}') \\ \\
\leq  \hat{d}((z=(w_{1},w_{2}),z'=(w_{1}',w_{2}'))$
Last inequality is consequence of triangle inequalities of $d$. So $\phi_{2}$ is 1-Lipschitz \textit{w.r.t} $\hat{d}$

Since all metric spaces are compact $\phi_{1}$ and $\phi_{2}$ are bounded by a constant $M_{1}$ and $M_{2}$.  Then the restriction $g_{1}$ of $g$ on $[0,M_{1}]$  and the restriction $g_{2}$ of $g$ on $[0,M_{2}]$ are Lipschitz with constants bounded by $qM_{1}^{q-1}$ and $qM_{2}^{q-1}$.

Since $\phi=(1-\alpha) g_{1}\circ\ \phi_{1}+\alpha g_{2}\circ\ \phi_{2}$ then $\phi$ is lipschitz with constant $(1-\alpha) qM_{1}^{q-1}+\alpha qM_{2}^{q-1}$, so by lemma \ref{minimizer} there exists a minimizer for $p=1$. For $p>1$ we can have the same reasoning to show that $\phi^{p}$ is lipschitz with constant $p( (1-\alpha) qM_{1}^{q-1}+\alpha qM_{2}^{q-1})^{p-1}$ so there exists a minimizer for all $p$.

\end{proof}

\begin{proposition}{Equality relation}

$d^{\Omega}_{FGW,\alpha,p,q}(\mu,\nu) =0$ \textit{iff} there exists $f=(f_{1},f_{2}): X\times A \rightarrow Y \times B$ verifying (\ref{metrictheo:conservation}), (\ref{metrictheo:conservationfeature}) and (\ref{metrictheo:isometry}).

\end{proposition}

\begin{proof}

First, let suppose that such an application exists. We consider the map $\pi=( I_{d} \times f ) \# \mu \in \couplingset(\mu,\nu)$. Then : \\ \\
{\small
$ \lossfgw(\pi){=}\int\limits_{(X\times A \times Y \times B)^{2}}\!\!\!\!\!\!\!\!\!\!\!\!\!\! \big((1-\alpha) d(a,b)^{q} +\alpha L((x,y,x',y')^{q} \big)^{p} \,d\pi((x,a),(y,b))\,d\pi((x',a'),(y',b')\!) \\ \\
=\int\limits_{(X\times A )^{2}} \big((1-\alpha) d(a,f_{2}(a))^{q} +\alpha L((x,f_{1}(x),x',f_{1}(x'))^{q} \big)^{p} \,d\mu(x,a)d\mu(x',a') \\ \\
=\int\limits_{(X\times A )^{2}} \big((1-\alpha) d(a,f_{2}(a))^{q} +\alpha |d_{X}(x,x')-d_{Y}(f_{1}(x)f_{1}(x'))|^{q} \big)^{p} \,d\mu(x,a)d\mu(x',a')=0
$
} 

Since $f_{2}(a)=a$ and $f_{1}$ is an isometry. So $\pi$ is the optimal map and $d^{\Omega}_{FGW,\alpha,p,q}(\mu,\nu) =0$.

Conversly suppose that $d^{\Omega}_{FGW,\alpha,p,q}(\mu,\nu) =0$. To prove the existence of a map $f=(f_{1},f_{2}): X\times A \rightarrow Y \times B$ we will use the Gromov-Wasserstein properties. We are looking for a vanishing Gromov-Wassersein distance between the spaces $X\times A$ and $Y\times B$ equipped with our two measures $\mu$ and $\nu$ and two distances applications.

More precisely, we define for $((x,a),(y,b),(x',a'),(y',b')) \in (X\times A \times Y \times B)^{2}$ :

$$ d_{X\times A} ((x,a),(x',a'))=\frac{1}{2} d_{X}(x,x')+\frac{1}{2} d(a,a')$$
and
$$ d_{Y\times B} ((y,b),(y',b'))=\frac{1}{2} d_{Y}(y,y')+\frac{1}{2} d(b,b')$$

We will prove that $d_{GW,p}(d_{X\times A},d_{Y\times B},\mu,\nu)=0$.

Let $\pi \in \couplingset(\mu,\nu)$ be any admissible transportation plan. Then for $n\geq 1$ : \\ \\
$J_{n}(d_{X\times A},d_{Y\times B},\pi)=\!\!\!\!\!\!\!\!\!\!\!\!\!\!\int\limits_{(X\times A \times Y \times B)^{2}}\!\!\!\!\!\!\!\!\!\!\!\!\!\! L(x,y,x',y')^{n} d\pi((x,a),(y,b))\,d\pi((x',a'),(y',b')\!) \\ \\
{=}\int\limits_{(X\times A \times Y \times B)^{2}}\!\!\!\!\!\!\!\!\!\!\!\!\!\! |\frac{1}{2} (d_{X}(x,x')-d_{Y}(y,y'))+\frac{1}{2} (d(a,a')-d(b,b')) |^{n} d\pi((x,a),(y,b))\,d\pi((x',a'),(y',b')\!) \\ \\
{\leq} \int\limits_{(X\times A \times Y \times B)^{2}}\!\!\!\!\!\!\!\!\!\!\!\!\!\! \frac{1}{2} | d_{X}(x,x')-d_{Y}(y,y')|^{n} d\pi((x,a),(y,b))\,d\pi((x',a'),(y',b')\!)\\ 
+\int\limits_{(X\times A \times Y \times B)^{2}}\!\!\!\!\!\!\!\!\!\!\!\!\!\! \frac{1}{2} |d(a,a')-d(b,b') |^{n} d\pi((x,a),(y,b))\,d\pi((x',a'),(y',b')\!)
$
using Jensen inequality with convexity of $t\rightarrow t^{n}$ and subadditivity of |.| . We note $(*)$ the first term and $(**)$ the second term. By the triangle inequality properties of $d$ we have : \\ \\
(**)$\leq \int\limits_{(X\times A \times Y \times B)^{2}}\!\!\!\!\!\!\!\!\!\!\!\!\!\! \frac{1}{2} (d(a,b)+d(a',b') )^{n} d\pi((x,a),(y,b))\,d\pi((x',a'),(y',b')\!) \stackrel{def}{=} M_{n}(\pi)$ such that we have shown :

\begin{equation}
\label{relationgwgw}
\forall \pi \in \couplingset(\mu,\nu), \forall n \geq 1, J_{n}(d_{X\times A},d_{Y\times B},\pi) \leq \frac{1}{2} J_{n}(d_{X},d_{Y},\pi) + M_{n}(\pi)
\end{equation}

Now let $\pi_{*}$ be the optimal coupling for $d^{\Omega}_{FGW,\alpha,p,q}$ between $\mu$ and $\nu$. By hypothesis $d^{\Omega}_{FGW,\alpha,p,q}(\mu,\nu)=0$ then since $\lossfgw(\pi) \geq \alpha^{p} J_{qp}(d_{X},d_{Y},\pi) $ and $\lossfgw(d_{X},d_{Y},\pi) \geq H_{qp}(\pi) $ :

 $$J_{qp}(d_{X},d_{Y},\pi{*})=0$$

and

$$H_{qp}(\pi^{*})=0$$

 Then $\int\limits_{(X\times A\times Y \times B)} d(a,b)^{qp} d\pi^{*}((x,a),(y,b))=0$. Since all terms are positive we can conclude that $\forall m \in \mathbb{N}^{*}, \int\limits_{(X\times A\times Y \times B)} d(a,b)^{m} d\pi^{*}((x,a),(y,b))=0$. In this way : \\ \\

$ M_{qp}(\pi^{*})=\frac{1}{2} \int\limits_{(X\times A\times Y \times B)^{2}} \sum_{h} {qp\choose h} d(a,b)^{h} d(a',b')^{qp-h} d\pi^{*}((x,a),(y,b))d\pi((x',a'),(y',b')) \\ \\
= \sum_{h} {qp\choose h} \big(\int\limits_{(X\times A\times Y \times B)} d(a,b)^{h} d\pi^{*}((x,a),(y,b))\big) \big(\int\limits_{(X\times A\times Y \times B)} d(a',b')^{qp-h} d\pi((x',a'),(y',b'))\big) \\ \\ 
=0
$

Using eq (\ref{relationgwgw}) we have shown : $$J_{qp}(d_{X\times A},d_{Y\times B},\pi_{*})=0$$

So $\pi_{*}$ is the optimal coupling for $d_{GW,p}(\mu,\nu)$ (for the distances $d_{X\times A}$ and $d_{Y\times B}$) and $d_{GW,p}(\mu,\nu)=0$. Thanks to Gromov-Wasserstein properties (see \cite{Memoli:2004:CPC:1057432.1057436})  there exists $f=(f_{1},f_{2}): X\times A \rightarrow Y \times B$ which verifies \eqref{metrictheo:conservation} and

$$\forall ((x,a),(x',a')) \in (X\times A)^{2}, \ d_{X\times A} ((x,a),(x',a'))=d_{Y\times B} (f(x,a),f(x',a'))$$

or equivalently :

\begin{equation}
\label{hat}
\forall ((x,a),(x',a')) \in (X\times A)^{2}, \ \frac{1}{2} d_{X}(x,x')+\frac{1}{2} d(a,a')=\frac{1}{2} d_{Y}(f_{1}(x),f_{1}(x'))+\frac{1}{2} d(f_{2}(a),f_{2}(a'))
\end{equation}

Moreover since $\pi^{*}$ is the optimal coupling for $d_{GW,p}(d_{X\times A},d_{Y\times B},\mu,\nu)$ leading to a cost nul, then  $\pi^{*}$ is supported by $f$, in particular $\pi^{*}=(I_{d} \times f)$. So  $H_{qp}(\pi^{*})=\int\limits_{(X\times A\times Y \times B)} d(a,b)^{qp} d\pi^{*}((x,a),(y,b))=\int\limits_{A}d(a,f_{2}(a))^{qp} d\mu_{A}(a)=0$. Since $\mu_{A}$ fully supported on $A$ we can conclude that $f_{2}=I_{d}$.

In this way, using the equality (\ref{hat}) we can conclude that : $$\forall (x,x') \in X\times X, \ d_{X}(x,x')=d_{Y}(f_{1}(x),f_{1}(x'))$$

\end{proof}

\begin{proposition}{Symmetry and triangle Inequality}
\label{triangleprop}
$d^{\Omega}_{FGW,\alpha,p,q}$ is symmetric and for $q=1$ satisfies the triangle inequality. For $q\geq 2$ the triangular inequality is relaxed by a factor $2^{q-1}$ 
\end{proposition}

To prove this result we will use the following lemma :

\begin{lemma}

Let $(X\times A,d_{X},\mu),(Y \times B,d_{Y},\beta),(Z,\times C,d_{Z}),\nu) \in H(\Omega)^{3}$.
For $(x,a),(x',a') \in (X \times A)^{2}$, $(y,b),(y',b') \in (Y \times B)^{2}$ and $(z,c),(z',c') \in (Z \times C)^{2}$ we have :

\begin{equation}
\label{ineqL}
L(x,z,x',z')^{q} \leq 2^{q-1}(L(x,y,x',y')^{q}+L(y,z,y',z')^{q})
\end{equation}

\begin{equation}
\label{ineqD}
d(a,c)^{q} \leq 2^{q-1}(d(a,b)^{q}+d(b,c)^{q})
\end{equation}

\end{lemma}

\begin{proof}
Direct consequence of (\ref{holder}) and triangle inequalities of $d,d_{X},d_{Y},d_{Z}$
\end{proof}

We now prove the proposition \ref{triangleprop}

\begin{proof}
To prove the triangle inequality of $d^{\Omega}_{FGW,\alpha,p,q}$ distance for arbitrary measures we will use the Gluing lemma which stresses the existence of couplings with a prescribed structure. Let $(X\times A,d_{X},\mu),(Y \times B,d_{Y},\beta),(Z,\times C,d_{Z}),\nu) \in H(\Omega)^{3}$

Let $\pi_{1} \in \couplingset(\mu,\beta)$ and $\pi_{2} \in \couplingset(\beta,\nu)$ be the optimal transportation plans for the Fused Gromov-Wasserstein distance between $\mu$, $\beta$ and $\beta$, $\nu$ respectively. By the Gluing Lemma (see \cite{Villani} and lemma 5.3.2 in \cite{ambrosio2005gradient}) there exists a probability measure $\pi \in P\big((X\times A) \times (Y \times B) \times (Z \times C)\big)$ with marginals $\pi_{1}$ on $(X\times A) \times (Y \times B)$ and $\pi_{2}$ on $(Y \times B) \times (Z \times C)$. Let $\pi_{3}$ be the marginal of $\pi$ on $(X\times A)\times (Z \times C)$. By construction $\pi_{3} \in \Pi(\mu,\nu)$. So by suboptimality of $\pi_{3}$:
\\ \\
$d^{\Omega}_{FGW,\alpha,p,q}(d_{X},d_{Z},\mu,\nu) 
{\leq}\big(\!\!\!\!\!\!\!\!\!\!\!\!\!\! \int\limits_{(X\times A \times Z \times C)^{2}}\!\!\!\!\!\!\!\!\!\!\!\!\!\! \big((1-\alpha) d(a,c)^{q} +\alpha L(x,z,x',z')^{q} \big)^{p} d\pi_{3}((x,a),(z,c))d\pi_{3}((x',a'),(z',c'))\big)^{\frac{1}{p}} \\ \\
{=}\big(\!\!\!\!\!\!\!\!\!\!\!\!\!\! \int\limits_{(X\times A \times Y \times B \times Z \times C)^{2}}\!\!\!\!\!\!\!\!\!\!\!\!\!\! \big((1-\alpha) d(a,c)^{q} +\alpha L(x,z,x',z')^{q} \big)^{p} d\pi((x,a),(y,b),(z,c))d\pi((x',a'),(y',b'),(z',c'))\big)^{\frac{1}{p}} \\ \\
\stackrel{(*)}{\leq}2^{q-1} \big(\!\!\!\!\!\!\!\!\!\!\!\!\!\! \int\limits_{(X\times A \times Y \times B \times Z \times C)^{2}}\!\!\!\!\!\!\!\!\!\!\!\!\!\! \big((1-\alpha) d(a,b)^{q}+(1-\alpha)d(b,c)^{q} +\alpha L(x,y,x',y')^{q}+\alpha L(y,z,y',z')^{q} \big)^{p} \\ \\
d\pi((x,a),(y,b),(z,c))d\pi((x',a'),(y',b'),(z',c'))\big)^{\frac{1}{p}} \\ \\
\stackrel{(**)}{\leq}2^{q-1} \bigg(\big(\!\!\!\! \int\limits_{(X\times A \times Y \times B \times Z \times C)^{2}}\!\!\!\!\!\!\!\!\!\!\!\!\!\! \big((1-\alpha) d(a,b)^{q} +\alpha L(x,y,x',y')^{q}\big)^{p} d\pi((x,a),(y,b),(z,c))d\pi((x',a'),(y',b'),(z',c'))\big)^{\frac{1}{p}} \\ \\
{+} \big(\!\!\!\!\!\!\!\!\!\!\!\!\!\! \int\limits_{(X\times A \times Y \times B \times Z \times C)^{2}}\!\!\!\!\!\!\!\!\!\!\!\!\!\! \big((1-\alpha) d(b,c)^{q} +\alpha L(y,z,y',z')^{q}\big)^{p} d\pi((x,a),(y,b),(z,c))d\pi((x',a'),(y',b'),(z',c'))\big)^{\frac{1}{p}} \bigg) \\ \\
{=}2^{q-1} \bigg(\big(\!\!\!\! \int\limits_{(X\times A \times Y \times B )^{2}}\!\!\!\!\!\!\!\!\!\!\!\!\!\! \big((1-\alpha) d(a,b)^{q} +\alpha L(x,y,x',y')^{q}\big)^{p} d\pi_{1}((x,a),(y,b)) d\pi_{1}((x',a'),(y',b'))\big)^{\frac{1}{p}} \\ \\
{+} \big(\!\!\!\!\!\!\!\!\!\!\!\!\!\! \int\limits_{(Y \times B \times Z \times C)^{2}}\!\!\!\!\!\!\!\!\!\!\!\!\!\! \big((1-\alpha) d(b,c)^{q} +\alpha L(y,z,y',z')^{q}\big)^{p} d\pi_{2}((y,b),(z,c))\,d\pi_{2}((y',b'),(z',c'))\big)^{\frac{1}{p}} \bigg) \\ \\
{=} 2^{q-1}(d^{\Omega}_{FGW,\alpha,p,q}(\mu,\beta){+}d^{\Omega}_{FGW,\alpha,p,q}(\beta,\nu))
$
\\ \\
with (*) comes from (\ref{ineqL}) and (\ref{ineqD}) and (**) is Minkowski inequality. So when $q=1$ $d^{\Omega}_{FGW,\alpha,p,q}$ satisfies the triangle inequality and when $q>1$ $d^{\Omega}_{FGW,\alpha,p,q}$ satisfies a relaxed triangle inequality so that it defines a semi-metric as described previously.

\end{proof}

\subsection{Proof of Prop. \ref{concentration} Convergence and concentration inequality}

\begin{proof}

The proof of the convergence in $FGW$ dervies directly from the weak convergence of the empirical measure and lemma \ref{minimizer}. For the concentration \eqref{samespace} is valid between $\mu_{n}$ and $\mu$ since they are both in the same ground space. Then we have :
$$d^{\Omega}_{FGW,\alpha,p,1}(\mu_{n},\mu)^{p} \leq 2d^{X \times \Omega}_{W,p}(\mu_{n},\mu)^{p} \implies \mathbb{E}[d^{\Omega}_{FGW,\alpha,p,1}(\mu_{n},\mu)^{p}] \leq 2 \mathbb{E}[d^{X \times \Omega}_{W,p}(\mu_{n},\mu)^{p}]$$

We can directly apply theorem 1 in \cite{weedbach2017} to state the inequality.

\end{proof}

\subsection{Proof of theorem \ref{interpolationtheorem} Interpolation properties between GW and W}

\begin{proof}

To prove the first point of the theorem we want to have a converse inequality of (\ref{wassinequality}) and (\ref{gromovinequality}) in the limit cases.

Let $\pi_{OT} \in \couplingset(\mu_{A},\nu_{B})$ the optimal coupling for the pq-Wasserstein distance between $\mu_{A}$ and $,\nu_{B}$. We can use the same Gluing lemma (lemma 5.3.2 in \cite{ambrosio2005gradient}) to construct :

$$
  \rho \in P( 
      \mathrlap{\overbrace{\phantom{(c - 2)}}^{\text{$\mu$}}}
      X \times 
      \mathrlap{\underbrace{\phantom{2d) + (}}_{\text{$\pi_{OT}$}}}
      A \times
      \mathrlap{\overbrace{\phantom{2d) + (}}^{\text{$\nu$}}}
	B \times Y
      )
$$

such that $\rho \in \couplingset(\mu,\nu)$ and $P_{2,3} \# \rho =\pi_{OT}$.

Moreover we have :
\begin{equation}
\label{rhorho}
\underset{A\times B}{\int}d(a,b)^{pq} d\pi_{OT}(a,b)=\underset{X \times A\times B \times Y}{\int} d(a,b)^{pq} d\rho(x,a,b,y)
\end{equation}

Let $\alpha \geq 0$ and $\pi_{\alpha}$ optimal plan for the fused Gromov-Wasserstein distance between $\mu$, $\nu$. 

We can deduce that :

$d^{\Omega}_{FGW,\alpha,p,q}(\mu,\nu)^{p}- (1-\alpha)^{p}d^{\Omega}_{W,pq}(\mu_{A},\nu_{B})^{pq} \\
{=}\!\!\!\!\!\!\!\!\!\!\!\!\!\! {\small \int\limits_{(X\times A \times Y \times B)^{2}}\!\!\!\!\!\!\!\!\!\!\!\!\!\!  \big((1{-}\alpha) d(a,b)^{q} +\alpha L(x,y,x',y')^{q}\big)^{p}d\pi_{\alpha}((x,a),(y,b))d\pi_{\alpha}((x',a'),(y',b'))}{-}\underset{A\times B}{\int} (1{-}\alpha)^{p}d(a,b)^{pq}d\pi_{OT}(a,b) \\ \\
{\stackrel{(*)}{\leq}} {\small \int\limits_{(X\times A \times Y \times B)^{2}}\!\!\!\!\!\!\!\!\!\!\!\!\!\!  \big((1{-}\alpha) d(a,b)^{q} +\alpha L(x,y,x',y')^{q} \big)^{p} d\rho(x,a,b,y)d\rho(x',a',b',y')}
{-}\!\!\!\!\!\!\underset{X\times A\times Y \times B}{\int}\!\!\!\!\!\!  (1{-}\alpha)^{p} d(a,b)^{pq} d\rho(x,a,b,y) \\ \\
=(1-\alpha)^{p}{\small \int\limits_{(X\times A \times Y \times B)^{2}}\!\!\!\!\!\!\!\!\!\!\!\!\!\!  d(a,b)^{pq} d\rho(x,a,b,y)d\rho(x',a',b',y')}
{-}(1-\alpha)^{p}\underset{X\times A\times Y \times B}{\int}\!\!\!\!\!\!  d(a,b)^{pq} d\rho(x,a,b,y) \\ \\
{+}\sum_{k=0}^{p-1} {p\choose k}(1-\alpha)^{k}\alpha^{p-k} \int\limits_{(X\times A \times Y \times B)^{2}} d(a,b)^{qk}L(x,y,x',y')^{q(p-k)} d\rho(x,a,b,y)d\rho(x',a',b',y')\\ \\
{=} \sum_{k=0}^{p-1} {p\choose k}(1-\alpha)^{k}\alpha^{p-k} \int\limits_{(X\times A \times Y \times B)^{2}} d(a,b)^{qk}L(x,y,x',y')^{q(p-k)} d\rho(x,a,b,y)d\rho(x',a',b',y')$

We note $H_{k}=\int\limits_{(X\times A \times Y \times B)^{2}} d(a,b)^{qk}L(x,y,x',y')^{q(p-k)} d\rho(x,a,b,y)d\rho(x',a',b',y')$

Using (\ref{wassinequality})  we have shown that:

$$ (1-\alpha)(d^{\Omega}_{W,pq}(\mu_{A},\nu_{B}))^{q} \leq d^{\Omega}_{FGW,\alpha,p,q}(\mu,\nu) \leq  \big((1-\alpha)^{p} (d^{\Omega}_{W,pq}(\mu_{A},\nu_{B}))^{pq}+\sum_{k=0}^{p-1} {p\choose k}(1-\alpha)^{k}\alpha^{p-k}H_{k}\big)^{\frac{1}{p}}$$

So $\lim\limits_{\alpha \rightarrow 0} d^{\Omega}_{FGW,\alpha,p,q}(\mu,\nu)=(d^{\Omega}_{W,pq}(\mu_{A},\nu_{B}))^{q}$
 \\ \\
 
For the case $\alpha \rightarrow 1$ we rather consider $\pi_{GW} \in \couplingset(\mu_{X},\nu_{Y})$ the optimal coupling for the pq-Gromov-Wasserstein distance between $\mu_{X}$ and $,\nu_{Y}$ and we construct

$$
  \gamma \in P( 
      \mathrlap{\overbrace{\phantom{(c - 2)}}^{\text{$\mu$}}}
      A \times 
      \mathrlap{\underbrace{\phantom{2d) + (}}_{\text{$\pi_{GW}$}}}
      X \times
      \mathrlap{\overbrace{\phantom{2d) + (}}^{\text{$\nu$}}}
	Y \times B
      )
$$

such that $\gamma \in \couplingset(\mu,\nu)$ and $P_{2,3} \# \rho =\pi_{GW}$. We have :

\begin{equation}
\label{gammagamma}
\underset{X\times Y \times X \times Y}{\int}\!\!\!\!\!\!\!\!\!\!L(x,y,x',y')^{pq} d\pi_{GW}(x,y)d\pi_{GW}(x',y'){=}\!\!\!\!\!\!\!\!\!\!\underset{A \times X\times Y \times B}{\int}\!\!\!\!\!\!\!\!\!\! L(x,y,x',y')^{pq} d\gamma(a,x,y,b) d\gamma(a',x',y',b')
\end{equation}

So by suboptimality :

$d^{\Omega}_{FGW,\alpha,p,q}(\mu,\nu)^{p}- \alpha^{p} d_{GW,pq}(\mu_{X},\nu_{Y})^{pq} \\
{=}\!\!\!\!\!\!\!\!\!\!\!\!\!\! {\small \int\limits_{(X\times A \times Y \times B)^{2}}\!\!\!\!\!\!\!\!\!\!\!\!\!\!  \big((1{-}\alpha) d(a,b)^{q} +\alpha L(x,y,x',y')^{q}\big)^{p}d\pi_{\alpha}((x,a),(y,b))d\pi_{\alpha}((x',a'),(y',b'))}{-}\underset{A\times B}{\int} \alpha^{p}L(x,y,x',y')^{pq}d\pi_{GW}(a,b) \\ \\
{\stackrel{(*)}{\leq}} {\small \int\limits_{(X\times A \times Y \times B)^{2}}\!\!\!\!\!\!\!\!\!\!\!\!\!\!  \big((1{-}\alpha) d(a,b)^{q} +\alpha L(x,y,x',y')^{q} \big)^{p} d\rho(x,a,b,y)d\rho(x',a',b',y')}
{-}\!\!\!\!\!\!\underset{X\times A\times Y \times B}{\int}\!\!\!\!\!\!  \alpha^{p} L(x,y,x',y')^{pq} d\rho(x,a,b,y) \\ \\
=\alpha^{p}{\small \int\limits_{(X\times A \times Y \times B)^{2}}\!\!\!\!\!\!\!\!\!\!\!\!\!\!  L(x,y,x',y')^{pq} d\rho(x,a,b,y)d\rho(x',a',b',y')}
{-}\alpha^{p}\underset{X\times A\times Y \times B}{\int}\!\!\!\!\!\!  L(x,y,x',y')^{pq} d\rho(x,a,b,y) \\ \\
{+}\sum_{k=0}^{p-1} {p\choose k}(1-\alpha)^{p-k}\alpha^{k} \int\limits_{(X\times A \times Y \times B)^{2}} d(a,b)^{q(p-k)}L(x,y,x',y')^{qk} d\rho(x,a,b,y)d\rho(x',a',b',y')\\ \\
{=}\sum_{k=0}^{p-1} {p\choose k}(1-\alpha)^{p-k}\alpha^{k} \int\limits_{(X\times A \times Y \times B)^{2}} d(a,b)^{q(p-k)}L(x,y,x',y')^{qk} d\rho(x,a,b,y)d\rho(x',a',b',y')$

We note $J_{k}= \int\limits_{(X\times A \times Y \times B)^{2}} d(a,b)^{q(p-k)}L(x,y,x',y')^{qk} d\rho(x,a,b,y)d\rho(x',a',b',y')$

Using (\ref{gromovinequality})  we have shown that:

$$ \alpha(d^{\Omega}_{GW,pq}(\mu_{X},\nu_{Y}))^{q} \leq d^{\Omega}_{FGW,\alpha,p,q}(\mu,\nu) \leq  \big(\alpha^{p} (d^{\Omega}_{GW,pq}(\mu_{X},\nu_{Y}))^{pq}+\sum_{k=0}^{p-1} {p\choose k}(1-\alpha)^{p-k}\alpha^{k} J_{k})^{\frac{1}{p}}$$

so $\lim\limits_{\alpha \rightarrow 1}d^{\Omega}_{FGW,\alpha,p,q}(\mu,\nu)=(d^{\Omega}_{GW,pq}(\mu_{X},\nu_{Y}))^{q}$

\end{proof}

\subsection{Proof of theorem \ref{cstespeedtheo} Constant speed geodesic}
\begin{proof}
Let $t,s \in [0,1]$. Recalling :
$$\forall x_{0},a_{0},x_{1},a_{1} \in X_{0} \times A_{0} \times X_{1} \times A_{1},  \eta_{t}(x_{0},a_{0},x_{1},a_{1})=(x_{0},x_{1},(1-t)a_{0}+ta_{1}) $$

We note $\mathbb{S}_{t}=( X_{0}\times X_{1} \times \hat{A_{t}},(1-t)d_{X_{0}} \oplus t d_{X_{1}}, \mu_{t}=\eta_{t} \# \pi^{*} )$ and $d_{t}=(1-t)d_{X_{0}} \oplus t d_{X_{1}}$. Let $\|.\|$ be any $\ell_m$ norm for $m \geq 1$.

It suffices to prove :

$$d^{\mathbb{R}^{d}}_{FGW,\alpha,p,1}(\mu_{t},\mu_{s}) \leq |t-s| d^{\mathbb{R}^{d}}_{FGW,\alpha,p,1}(\mu_{0},\mu_{1})$$

To do so we consider $\Delta^{t}_{s} \in P(X_{0} \times X_{1} \times \hat{A_{t}} \times X_{0} \times X_{1} \times \hat{A_{s}})$ defined by  $\Delta^{t}_{s}=(\eta_{t} \times \eta_{s} )\# \pi^{*} \in \couplingset(\mu_{t},\mu_{s})$ and the following "diagonal" coupling :
$$d\gamma_{s}^{t}( (x_{0},x_{1}), a , (x_{0}'',x_{1}''),b) = d \Delta^{t}_{s} ((x_{0},x_{1}), a , (x_{0}'',x_{1}''),b) d\delta_{(x_{0},x_{1})}(x_{0}'',x_{1}'')$$

Then $\gamma_{s}^{t} \in P(X_{0} \times X_{1} \times A_{t} \times X_{0} \times X_{1} \times A_{s})$  and since $\Delta^{t}_{s} \in \couplingset(\mu_{t},\mu_{s})$ then $\gamma_{s}^{t} \in \couplingset(\mu_{t},\mu_{s})$ So by suboptimality :

$d^{\mathbb{R}^{d}}_{FGW,\alpha,p,1}(\mu_{t},\mu_{s})^{p} \\ \\
{\leq}\!\!\!\!\!\!\!\!\!\!\!\!\!\!\!\!\!\!\!\!\!\!\!\!\!\!\!\!{\small \int\limits_{(X_{0} \times X_{1} \times \hat{A_{t}} \times X_{0} \times X_{1} \times \hat{A_{s}})^{2}}\!\!\!\!\!\!\!\!\!\!\!\!\!\!\!\!\!\!\!\!\!\!\!\!\!\!\!\!\!\! \big((1-\alpha)d(a,b){+}\alpha |d_{t}[(x_{0},x_{1}),(x_{0}',x_{1}')]{-}d_{s}[(x_{0}'',x_{1}''),(x_{0}''',x_{1}''')] |\big)^{p} \,d\gamma_{s}^{t}(x_{0},x_{1}, a ,x_{0}'',x_{1}'',b)d\gamma_{s}^{t}(x_{0}',x_{1}', a',x_{0}''',x_{1}''',b')} \\ \\
=\!\!\!\!\!\!\!\!\!\!\!\!\!\!\!\!\!\!\!\!\!\!\!\!\!\!\!\!{\small \int\limits_{(X_{0} \times X_{1} \times \hat{A_{t}} \times X_{0} \times X_{1} \times \hat{A_{s}})^{2}}\!\!\!\!\!\!\!\!\!\!\!\!\!\!\!\!\!\!\!\!\!\!\!\!\!\!\!\!\!\! \big((1-\alpha)d(a,b){+}\alpha |d_{t}[(x_{0},x_{1}),(x_{0}',x_{1}')]{-}d_{s}[(x_{0},x_{1}),(x_{0}',x_{1}')] |\big)^{p} \,d\Delta_{s}^{t}(x_{0},x_{1}, a ,x_{0},x_{1},b)d\Delta_{s}^{t}(x_{0}',x_{1}', a',x_{0}',x_{1}',b')} \\ \\
=\int\limits_{(X_{0} \times A_{0} \times X_{1} \times A_{1})^{2}}\big((1-\alpha) \|(1{-}t)a+tb{-}(1{-}s)a{-}sb \| \\ \\
+\alpha |(1{-}t)d_{X_{0}}(x_{0},x_{0}'){+}td_{X_{1}}(x_{1},x_{1}'){-}(1-s)d_{X_{0}}(x_{0},x_{0}'){+}sd_{X_{1}}(x_{1},x_{1}') |\big)^{p} d\pi^{*}(x_{0},a,x_{1},b)d\pi^{*}(x_{0}',a',x_{1}',b') \\ \\
{=} |t{-}s|^{p} \int\limits_{(X_{0} \times A_{0} \times X_{1} \times A_{1})^{2}}\big( (1{-}\alpha) \|a{-}b\| + \alpha |d_{X_{0}}(x_{0},x_{0}'){-}d_{X_{1}}(x_{1},x_{1}')| \big)^{p} d\pi^{*}(x_{0},a,x_{1},b)d\pi^{*}(x_{0}',a',x_{1}',b')
$
\\ \\
So $d^{\mathbb{R}^{d}}_{FGW,\alpha,p,1}(\mu_{t},\mu_{s}) \leq |t-s| d^{\mathbb{R}^{d}}_{FGW,\alpha,p,1}(d_{0},d_{1},\mu_{0},\mu_{1})$. Moreover by linearity of $\text{supp}$ it is clear that $\text{supp}(\mu_{t})=\hat{A_{t}}$.

It is straightforward to extend this result for any $q\geq 1$. More precisely we have the following inequality 

\begin{equation}
\label{geodesicforanyq}
d^{\mathbb{R}^{d}}_{FGW,\alpha,p,q}(\mu_{t},\mu_{s}) \leq |t-s|^{q} d^{\mathbb{R}^{d}}_{FGW,\alpha,p,q}(\mu_{0},\mu_{1}) \leq |t-s| d^{\mathbb{R}^{d}}_{FGW,\alpha,p,q}(\mu_{0},\mu_{1})
\end{equation}
\end{proof}

\subsection{Proof of theorem \ref{unicitytheorem} Unicity of the geodesic}

We note $\|.\|_{p}$ the $\ell_{p}$ norm. In order to prove the unicity of the geodesic we will need the following lemma :

\begin{lemma}
\begin{itemize}
\item $\forall p \in [1,\infty[, \forall t_{0} <...<t_{n}, \forall a_{1},..,a_{n} \in \mathbb{R_{+}}$,

\begin{equation}
\label{sommepuissancep}
\frac{1}{(t_{n}-t_{0})^{p-1}}(\sum_{i=1}^{n} a_{i})^{p} \leq \sum_{i=1}^{n} \frac{1}{(t_{i}-t_{i-1})^{p-1}} a_{i}^{p}
\end{equation}

\item $\forall p \in [2,\infty[, \forall \ \text{dyadic} \ t\in ]0,1[, \exists C=C(p,t) > 0$,

\begin{equation}
\label{convexe}
\forall a,b\in \mathbb{R}^{d}, \|ta+(1-t)b\|_{p}^{p} \leq t\|a\|_{p}^{p}+(1-t)\|b\|_{p}^{p}-\frac{t(1-t)}{C}\|a-b\|_{p}^{p}
\end{equation}

\end{itemize}

\end{lemma}

\begin{proof}
Direct application of lemma 3.4 in \cite{Sturm2006}.
\end{proof}

\begin{proof}[Proof of the theorem]

Let $q\geq 2$ and $ (\mathbb{S}_{t})_{t\in [0,1]}=\big((X_{t} \times  A_{t} ,d_{X_{t}}, \mu_{t} )\big)_{t\in [0,1]} $ a geodesic in $H(\mathbb{R}^{d})$ be given. $\|.\|$ denotes the $\ell_q$ norm.

The goal is to show that this geodesic in $(H(\mathbb{R}^{d}),d^{\mathbb{R}^{d}}_{FGW,\alpha,1,q})$ is actually in the form : 

$$\big((X_{0}\times X_{1} \times \hat{A_{t}},(1-t)d_{X_{0}} \oplus t d_{X_{1}},\eta_{t}\#\pi^{*})\big)_{t\in [0,1]}$$ with $\pi^{*}$ an optimal coupling between the endpoints $(X_{0} \times A_{0} ,d_{X_{0}},\mu_{0})$ and $(X_{1} \times A_{1} ,d_{X_{1}},\mu_{1})$ for the $,d^{\Omega}_{FGW,\alpha,1,q}$ distance and $\eta_{t},\hat{A_{t}}$ defined in theorem \ref{cstespeedtheo}. The equality of this two geodesics will be with respect to the equivalence relation $\sim$ of structured objects defined previously. 
\\ \\
In order to prove this result we first consider discrete dyadic times $t=i2^{-k}$ for $k \in \mathbb{N}$ and $i\in 1,..,2^{k}$ and we will extend by continuity for any $t\in [0,1]$. 

Let $\pi_{i}$ be the optimal couplings for $d^{\Omega}_{FGW,\alpha,1,q}$ distance between $\mu_{(i-1)2^{-k}}$ and $\mu_{i2^{-k}}$.

Using Gluing lemma we can construct : $$\pi \in P(
X_{0} \times A_{0} \times X_{2^{-k}} \times A_{2^{-k}} \times ... \times X_{i2^{-k}} \times A_{i2^{-k}} \times ... \times X_{1} \times A_{1}) $$

by gluing all the couplings $\pi_{i}$.

We consider the measures  :

$$\tilde{\pi}=(P_{1,2} \times P_{2^{k}-1,2^{k}})\# \pi^{1} \in P(X_{0} \times A_{0} \times X_{1} \times A_{1})$$
coupling of $\mu_{0}$ and $\mu_{1}$ 

and for $t \in [0,1]$ with the form $t=i2^{-k}$ :

$$\tilde{\pi}_{t}=(P_{1,2} \times P_{t-1,t} \times P_{2^{k}-1,2^{k}})\# \pi \in P(X_{0} \times A_{0} \times X_{t} \times A_{t} \times X_{1} \times A_{1})$$
such that $P_{1,2,5,6} \# \tilde{\pi}_{t}=\tilde{\pi}$ and $\tilde{\pi}_{t}$ is a coupling of $\mu_{t}$ and $\tilde{\pi}$

So, by suboptimality of $\tilde{\pi}$ : \\ \\
$
d^{\Omega}_{FGW,\alpha,1,q}(\mu_{0},\mu_{1}) \\ \\
{\leq}{\small \int\limits_{(X_{0} \times A_{0} \times X_{1} \times A_{1})^{2}}(1-\alpha)d(a_{0},a_{1})^{q}{+}\alpha |d_{X_{0}}(x_{0},x_{0}'){-}d_{X_{1}}(x_{1},x_{1}')|^{q}}d\tilde{\pi}(x_{0},a_{0},x_{1},a_{1})d\tilde{\pi}(x_{0}',a_{0}',x_{1}',a_{1}') \\ \\
{=}\!\!\!\!\!\!\!\!\!\!\!\!\!\!\!\!\!\!\!\!\!\!\!\!\!\!\!\!{\small \int\limits_{(X_{0} \times A_{0} \times X_{t} \times A_{t} \times X_{1} \times A_{1})^{2}}\!\!\!\!\!\!\!\!\!\!\!\!\!\!\!\!\!\!\!\!\!\!\!\!\!\!\!\!(1-\alpha)\|a_{0}{-}a_{1}\|^{q}{+}\alpha |d_{X_{0}}(x_{0},x_{0}'){-}d_{X_{1}}(x_{1},x_{1}')|^{q}}d\tilde{\pi}_{t}(x_{0},a_{0},x_{t},a_{t},x_{1},a_{1})d\tilde{\pi}_{t}(x_{0}',a_{0}',x_{t}',a_{t}',x_{1}',a_{1}') \\ \\
{=}\!\!\!\!\!\!\!\!\!\!\!\!\!\!\!\!\!\!\!\!\!\!\!\!\!\!\!\!{\small \int\limits_{(X_{0} \times A_{0} \times X_{t} \times A_{t} \times X_{1} \times A_{1})^{2}}\!\!\!\!\!\!\!\!\!\!\!\!\!\!\!\!\!\!\!\!\!\!\!\!\!\!\!\!(1-\alpha)\|a_{0}{-}a_{t}{+}a_{t}{-}a_{1}\|^{q}{+}\alpha |d_{X_{0}}(x_{0},x_{0}'){-}d_{X_{t}}(x_{t},x_{t}'){+}d_{X_{t}}(x_{t},x_{t}')-d_{X_{1}}(x_{1},x_{1}')|^{q}}\\ \\ d\tilde{\pi}_{t}(x_{0},a_{0},x_{t},a_{t},x_{1},a_{1})d\tilde{\pi}_{t}(x_{0}',a_{0}',x_{t}',a_{t}',x_{1}',a_{1}') \\ \\
{=}(1-\alpha)\!\!\!\!\!\!\!\!\!\!\!\!\!\!\!\!\!\!\!\!\!\!\!\!\!\!\!\!{\small \int\limits_{(X_{0} \times A_{0} \times X_{t} \times A_{t} \times X_{1} \times A_{1})^{2}}\!\!\!\!\!\!\!\!\!\!\!\!\!\!\!\!\!\!\!\!\!\!\!\!\!\!\!\! \|a_{0}{-}a_{t}{+}a_{t}{-}a_{1}\|^{q}d\tilde{\pi}_{t}(x_{0},a_{0},x_{t},a_{t},x_{1},a_{1})}d\tilde{\pi}_{t}(x_{0}',a_{0}',x_{t}',a_{t}',x_{1}',a_{1}') \\ \\
{+} \alpha \!\!\!\!\!\!\!\!\!\!\!\!\!\!\!\!\!\!\!\!\!\!\!\!\!\!\!\!{\small \int\limits_{(X_{0} \times A_{0} \times X_{t} \times A_{t} \times X_{1} \times A_{1})^{2}}\!\!\!\!\!\!\!\!\!\!\!\!\!\!\!\!\!\!\!\!\!\!\!\!\!\!\!\!|d_{X_{0}}(x_{0},x_{0}'){-}d_{X_{t}}(x_{t},x_{t}'){+}d_{X_{t}}(x_{t},x_{t}')-d_{X_{1}}(x_{1},x_{1}')|^{q}}d\tilde{\pi}_{t}(x_{0},a_{0},x_{t},a_{t},x_{1},a_{1})d\tilde{\pi}_{t}(x_{0}',a_{0}',x_{t}',a_{t}',x_{1}',a_{1}')
$

We note $(I)$ and $(II)$ this two terms. We have using \eqref{convexe} : \\ \\
$(I) \leq (1-\alpha) \big(\!\!\!\!\!\!\!\!\!\!\!\!\!\!\!\!\!\!\!\!\!\!\!\!\!\!\!\!{\small \int\limits_{(X_{0} \times A_{0} \times X_{t} \times A_{t} \times X_{1} \times A_{1})^{2}}\!\!\!\!\!\!\!\!\!\!\!\!\!\!\!\!\!\!\!\!\!\!\!\!\!\!\!\!\frac{1}{t^{q-1}} \|a_{0}{-}a_{t}\|^{q}+\frac{1}{(1-t)^{q-1}} \|a_{t}{-}a_{1}\|^{q}}d\tilde{\pi}_{t}(x_{0},a_{0},x_{t},a_{t},x_{1},a_{1})d\tilde{\pi}_{t}(x_{0}',a_{0}',x_{t}',a_{t}',x_{1}',a_{1}') \\ \\
{-}\frac{1}{C(t(1-t))^{q-1}}\!\!\!\!\!\!\!\!\!\!\!\!\!\!\!\!\!\!\!\!\!\!\!\!\!\!\!\!{\small \int\limits_{(X_{0} \times A_{0} \times X_{t} \times A_{t} \times X_{1} \times A_{1})^{2}}\!\!\!\!\!\!\!\!\!\!\!\!\!\!\!\!\!\!\!\!\!\!\!\!\!\!\!\!\|(1-t)(a_{0}{-}a_{t})-t(a_{t}-a_{1})\|^{q}}d\tilde{\pi}_{t}(x_{0},a_{0},x_{t},a_{t},x_{1},a_{1})d\tilde{\pi}_{t}(x_{0}',a_{0}',x_{t}',a_{t}',x_{1}',a_{1}') \big)
$
\\ \\
and : 
\\ \\
$(II) \leq (1-\alpha) \bigg(\!\!\!\!\!\!\!\!\!\!\!\!\!\!\!\!\!\!\!\!\!\!\!\!\!\!\!\!{\small \int\limits_{(X_{0} \times A_{0} \times X_{t} \times A_{t} \times X_{1} \times A_{1})^{2}}\!\!\!\!\!\!\!\!\!\!\!\!\!\!\!\!\!\!\!\!\!\!\!\!\!\!\!\!\frac{1}{t^{q-1}} |d_{X_{0}}(x_{0},x_{0}'){-}d_{X_{t}}(x_{t},x_{t}')|^{q}+\frac{1}{(1-t)^{q-1}} |d_{X_{t}}(x_{t},x_{t}'){-}d_{X_{1}}(x_{1},x_{1}')|^{q}}\\ \\
d\tilde{\pi}_{t}(x_{0},a_{0},x_{t},a_{t},x_{1},a_{1})d\tilde{\pi}_{t}(x_{0}',a_{0}',x_{t}',a_{t}',x_{1}',a_{1}') \\ \\
{-}\frac{1}{C(t(1-t))^{q-1}}\!\!\!\!\!\!\!\!\!\!\!\!\!\!\!\!\!\!\!\!\!\!\!\!\!\!\!\!{\small \int\limits_{(X_{0} \times A_{0} \times X_{t} \times A_{t} \times X_{1} \times A_{1})^{2}}\!\!\!\!\!\!\!\!\!\!\!\!\!\!\!\!\!\!\!\!\!\!\!\!\!\!\!\!|(1-t)(d_{X_{0}}(x_{0},x_{0}'){-}d_{X_{t}}(x_{t},x_{t}'))-t(d_{X_{t}}(x_{t},x_{t}')-d_{X_{1}}(x_{1},x_{1}'))|^{q}} \\ \\
d\tilde{\pi}_{t}(x_{0},a_{0},x_{t},a_{t},x_{1},a_{1})d\tilde{\pi}_{t}(x_{0}',a_{0}',x_{t}',a_{t}',x_{1}',a_{1}') \bigg)
$
\\ \\
Combining both terms we can bound the $FGW$ distance as : \\ \\ 
$
d^{\mathbb{R}^{d}}_{FGW,\alpha,1,q}(\mu_{0},\mu_{1}) \\ \\
{\leq}\!\!\!\!\!\!\!\!\!\!\!\!\!\!\!\!\!\!\!\!\!\!\!\!\!\!\!\!{\small \int\limits_{(X_{0} \times A_{0} \times X_{t} \times A_{t} \times X_{1} \times A_{1})^{2}}\!\!\!\!\!\!\!\!\!\!\!\!\!\!\!\!\!\!\!\!\!\!\!\!\!\!\!\! \frac{(1-\alpha)\|a_{0}-a_{t}\|^{q}+\alpha|d_{X_{0}}(x_{0},x_{0}')-d_{X_{t}}(x_{t},x_{t}')|^{q}}{t^{q-1}}+\frac{(1-\alpha)\|a_{t}-a_{1}\|^{q}+\alpha|d_{X_{t}}(x_{t},x_{t}')-d_{X_{1}}(x_{1},x_{1}')|^{q}}{(1-t)^{q-1}}}d\tilde{\pi}_{t}(x_{0},a_{0},x_{t},a_{t},x_{1},a_{1})\\\\d\tilde{\pi}_{t}(x_{0}',a_{0}',x_{t}',a_{t}',x_{1}',a_{1}') \\ \\
-\frac{1}{C(t(1-t))^{q-1}}\!\!\!\!\!\!\!\!\!\!\!\!\!\!\!\!\!\!\!\!\!\!\!\!\!\!\!\!{\small \int\limits_{(X_{0} \times A_{0} \times X_{t} \times A_{t} \times X_{1} \times A_{1})^{2}}\!\!\!\!\!\!\!\!\!\!\!\!\!\!\!\!\!\!\!\!\!\!\!\!\!\!\!\!(1-\alpha)\|(1-t)a_{0}+ta_{1}{-}a_{t}\|^{q}+\alpha|(1-t)d_{X_{0}}(x_{0},x_{0}')+td_{X_{1}}(x_{1},x_{1}'){-}d_{X_{t}}(x_{t},x_{t}'))|^{q}}\\\\d\tilde{\pi}_{t}(x_{0},a_{0},x_{t},a_{t},x_{1},a_{1}) 
d\tilde{\pi}_{t}(x_{0}',a_{0}',x_{t}',a_{t}',x_{1}',a_{1}')
$
\\ \\
We write this inequality as : $$d^{\mathbb{R}^{d}}_{FGW,\alpha,1,q}(\mu_{0},\mu_{1})\leq (I')-\frac{1}{C(t(1-t))^{q-1}}(II').$$

We will show that $(I')$ actually achieves $d^{\mathbb{R}^{d}}_{FGW,\alpha,1,q}(\mu_{0},\mu_{1})$. In this way we will conclude that all inequalities are equalities which will result on $(II')$ being nul and $\tilde{\pi}$ being the optimal coupling between the endpoints. As $(II')$ vanishes we will conclude further that  the geodesic in $(H(\mathbb{R}^{d}),d^{\mathbb{R}^{d}}_{FGW,\alpha,1,q})$ has the wanted form.

Let us have a closer look at $(I')$ with $t=i2^{-k}$ :
\\ \\
$(I')=\!\!\!\!\!\!\!\!\!\!\!\!\!\!\!\!\!\!\!\!\!\!\!\!\!\!\!\!{\small \int\limits_{(X_{0} \times A_{0} \times X_{t} \times A_{t} \times X_{1} \times A_{1})^{2}}\!\!\!\!\!\!\!\!\!\!\!\!\!\!\!\!\!\!\!\!\!\!\!\!\!\!\!\! \frac{(1-\alpha)\|a_{0}-a_{t}\|^{q}+\alpha|d_{X_{0}}(x_{0},x_{0}')-d_{X_{t}}(x_{t},x_{t}')|^{q}}{t^{q-1}}+\frac{(1-\alpha)\|a_{t}-a_{1}\|^{q}+\alpha|d_{X_{t}}(x_{t},x_{t}')-d_{X_{1}}(x_{1},x_{1}')|^{q}}{(1-t)^{q-1}}}
d\tilde{\pi}_{t}(x_{0},a_{0},x_{t},a_{t},x_{1},a_{1})\\\\d\tilde{\pi}_{t}(x_{0}',a_{0}',x_{t}',a_{t}',x_{1}',a_{1}') \\ \\
{=}2^{k(q-1)}{\small \int\limits_{}\frac{(1-\alpha)\|a_{0}-a_{i2^{-k}}\|^{q}+\alpha|d_{X_{0}}(x_{0},x_{0}')-d_{X_{i2^{-k}}}(x_{i2^{-k}},x_{i2^{-k}}')|^{q}}{i^{q-1}}+\frac{(1-\alpha)\|a_{i2^{-k}}-a_{1}\|^{q}+\alpha|d_{X_{i2^{-k}}}(x_{i2^{-k}},x_{i2^{-k}}')-d_{X_{1}}(x_{1},x_{1}')|^{q}}{(2^{k}-i)^{q-1}}}\\ \\
d\tilde{\pi}(x_{0},a_{0},...,x_{i2^{-k}},a_{i2^{-k}},...,x_{1},a_{1})d\tilde{\pi}(x_{0}',a_{0}',...,x_{i2^{-k}}',a_{i2^{-k}}',...,x_{1}',a_{1}')
$
\\ \\
Using a telescopic sum and \eqref{sommepuissancep} we have : 
$$\frac{1}{i^{q-1}}\|a_{0}-a_{i2^{-k}}\|^{q}=\frac{1}{i^{q-1}}\|\sum_{j=1}^{i}a_{(j-1)2^{-k}}-a_{j2^{-k}}\|^{q}\leq \frac{1}{i^{q-1}}\big(\sum_{j=1}^{i}\|a_{(j-1)2^{-k}}-a_{j2^{-k}}\|\big)^{q} \leq \sum_{j=1}^{i}\|a_{(j-1)2^{-k}}-a_{j2^{-k}}\|^{q}$$

and in the same way :
$$\frac{1}{(2^{k}-i)^{q-1}}\|a_{i2^{-k}}-a_{1}\|^{q} \leq \sum_{j=i+1}^{2^{k}}\|a_{(j-1)2^{-k}}-a_{j2^{-k}}\|^{q}$$

such that (by doing the same reasoning for the $d_{X_{s}}$'s) :
\\ \\
$(I')\leq 2^{k(q-1)} \sum_{j=1}^{2^{k}} \int\limits_{} (1-\alpha)\|a_{(j-1)2^{-k}}-a_{j2^{-k}}\|^{q}+\alpha |d_{X_{(j-1)2^{-k}}}(x_{(j-1)2^{-k}},x_{(j-1)2^{-k}}')-d_{X_{j2^{-k}}}(x_{j2^{-k}},x_{j2^{-k}}')|^{q}\\ \\d\tilde{\pi}(x_{0},a_{0},...,x_{(j-1)2^{-k}},a_{(j-1)2^{-k}},,x_{j2^{-k}},a_{j2^{-k}},...,x_{1},a_{1})d\tilde{\pi}(x_{0}',a_{0}',...,,x_{(j-1)2^{-k}}',a_{(j-1)2^{-k}}',x_{j2^{-k}}',a_{j2^{-k}}',...,x_{1}',a_{1}')\\ \\
{=}2^{k(q-1)} \sum_{j=1}^{2^{k}} d^{\mathbb{R}^{d}}_{FGW,\alpha,1,q}(\mu_{(j-1)2^{-k}},\mu_{j2^{-k}})
$

Now using equation \eqref{geodesicforanyq} :
$$(I')\leq 2^{k(q-1)} \sum_{j=1}^{2^{k}} 2^{-kq} d^{\mathbb{R}^{d}}_{FGW,\alpha,1,q}(\mu_{0},\mu_{1})=d^{\mathbb{R}^{d}}_{FGW,\alpha,1,q}(\mu_{0},\mu_{1})$$

So $(I')=d^{\mathbb{R}^{d}}_{FGW,\alpha,1,q}(\mu_{0},\mu_{1})$ 

This result means firstly that all inequalities are equalities and so $\tilde{\pi}$ is actually the optimal coupling between the endpoints for the $d^{\mathbb{R}^{d}}_{FGW,\alpha,1,q}$ distance.

Secondly since we had : $$d^{\mathbb{R}^{d}}_{FGW,\alpha,1,q}(\mu_{0},\mu_{1})\leq (I')-\frac{1}{C(t(1-t))^{q-1}}(II')$$ we can conclude that $(II')=0$ and more precisely :
\\ \\
$\!\!\!\!\!\!\!\!\!\!\!\!\!\!\!\!\!\!\!\!\!\!\!\!\!\!\!\!\!\!\!\!\!\!\!\!\!\!\!\!\!\!\!\!\!\!\!\!\!\!\!\!\!\!\!\!\int\limits_{(X_{0} \times A_{0} \times X_{t} \times A_{t} \times X_{1} \times A_{1})^{2}}\!\!\!\!\!\!\!\!\!\!\!\!\!\!\!\!\!\!\!\!\!\!\!\!\!\!\!\!(1-\alpha)\|(1-t)a_{0}+ta_{1}{-}a_{t}\|^{q}+\alpha|(1-t)d_{X_{0}}(x_{0},x_{0}')+td_{X_{1}}(x_{1},x_{1}'){-}d_{X_{t}}(x_{t},x_{t}'))|^{q}\\ \\ d\tilde{\pi}_{t}(x_{0},a_{0},x_{t},a_{t},x_{1},a_{1})d\tilde{\pi}_{t}(x_{0}',a_{0}',x_{t}',a_{t}',x_{1}',a_{1}')=0$
\\ \\

Moreover since $\tilde{\pi}_{t}$ is a coupling of $\mu_{t}$ and $\tilde{\pi}$ using theorem \ref{Topology} this implies that $(X_{t} \times A_{t},d_{X_{t}}, ,\mu_{t})$ and $(X_{0}\times X_{1} \times \hat{A_{t}},(1-t)d_{X_{0}} \oplus t d_{X_{1}},\eta_{t}\#\tilde{\pi} )$ are equivalent.

To summarize : we have proven that the coupling constructed by guing all the optimal couplings between each point of the geodesic with respect to the $d^{\mathbb{R}^{d}}_{FGW,\alpha,1,q}$ distance is actually an optimal coupling between the endpoints of this geodesic. Moreover we have proven that the geodesic is equivalent with $(X_{0}\times X_{1} \times \hat{A_{t}},(1-t)d_{X_{0}} \oplus t d_{X_{1}},\eta_{t}\#\tilde{\pi} )$. 

This result is valid for all time $t=i2^{-k}$ and actually $\tilde{\pi}$ depend on $k$ such that $\tilde{\pi}=\tilde{\pi}^{k}$. According to the sequential compactness of $\Pi(\mu_{0},\mu_{1})$ the family $(\tilde{\pi}^{k})_{k \in \mathbb{N}}$ has a accumulation point $\tilde{\pi}^{\infty}$ in $\Pi(\mu_{0},\mu_{1})$. The previous result can be applied to $\tilde{\pi}^{\infty}$ such that for all dyadic numbers $t \in [0,1]$, $(X_{t}\times A_{t},d_{X_{t}}, \mu_{t})$ and $(X_{0}\times X_{1} \times \hat{A_{t}},(1-t)d_{X_{0}} \oplus t d_{X_{1}},\eta_{t}\#\tilde{\pi}^{\infty} )$ are equivalent. 

We conclude for all $t \in [0,1]$ by continuity in $t$ of theses two sets as elements of $H(\mathbb{R}^{d})$.

\end{proof}

\subsection{Equivalence between definitions of $FGW$ \label{deffgweq}}

Let $\mu$ and $\nu$ be two structured objects. We note $d^{\Omega}_{FGW,\alpha,p,q}$ the $FGW$ distance for the cost with convex combination \textit{i.e} for $\alpha \in ]0,1[$ (we omit the extreme cases) : 

$$d^{\Omega}_{FGW,\alpha,p,q}(\mu,\nu)^{p}{=}\underset{\pi \in \Pi(\mu,\nu)}{\text{inf}}\!\!\!\!\!\!\!\!\!\!\!\! \int\limits_{(X\times A \times Y \times B)^{2}} \!\!\!\!\!\!\!\!\!\!\!\!\!\!\!\! \big((1-\alpha) d(a,b)^{q} +\alpha L(x,y,x',y')^{q} \big)^{p}d\pi((x,a),(y,b))d\pi((x',a'),(y',b'))$$

and $\tilde{d}^{\Omega}_{FGW,\alpha,p,q}$ the $FGW$ distance for the cost defined in \cite{2018arXiv180509114V}, \textit{i.e} for $\tilde{\alpha} \in ]0,\infty[$ :

$$\tilde{d}^{\Omega}_{FGW,\tilde{\alpha},p,q}(\mu,\nu)^{p}{=}\underset{\pi \in \Pi(\mu,\nu)}{\text{inf}}\!\!\!\!\!\!\!\!\!\!\!\! \int\limits_{(X\times A \times Y \times B)^{2}} \!\!\!\!\!\!\!\!\!\!\!\!\!\!\!\! \big(d(a,b)^{q} +\tilde{\alpha} L(x,y,x',y')^{q} \big)^{p}d\pi((x,a),(y,b))d\pi((x',a'),(y',b'))$$

We note also $\pi^{*}$ and $\tilde{\pi}^{*}$ the optimal plan for the first and second distance respectively. Then we have for any $\pi$ in $\Pi(\mu,\nu)$ (with slight abuses of notations) : \\ \\
$\int\limits_{} \big((1-\alpha) d(a,b)^{q} +\alpha L(x,y,x',y')^{q} \big)^{p}d\pi^{*} d\pi^{*} \leq \int\limits_{} \big((1-\alpha) d(a,b)^{q} +\alpha L(x,y,x',y')^{q} \big)^{p}d\pi d\pi \\ \\
\implies \int\limits_{} \big( d(a,b)^{q} +\frac{\alpha}{1-\alpha} L(x,y,x',y')^{q} \big)^{p}d\pi^{*} d\pi^{*} \leq \int\limits_{} \big(d(a,b)^{q} +\frac{\alpha}{1-\alpha} L(x,y,x',y')^{q} \big)^{p}d\pi d\pi $

by dividing by $(1-\alpha)^{p}$. This implies that $\pi^{*}$ is an optimal plan for $\tilde{d}^{\Omega}_{FGW,\frac{\alpha}{1-\alpha},p,q}$ since $\frac{\alpha}{1-\alpha}\in ]0,\infty[$. Conversly, \\ \\
$\int\limits_{} \big( d(a,b)^{q} +\tilde{\alpha} L(x,y,x',y')^{q} \big)^{p}d\tilde{\pi}^{*} d\tilde{\pi}^{*} \leq \int\limits_{} \big( d(a,b)^{q} +\tilde{\alpha} L(x,y,x',y')^{q} \big)^{p}d\pi d\pi \\ \\
\implies \int\limits_{} \big( (1- \frac{\tilde{\alpha}}{1+\tilde{\alpha}})d(a,b)^{q} +\frac{\tilde{\alpha}}{1+\tilde{\alpha}} L(x,y,x',y')^{q} \big)^{p}d\tilde{\pi}^{*} d\tilde{\pi}^{*} \leq \int\limits_{} \big(d(a,b)^{q} +\frac{\tilde{\alpha}}{1+\tilde{\alpha}}L(x,y,x',y')^{q} \big)^{p}d\pi d\pi $

by dividing by $(1+\tilde{\alpha})^{p}$. This implies that $\tilde{\pi}^{*}$ is an optimal plan for $d^{\Omega}_{FGW,\frac{\tilde{\alpha}}{1+\tilde{\alpha}},p,q}$ since $\frac{\tilde{\alpha}}{1+\tilde{\alpha}} \in ]0,1[$.


\section{Conclusion}

We have presented in this paper a new OT distance called Fused Gromov-Wasserstein distance. Inspired by both Wasserstein and Gromov-Wasserstein distances the $FGW$ distance compare can compare structured objects by including the inherent relations that exist between the elements of the objects, constituting their structure information, and their feature information, part of a common ground space between each structured objects. We stated mathematical results about this new distance such as metric, interpolation and geodesic properties. We also gave a concentration result for the convergence of finite samples. We illustrated this new distance on structured objects and applied it to graph barycenter computation, graph clustering and mesh interpolation. 

\bibliographystyle{imaiai}
\ifx\undefined\BySame
\newcommand{\BySame}{\leavevmode\rule[.5ex]{3em}{.5pt}\ }
\fi
\ifx\undefined\textsc
\newcommand{\textsc}[1]{{\sc #1}}
\newcommand{\emph}[1]{{\em #1\/}}
\let\tmpsmall\small
\renewcommand{\small}{\tmpsmall\sc}
\fi

\end{document}